%% file: main.tex
\documentclass{article}

\usepackage[nonatbib,preprint]{neurips_2024}

\usepackage[pagebackref=true]{hyperref}  
\renewcommand*{\backrefalt}[4]{%
    \ifcase #1 \footnotesize{(Not cited.)}%
    \or        \footnotesize{(Cited on page~#2)}%
    \else      \footnotesize{(Cited on pages~#2)}%
    \fi}
\usepackage{url}            
\usepackage{booktabs}       
\usepackage{amsfonts}       
\usepackage{nicefrac}       
\usepackage{microtype}      

\usepackage{natbib}
\usepackage{sidecap}

\usepackage{xcolor}         

\input{macros}

\title{
Byzantine Robustness and Partial Participation Can Be Achieved at Once: Just Clip Gradient Differences\thanks{\bf The original version with all the results was prepared on September 29, 2023. The first version with minor modifications appeared on arXiv on November 23, 2023. The changes in the second version: a heuristic extension of the proposed method, new numerical results, a simpler presentation of the main results, and corrections of small typos.}}

\author{%
  Grigory Malinovsky\thanks{Part of the work was done when G.~Malinovsky was visiting MBZUAI.} \\
KAUST\thanks{King Abdullah University of Science and Technology.}, MBZUAI\thanks{Mohamed bin Zayed University of Artificial Intelligence.}
\And
Peter Richt\'arik\\
KAUST
\And
Samuel Horv\'ath\\
MBZUAI
\And
Eduard Gorbunov\thanks{Corresponding author: \texttt{eduard.gorbunov@mbzuai.ac.ae}.}\\
MBZUAI
}

\begin{document}

\maketitle

\begin{abstract}
  Distributed learning has emerged as a leading paradigm for training large machine learning models. However, in real-world scenarios, participants may be unreliable or malicious, posing a significant challenge to the integrity and accuracy of the trained models. Byzantine fault tolerance mechanisms have been proposed to address these issues, but they often assume full participation from all clients, which is not always practical due to the unavailability of some clients or communication constraints. In our work, we propose the first distributed method with client sampling and provable tolerance to Byzantine workers. The key idea behind the developed method is the use of gradient clipping to control stochastic gradient differences in recursive variance reduction. This allows us to bound the potential harm caused by Byzantine workers, even during iterations when all sampled clients are Byzantine. Furthermore, we incorporate communication compression into the method to enhance communication efficiency. Under general assumptions, we prove convergence rates for the proposed method that match the existing state-of-the-art (SOTA) theoretical results. We also propose a heuristic on adjusting any Byzantine-robust method to a partial participation scenario via clipping.
\end{abstract}

\section{Introduction}
\label{introduction}

Distributed optimization problems are a cornerstone of modern machine learning research. They naturally arise in scenarios where data is distributed across multiple clients; for instance, this is typical in Federated Learning (FL)~\citep{FEDLEARN, kairouz2021advances}. Such problems require specialized algorithms adapted to the distributed setup. Additionally, the adoption of distributed optimization methods is motivated by the sheer computational complexity involved in training modern machine learning models. Many models deal with massive datasets and intricate architectures, rendering training infeasible on a single machine \citep{gpt3costlambda}. Distributed methods, by parallelizing computations across multiple machines, offer a pragmatic solution to accelerate training and address these computational challenges, thus pushing the boundaries of machine learning capabilities.

To make distributed training accessible to the broader community, collaborative learning approaches have been actively studied in recent years \citep{volunteer_dl_async, hivemind_dmoe, atre2021distributed, dedloc}. In such applications, there is a high risk of the occurrence of so-called \emph{Byzantine workers} \citep{lamport1982byzantine, su2016fault}—participants who can violate the prescribed distributed algorithm/protocol either intentionally or simply because they are faulty. In general, such workers may even have access to some private data of certain participants and may collude to increase their impact on the training. Since the ultimate goal is to achieve robustness in the worst case, many papers in the field make no assumptions limiting the power of Byzantine workers. Clearly, in this scenario, standard distributed methods based on the averaging of received information (e.g., stochastic gradients) are not robust, even to a single Byzantine worker. Such a worker can send an arbitrarily large vector that can shift the method arbitrarily far from the solution. This aspect makes it non-trivial to design methods with provable robustness to Byzantines \citep{baruch2019little, xie2020fall}. Despite all the challenges, multiple methods are developed/analyzed in the literature \citep{alistarh2018byzantine, allen2020byzantine, wu2020federated, zhu2021broadcast, karimireddy2021learning, karimireddy2020byzantine, gorbunov2021secure, gorbunov2023variance, allouah2023fixing}.

However, literally, all existing methods with provable Byzantine robustness require \emph{the full participation of clients}. The requirement of full participation is impractical for modern distributed learning problems since they can have millions of clients \citep{bonawitz2017practical, niu2020billion}. In such scenarios, it is more natural to use sampling of clients to speed up the training. Moreover, some clients can be unavailable at certain moments, e.g., due to a poor connection, low battery, or simply because of the need to use the computing power for some other tasks. Although \emph{partial participation of clients} is a natural attribute of large-scale collaborative training, it is not studied under the presence of Byzantine workers. Moreover, this question is highly non-trivial: the existing methods can fail to converge if combined na\"ively with partial participation since Byzantine can form a majority during particular rounds and, thus, destroy the whole training with just one round of communications. \emph{Therefore, the field requires the development of new distributed methods that are provably robust to Byzantine attacks and can work with partial participation even when Byzantine workers form a majority during some rounds.}

\vspace{-0.8em}
\paragraph{Our Contributions}\label{section:contribution}

We develop Byzantine-tolerant Variance-Reduced \algname{MARINA} with Partial Participation (\algname{Byz-VR-MARINA-PP}, Algorithm~\ref{alg:byz_vr_marina}) -- the first distributed method having Byzantine robustness and allowing partial participation of clients. Our method uses variance reduction to handle Byzantine workers and clipping of stochastic gradient differences to bound the potential harm of Byzantine workers even when they form a majority during particular rounds of communication. To make the method even more communication efficient, we add communication compression. We prove the convergence of \algname{By-VR-MARINA-PP} for general smooth non-convex functions and Polyak-{\L}ojasiewicz functions. In the special case of full participation, our complexity bounds recover the ones for \algname{Byz-VR-MARINA} \citep{gorbunov2023variance} that are the current SOTA convergence results. Moreover, we prove that in some cases, partial participation is theoretically beneficial for \algname{By-VR-MARINA-PP}. We also propose a heuristic on how to use clipping to adapt any Byzantine-robust method to the partial participation setup and illustrate its performance in experiments.
\vspace{-0.5em}
\subsection{Related Work}
\vspace{-0.5em}

Below we overview closely related works. Additional discussion is deferred to Appendix~\ref{appendix:extra_related_work}.

\textbf{Byzantine robustness.} The primary vulnerability of standard distributed methods to Byzantine attacks lies in the aggregation rule: even one worker can arbitrarily distort the average. Therefore, many papers on Byzantine robustness focus on the application of robust aggregation rules, such as the geometric median \citep{pillutla2022robust}, coordinate-wise median, trimmed median \citep{yin2018byzantine}, Krum \citep{blanchard2017machine}, and Multi-Krum \citep{damaskinos2019aggregathor}. However, simply robustifying the aggregation rule is insufficient to achieve provable Byzantine robustness, as illustrated by \citet{baruch2019little} and \citet{xie2020fall}, who design special Byzantine attacks that can bypass standard defenses. This implies that more significant algorithmic changes are required to achieve Byzantine robustness, a point also formally proven by \citet{karimireddy2021learning}, who demonstrate that permutation-invariant algorithms -- i.e., algorithms independent of the order of stochastic gradients at each step -- cannot provably converge to any predefined accuracy in the presence of Byzantines.

\citet{wu2020federated} are the first who exploit variance reduction to tolerate Byzantine attacks. They propose and analyze the method called \algname{Byrd-SAGA}, which uses \algname{SAGA}-type \citep{defazio2014saga} gradient estimators on the good workers and geometric median for the aggregation. \citet{gorbunov2023variance} develop another variance-reduced method called \algname{Byz-VR-MARINA}, which is based on (conditionally biased) \algname{GeomSARAH}/\algname{PAGE}-type \citep{horvath2019stochastic, li2021page} gradient estimator and any robust aggregation in the sense of the definition from \citep{karimireddy2021learning, karimireddy2020byzantine}, and derive the improved convergence guarantees that are the current SOTA in the literature. There are also many other approaches and we discuss some of them in Appendix~\ref{appendix:extra_related_work}.

\textbf{Partial participation and client sampling.} In the context of Byzantine-robust learning, there exists one work that develops and analyzes the method with partial participation \citep{data2021byzantine}. However, this work relies on the restrictive assumption that the number of participating clients at each round is at least three times larger than the number of Byzantine workers. In this case, Byzantines cannot form a majority, and standard methods can be applied without any changes. In contrast, our method converges in more challenging scenarios, e.g., \algname{Byz-VR-MARINA-PP} provably converges even when the server samples one client, which can be Byzantine. The results from \citet{data2021byzantine} have some other noticeable limitations that we discuss in Appendix~\ref{appendix:extra_related_work}.

\section{Preliminaries}\label{section:prelim}

In this section, we formally introduce the problem, main definition, and assumptions used in the analysis. That is, we consider finite-sum distributed optimization problem\footnote{For simplicity, we assume that all regular workers have the same size of local datasets. Our analysis can be easily generalized to the case of different sizes of local datasets: this will affect only the value of $\cL_{\pm}$ from Assumption~\ref{assm:local} for some sampling strategies.}
\begin{align}
\label{eq:main_problem}
\textstyle
    \min_{x\in \R^d}\left\{f(x) \eqdef \frac{1}{G}\sum_{i\in \cG}f_i(x)\right\},\quad    f_i(x) \eqdef \frac{1}{m}\sum_{j=1}^m f_{i,j}(x) \quad \forall i\in \cG, 
\end{align}
where $\cG$ is a set of regular clients of size $G \eqdef |\cG|$. In the context of distributed learning, $f_i:\R^d \to \R$ corresponds to the loss function on the data of client $i$, and $f_{i,j}:\R^d \to \R$ is the loss computed on the $j$-th sample from the dataset of client $i$. Next, we assume that the set of all clients taking part in the training is $[n] = \{1,2,\ldots, n\}$ and $\cG \subseteq [n]$. The remaining clients $\cB \eqdef [n]\setminus \cG$ are Byzantine ones. We assume that $B \eqdef |\cB| \eqdef \deltar n \leq \delta n$, where $\deltar$ is an exact ratio of Byzantine workers and $\delta$ is a known upper bound for $\deltar$. We also assume that $0 \leq \deltar \leq \delta <\nicefrac{1}{2}$ since otherwise Byzantine workers form a majority and problem \eqref{eq:main_problem} becomes impossible to solve in general.

\textbf{Notation.} We use a standard notation for the literature on distributed stochastic optimization. Everywhere in the text $\|x\|$ denotes a standard $\ell_2$-norm of $x\in\R^d$, $\langle a, b \rangle$ refers to the standard inner product of vectors $a,b \in \R^d$. The clipping operator is defined as follows: $\clip_\lambda(x) \eqdef \min\{1, \nicefrac{\lambda}{\|x\|}\}x$ for $x\neq 0$ and $\clip_\lambda(0) \eqdef 0$. Finally, $\PP\{A\}$ denotes the probability of event $A$, $\EE[\xi]$ is the full expectation of random variable $\xi$, $\EE[\xi \mid A]$ is the expectation of $\xi$ conditioned on the event $A$. We also sometimes use $\EE_{k}[\xi]$ to denote an expectation of $\xi$ w.r.t.\ the randomness coming from step $k$.

\textbf{Robust aggregator.} We follow the definition from \citep{gorbunov2023variance} of $(\delta,c)$-robust aggregation, which is a generalization of the definitions proposed by \citet{karimireddy2021learning, karimireddy2020byzantine}.

\begin{definition}[$(\delta, c)$-Robust Aggregator] 
\label{def:aragg}
Assume that $\left\{x_1, x_2, \ldots, x_n\right\}$ is such that there exists a subset $\mathcal{G} \subseteq[n]$ of size $|\mathcal{G}|=G \geq(1-\delta) n$ for $\delta\leq \delta_{\max} < 0.5$ and there exists $\sigma \geq 0$ such that $\frac{1}{G(G-1)} \sum_{i, l \in \mathcal{G}} \mathbb{E}\left[\left\|x_i-x_l\right\|^2\right] \leq \sigma^2$ where the expectation is taken w.r.t.\ the randomness of $\left\{x_i\right\}_{i \in \mathcal{G}}$. We say that the quantity $\widehat{x}$ is $(\delta, c)$-Robust Aggregator $\left(\delta, c)\right.$-\texttt{RAgg}) and write $\widehat{x} = \texttt{RAgg}\left(x_1, \ldots, x_n\right)$ for some $c>0$, if the following inequality holds:
\begin{equation}
\textstyle
    \mathbb{E}\left[\|\widehat{x}-\bar{x}\|^2\right] \leq c \delta \sigma^2, \label{eq:robust_aggr_definition}
\end{equation}
where $\bar{x}\eqdef\frac{1}{|\mathcal{G}|} \sum_{i \in \mathcal{G}} x_i$. If additionally $\widehat{x}$ is computed without the knowledge of $\sigma^2$, we say that $\widehat{x}$ is $(\delta, c)$-Agnostic Robust Aggregator $\left(\delta, c)\right.$-\texttt{ARAgg} and write $\widehat{x} = \texttt{ARAgg}\left(x_1, \ldots, x_n\right)$.
\end{definition}

One can interpret the definition as follows. Ideally, we would like to filter out all Byzantine workers and compute just an average $\bar x$ over the set of good clients. However, this is impossible in general since we do not know apriori who are Byzantine workers. Instead of this, it is natural to expect that the aggregation rule approximates the ideal average up in a certain sense, e.g., in terms of the expected squared distance to $\bar x$. As \citet{karimireddy2021learning} formally show, in terms of such criterion ($\EE[\|\widehat x - \bar x\|^2]$), the definition of $\left(\delta, c)\right.$-\texttt{RAgg} cannot be improved (up to the numerical constant). Moreover, standard aggregators such as Krum \citep{blanchard2017machine}, geometric median, and coordinate-wise median do not satisfy Definition~\ref{def:aragg} \citep{karimireddy2021learning}, though another popular standard aggregation rule called coordinate-wise trimmed mean \citep{yin2018byzantine} satisfies Definition~\ref{def:aragg} as shown by \citet{allouah2023fixing} through the more general definition of robust aggregation. To address this issue, \citet{karimireddy2021learning} develop the aggregator called \algname{CenteredClip} and prove that it fits the definition of $\left(\delta, c)\right.$-\texttt{RAgg}. \citet{karimireddy2020byzantine} propose a procedure called \algname{Bucketing} that fixes Krum, geometric median, and coordinate-wise median, i.e., with \algname{Bucketing} Krum, geometric, and coordinate-wise median become $\left(\delta, c)\right.$-\texttt{ARAgg}, which is important for our algorithm since the variance of the vectors received from regular workers changes over time in our method. We notice here that $\delta$ is a part of the input that should satisfy $\deltar \leq \delta \leq \delta_{\max}$.

\textbf{Compression operators.} In our work, we use standard unbiased compression operators with relatively bounded variance \citep{khirirat2018distributed, horvath2019stochastic}.

\begin{definition}[Unbiased compression]
\label{def:Q}
Stochastic mapping $\mathcal{Q}: \mathbb{R}^d \rightarrow \mathbb{R}^d$ is called unbiased compressor/compression operator if there exists $\omega \geq 0$ such that for any $x \in \mathbb{R}^d$
$\mathbb{E}[\mathcal{Q}(x)]=x, \quad \mathbb{E}\left[\|\mathcal{Q}(x)-x\|^2\right] \leq \omega\|x\|^2 .$
For the given unbiased compressor $\mathcal{Q}(x)$, one can define the expected density\footnote{This quantity is well-suited for sparsification-type compression operators like random sparsification \citep{stich2018sparsified} and $1$-level $\ell_2$-quantization \citep{alistarh2017qsgd}. For other compressors, such as quantization with more than one level \citep{goodall1951television, roberts1962picture},  $\zeta_{\cQ}$ is not the main characteristic describing their properties.} as $\zeta_{\mathcal{Q}}\eqdef$ $\sup _{x \in \mathbb{R}^d} \mathbb{E}\left[\|\mathcal{Q}(x)\|_0\right]$, where $\|y\|_0$ is the number of non-zero components of $y \in \mathbb{R}^d$.
\end{definition}

In this definition, parameter $\omega$ reflects how lossy the compression operator is: the larger $\omega$ the more lossy the compression. For example, this class of compression operators includes random sparsification (RandK) \citep{stich2018sparsified} and quantization \citep{goodall1951television, roberts1962picture, alistarh2017qsgd}. For RandK compression $\omega = \frac{d}{K} - 1, \zeta_{\cQ} = K$ and for $\ell_2$-quantization $\omega = \sqrt{d}-1, \zeta_{\cQ} = \sqrt{d}$, see the proofs in \citep{beznosikov2020biased}.

\textbf{Assumptions.} Up to a couple of assumptions that are specific to our work, we use the same assumptions as in \citep{gorbunov2023variance}. We start with two new assumptions.

\begin{assumption}[Bounded \texttt{ARAgg}]
\label{assm:bounded-aggr}
    We assume that the server applies aggregation rule $\cA$ such that $\cA$ is $(\delta,c)$-\texttt{ARAgg} and there exists constant $F_{\cA} > 0$ such that for any inputs $x_1,\ldots, x_n \in \R^d$ the norm of the aggregator is not greater than the maximal norm of the inputs:
        $\left\| \cA\left(x_1, \ldots, x_n\right)  \right\| \leq F_{\cA} \max_{i\in [n]} \|x_i\|.$
\end{assumption}

The above assumption is satisfied for popular $(\delta,c)$-robust aggregation rules presented in the literature \citep{karimireddy2021learning, karimireddy2020byzantine}. Therefore, this assumption is more a formality than a real limitation: it is needed to exclude some pathological examples of $(\delta,c)$-robust aggregation rules, e.g., for any $\cA$ that is $(\delta,c)$-\texttt{RAgg} one can construct unbounded $(\delta,2c)$-\texttt{RAgg} as $\overline\cA = \cA + X$, where $X$ is a random sample from the Gaussian distribution $\cN(0, c\delta\sigma^2)$.

Next, for part of our results, we also make the following assumption.

\begin{assumption}[Bounded compressor (optional)]
\label{assm:bounded-compressor}
    We assume that workers use compression operator $\cQ$ satisfying Definition~\ref{def:Q} and bounded as follows:
        $\left\| \cQ(x)\right\| \leq D_{Q} \|x\| \quad \forall x \in \R^d.$
\end{assumption}

For example, RandK and $\ell_2$-quantization meet this assumption with $D_\cQ = \frac{d}{K}$ and $D_{\cQ} = \sqrt{d}$ respectively. In general, constant $D_{\cQ}$ can be large (proportional to $d$). However, in practice, one can use RandK with $K = \frac{d}{100}$ and, thus, have moderate $D_{\cQ} = 100$. We also have the results without Assumption~\ref{assm:bounded-compressor}, but with worse dependence on some other parameters, see  Section~\ref{section:convergence_results}.

Next, we assume that good workers have $\zeta^2$-heterogeneous local loss functions.

\begin{assumption}[$\zeta^2$-heterogeneity] 
\label{assm:het_simplified}
We assume that good clients have $\zeta^2$-heterogeneous local loss functions for some $\zeta \geq 0$, i.e.,
$
\frac{1}{G} \sum_{i \in \mathcal{G}}\left\|\nabla f_i(x)-\nabla f(x)\right\|^2 \leq \zeta^2 \quad \forall x \in \mathbb{R}^d.
$
\end{assumption}

The above assumption is quite standard for the literature on Byzantine robustness \citep{wu2020federated, karimireddy2020byzantine, gorbunov2023variance, allouah2023fixing}. Moreover, some kind of a bound on the heterogeneity of good clients is necessary since otherwise Byzantine robustness cannot be achieved in general. In the appendix, all proofs are given under a more general version of Assumption~\ref{assm:het_simplified}, see Assumption~\ref{assm:het}. Finally, the case of homogeneous data ($\zeta = 0$) is also quite popular for collaborative learning \citep{diskin2021distributed, kijsipongse2018hybrid}.

The following assumption is classical for the literature on non-convex optimization.
\begin{assumption}[Smoothness (simplified)]
\label{assm:smoothness_simplified}
We assume that for all $i\in \cG$ and $j\in [m]$ there exists $\cL \geq 0$ such that $f_{i,j}$ is $\cL$-smooth, i.e., for all $x, y \in \mathbb{R}^d$
\begin{equation}
\textstyle
    \|\nabla f_{i,j}(x) - \nabla f_{i,j}(y)\| \leq \cL \|x - y\|. \label{eq:f_i_j_smooth_simple}
\end{equation}
Moreover, we assume that $f$ is uniformly lower bounded by $f^* \in \mathbb{R}$, i.e., $f^*\eqdef\inf _{x \in \mathbb{R}^d} f(x)$.
\end{assumption}
For the sake of simplicity, we do not differentiate between various notions of smoothness in the main text. However, our analysis takes into account the differences between smoothness constants, similarity of local functions, and sampling strategy (see Appendix~\ref{appendix:refined_assumptions}).

\begin{algorithm*}[t]
   \caption{\algname{Byz-VR-MARINA-PP}: Byzantine-tolerant \algname{VR-MARINA} with Partial Participation}\label{alg:byz_vr_marina}
\begin{algorithmic}[1]
   \STATE {\bfseries Input:} vectors $x^0, g^0 \in \R^d$, stepsize $\gamma$, minibatch size $b$, probability $p\in(0,1]$, number of iterations $K$, $(\delta,c)$-\texttt{ARAgg}, clients' sample size $1 \leq C \leq \widehat{C} \leq n$, clipping coefficients $\{\alpha_{k}\}_{k\geq 1}$ 
   \FOR{$k=0,1,\ldots,K-1$}
   \STATE Get a sample from Bernoulli distribution with parameter $p$: $c_k \sim \text{Be}(p)$
   \STATE Sample the set of clients $S_k \subseteq [n]$, $|S_k| = C$ if $c_k = 0$; otherwise $|S_k| = \widehat{C}$
   \STATE Broadcast $g^k$, $c_k$ to all workers
   \FOR{$i \in \cG \cap S_k$ in parallel} 
   \STATE $x^{k+1} = x^k - \gamma g^k$ and $\lambda_{k+1} = \alpha_{k+1}\|x^{k+1} - x^k\|$
   \STATE  Set $g_i^{k+1} = \begin{cases} \nabla f_i(x^{k+1}),& \text{if } c_k = 1,\\ g^k + \clip_{\lambda_{k+1}}\left(\cQ\left(\widehat{\Delta}_i(x^{k+1}, x^k)\right)\right),& \text{otherwise,}\end{cases}$\\ where $\widehat{\Delta}_i(x^{k+1}, x^k)$ is a minibatched estimator of $\nabla f_i(x^{k+1}) - \nabla f_i(x^k)$, $\cQ(\cdot)$ for $i\in\cG \cap S_k$ are computed independently
   \ENDFOR
   \STATE $g^{k+1} = \begin{cases} \texttt{ARAgg}\left(\{g_i^{k+1}\}_{i\in S_k}\right),& \text{if } c_k = 1, \\
       g^k + \texttt{ARAgg}\left(\left\{\clip_{\lambda_{k+1}}\left(\cQ\left(\widehat{\Delta}_i(x^{k+1}, x^k)\right)\right)\right\}_{i\in S_k}\right),&\text{otherwise}\end{cases}$
   \ENDFOR
\end{algorithmic}
\end{algorithm*}

Finally, we also consider functions satisfying Polyak-Łojasiewicz (PŁ) condition \citep{polyak1963gradient, lojasiewicz1963topological}. This assumption belongs to the class of assumptions on the structured non-convexity that allows achieving linear convergence \citep{necoara2019linear}.

\begin{assumption}[PŁ condition (optional)]
\label{assm:PL}
We assume that function $f$ satisfies Polyak-Łojasiewicz (PŁ) condition with parameter $\mu > 0$, i.e., for all $x \in \mathbb{R}^d$ there exists  $f^*\eqdef\inf _{x \in \mathbb{R}^d} f(x)$ such that
  $ \|\nabla f(x)\|^2 \geq 2 \mu\left(f(x)-f^*\right) .$

\end{assumption}
\section{New Method: \algname{Byz-VR-MARINA-PP}}\label{section:method}

We propose a new method called Byzantine-tolerant Variance-Reduced \algname{MARINA} with Partial Participation (\algname{Byz-VR-MARINA-PP}, Algorithm~\ref{alg:byz_vr_marina}). Our method extends \algname{Byz-VR-MARINA} \citep{gorbunov2023variance} to the partial participation case via the proper usage of the clipping operator. To illustrate how \algname{Byz-VR-MARINA-PP} works, we first consider a special case of full participation.

\textbf{Special case: \algname{Byz-VR-MARINA}.} If all clients participate at each round ($S_k \equiv [n]$) and clipping is turned off ($\lambda_{k} \equiv +\infty$), then \algname{Byz-VR-MARINA-PP} reduces to \algname{Byz-VR-MARINA} that works as follows. Consider the case when no compression is applied ($\cQ(x) = x$) and $\widehat{\Delta}_i(x^{k+1}, x^k) = \nabla f_{i,j_k}(x^{k+1}) - \nabla f_{i,j_k}(x^{k})$, where $j_k$ is sampled uniformly at random from $[m]$, $i\in \cG$. Then, regular workers compute \algname{GeomSARAH}/\algname{PAGE} gradient estimator at each step: for $i\in \cG$
\begin{equation*}
\textstyle
    g_i^{k+1} = \begin{cases}
        \nabla f_i(x^{k+1}), \quad \text{with probability } p,\\ g^k + \nabla f_{i,j_k}(x^{k+1}) - \nabla f_{i,j_k}(x^{k}),\quad\text{otherwise} 
    \end{cases}
\end{equation*}
With small probability $p$, good workers compute full gradients, and with larger probability $1-p$ they update their estimator via adding stochastic gradient difference. To balance the oracle cost of these two cases, one can choose $p \sim \nicefrac{1}{m}$ (for $b$-size minibatched estimator -- $p\sim \nicefrac{b}{m}$). Such estimators are known to be optimal for finding stationary points in the stochastic first-order optimization \citep{fang2018spider, arjevani2023lower}. Next, good workers send $g_i^{k+1}$ or $\nabla f_{i,j_k}(x^{k+1}) - \nabla f_{i,j_k}(x^{k})$ to the server who robustly aggregate the received vectors. Since estimators are conditionally biased, i.e., $\EE[g_i^{k+1}\mid x^{k+1}, x^k] \neq \nabla f_i(x^{k+1})$, the additional bias coming from the aggregation does not cause significant issues in the analysis or practice. Moreover, the variance of $\{g_i^{k+1}\}_{i\in \cG}$ w.r.t.\ the sampling of the stochastic gradients is proportional to $\|x^{k+1} - x^k\|^2 \to 0$ with probability $1-p$ (due to Assumption~\ref{assm:local}) that progressively limits the effect of Byzantine attacks. For a more detailed explanation of why recursive variance reduction works better than \algname{SAGA}/\algname{SVRG}-type variance reduction, we refer to \citep{gorbunov2023variance}. Arbitrary sampling allows the improvement of the dependence on the smoothness constants. Unbiased communication compression also naturally fits the framework since it is applied to the stochastic gradient difference, meaning that the variance of $\{g_i^{k+1}\}_{i\in \cG}$ w.r.t.\ the sampling of the stochastic gradients and compression remains proportional to $\|x^{k+1} - x^k\|^2$ with probability $1-p$.

\textbf{New ingredients: client sampling and clipping.} The algorithmic novelty of \algname{Byz-VR-MARINA-PP} in comparison to \algname{Byz-VR-MARINA} is twofold: with (typically large) probability $1-p$ only $C$ clients sampled uniformly at random from the set of all clients participate at each round, and clipping is applied to the compressed stochastic gradient differences. With a small probability $p$, a larger number\footnote{As we show next, it is sufficient to take $\widehat C \geq \max\{1, \nicefrac{\deltar n}{\delta}\}$ similarly to \citep{data2021byzantine}. However, in contrast to the approach from \citet{data2021byzantine}, \algname{Byz-VR-MARINA-PP} requires such communications only with small probability $p$.} of clients $\widehat C \leq n$ takes part in the communication. The main role of clipping is to ensure that the method can withstand the attacks of Byzantines when they form a majority or, more precisely when there are more than $\delta C$ Byzantine workers among the sampled ones. \emph{Indeed, without clipping (or some other algorithmic changes) such situations are critical for convergence: Byzantine workers can shift the method arbitrarily far from the solution, e.g., they can collectively send some vector with the arbitrarily large norm}. In contrast, \algname{Byz-VR-MARINA-PP} tolerates any attacks even when all sampled clients are Byzantine workers since the update remains bounded due to the clipping. Via choosing $\lambda_{k+1}\sim \|x^{k+1} - x^k\|$ we ensure that the norm of transmitted vectors decreases with the same rate as it does in \algname{Byz-VR-MARINA} with full client participation. Finally, with probability $1-p$ regular workers can transmit just compressed vectors and leave the clipping operation to the server since Byzantines can ignore clipping operation.

\section{Convergence Results}\label{section:convergence_results}

We define $\cG_C^k = \cG\cap S_k$ and $G_{C}^k = |\cG_C^k|$ and $\binom{n}{k} = \frac{n !}{k !(n-k) !}$ represents the binomial coefficient. We also use the following probabilities:
\begin{align*}
\textstyle
p_G &\eqdef  \PP\left\lbrace G_C^k \geq\left(1-\delta\right) C\right\rbrace = \textstyle\sum_{\lceil(1-\delta)C\rceil\leq t \leq C} \nicefrac{\binom{G}{t}\binom{n-G}{C-t}}{\binom{n}{C}},\\
\mathcal{P}_{\mathcal{G}^k_C} &\eqdef  \PP\left\lbrace i \in \mathcal{G}_C^k \mid G_C^k \geq\left(1-\delta\right) C\right\rbrace = \nicefrac{C}{np_G} \cdot \textstyle\sum_{\lceil(1-\delta)C\rceil\leq t \leq C} \nicefrac{\binom{G-1}{t-1}\binom{n-G}{C-t}}{\binom{n-1}{C-1}}.
\end{align*}
These probabilities naturally appear in the analysis and statements of the theorems. When $c_k = 0$, then server samples $C$ clients, and two situations can appear: either $G_C^k$ is at least $\left(1-\delta\right) C$ meaning that the aggregator can ensure robustness according to Definition~\ref{def:aragg} or $G_C^k < \left(1-\delta\right) C$. Probability $p_G$ is the probability of the first event, and the second event implies that the aggregation can be spoiled by Byzantine workers (but clipping bounds the ``harm''). Finally, we use $\mathcal{P}_{\mathcal{G}^k_C}$ in the computation of some conditional expectations when the first event occurs. The mentioned probabilities can be easily computed for some special cases. For example, if $C = 1$, then $p_G = \nicefrac{G}{n}$ and $\mathcal{P}_{\mathcal{G}^k_C} = \nicefrac{1}{G}$; if $C = 2$, then $p_G = \nicefrac{G(G-1)}{n(n-1)}$ and $\mathcal{P}_{\mathcal{G}^k_C} = \nicefrac{2}{G}$; finally, if $C = n$, then  $p_G = 1$ and $\mathcal{P}_{\mathcal{G}^k_C} = 1$.

The next theorem is our main convergence result for general unbiased compression operators.
\begin{theorem}\label{thm:main_result_1}
 Let Assumptions \ref{assm:bounded-aggr}, \ref{assm:het_simplified}, \ref{assm:smoothness_simplified} hold and $\lambda_{k+1} = 2\cL \left\|x^{k+1} - x^k\right\|$. Assume that $0<\gamma \leq \nicefrac{1}{\cL(1+\sqrt{A})},$ 
where constant $A$ is defined as 
\begin{eqnarray}
\textstyle
    A &=& \frac{32p_G G\mathcal{P}_{\mathcal{G}^k_C}}{p^2(1-\delta)C} \left( 30\omega + 11\right) (1 + 2c\delta) + \frac{16(1-p_G)(1+4F_{\cA}^2)}{p^2}. \label{eq:A_unbounded_compr}
\end{eqnarray}
Then for all $K \geq 0$ the iterates produced by \algname{Byz-VR-MARINA-PP} (Algorithm \ref{alg:byz_vr_marina}) satisfy
\begin{equation}
\textstyle
    \mathbb{E}\left[\left\|\nabla f\left(\widehat{x}^K\right)\right\|^2\right] \leq \frac{2 \Phi^0}{\gamma(K+1)}+\frac{4 \widehat{D} \zeta^2}{p}, \label{eq:main_result_non_convex}
\end{equation}
where $\widehat{D} = \frac{ 2\delta\mathcal{P}_{\mathcal{G}^k_{\widehat{C}}} }{1-\delta} \left(\frac{6cG}{\widehat C} + p \right)$ and $\widehat{x}^K$ is chosen uniformly at random from $x^0, x^1, \ldots, x^K$, and $\Phi^0=$ $f\left(x^0\right)-f^*+\frac{2\gamma}{p}\left\|g^0-\nabla f\left(x^0\right)\right\|^2$. If, in addition, Assumption~\ref{assm:PL} holds and $0<\gamma \leq \nicefrac{1}{\cL(1+\sqrt{2 A})},$ then for all $K \geq 0$ the iterates produced by \algname{Byz-VR-MARINA-PP} (Algorithm \ref{alg:byz_vr_marina}) with $\rho = \min\left\{\gamma\mu, \frac{p}{8}\right\}$ satisfy
\begin{equation}
\textstyle
    \mathbb{E}\left[f\left(x^K\right)-f\left(x^*\right)\right] \leq\left(1-\rho\right)^K \Phi^0+\frac{4\widehat{D}\zeta^2\gamma}{p\rho}, \label{eq:main_result_PL}
\end{equation}
where $\Phi^0=$ $f\left(x^0\right)-f^*+\frac{4\gamma}{p}\left\|g^0-\nabla f\left(x^0\right)\right\|^2$.
\end{theorem}

The above theorem establishes similar guarantees to the current SOTA ones obtained for \algname{Byz-VR-MARINA}. That is, in the general non-convex case, we prove $\cO(\nicefrac{1}{K})$ rate, which is optimal \citep{arjevani2023lower}, and for P\L-functions we derive linear convergence result to the neighborhood depending on the heterogeneity. The size of this neighborhood matches the one derived for \algname{Byz-VR-MARINA} by \citet{gorbunov2023variance}. However, since our result is obtained considering the challenging scenario of partial participation of clients, the maximal theoretically allowed stepsize in our analysis of \algname{Byz-VR-MARINA-PP} is smaller than the one from \citep{gorbunov2023variance}.


In particular, the second term in the constant $A$ appears due to the partial participation, and the whole expression for $A$ is proportional to $\nicefrac{1}{p^2}$. In contrast, a similar constant $A$ from the result for \algname{Byz-VR-MARINA} is proportional to $\nicefrac{1}{p}$, which can be noticeably smaller than $\nicefrac{1}{p^2}$. Indeed, to make the expected number of clients participating in the communication round equal to $\cO(C)$, to make the expected number of stochastic oracle calls equal to $\cO(b)$, and to make the expected number of transmitted components for each worker taking part in the communication round equal $\cO(\zeta_{\cQ})$, parameter $p$ should be chosen as $p = \min\{\nicefrac{C}{n}, \nicefrac{b}{m}, \nicefrac{\zeta_{\cQ}}{d}\}$, where the latter term in the minimum often equals to $\Theta(\nicefrac{1}{(\omega+1)})$ \citep{gorbunov2021marina}. Therefore, in some scenarios, $p$ can be small.

Next, in the special case of full participation, we have $C=\widehat C=n$, $p_G = \cP_{\cG_{C}^k} = 1$, meaning that $A = \Theta(\nicefrac{(1+\omega)(1+c\delta)}{p^2})$ for \algname{Byz-VR-MARINA-PP}. In contrast, the corresponding constant for \algname{Byz-VR-MARINA} is of the order $\Theta(\nicefrac{(1+\omega)}{pn} + \nicefrac{(1+\omega)c\delta}{p^2})$, which is strictly better than our bound. In this special case, we do not recover the result for \algname{Byz-VR-MARINA}.

Such a complexity deterioration can be explained as follows: the presence of clipping introduces additional technical difficulties in the analysis, resulting in a reduced step size compared to \algname{Byz-VR-MARINA}, even when $C = \widehat{C} = n$. To achieve a more favorable convergence rate, particularly in scenarios of complete participation, we also establish the results under Assumption~\ref{assm:bounded-compressor}.
\begin{theorem}\label{thm:main_result_2}
 Let Assumptions \ref{assm:bounded-aggr}, \ref{assm:bounded-compressor}, \ref{assm:het_simplified}, \ref{assm:smoothness_simplified} hold and $\lambda_{k+1} = D_Q\cL \left\|x^{k+1} - x^k\right\|$. Assume that $0<\gamma \leq \nicefrac{1}{\cL(1+\sqrt{A})},$ 
where constant $A$ equals
\begin{eqnarray}
\textstyle
    A &=& \frac{4p_G  G\mathcal{P}_{\mathcal{G}^k_C}}{p(1-\delta)C}\left( \frac{3\omega + 2}{(1-\delta)C} + \frac{8(5\omega + 4)c\delta}{p}\right) + \frac{8(1-p_G) (2 + F_{\cA}^2 D_Q^2)}{p^2}.\label{eq:A_bounded_compr}
\end{eqnarray}
Then for all $K \geq 0$ the iterates produced by \algname{Byz-VR-MARINA-PP} (Algorithm \ref{alg:byz_vr_marina}) satisfy
\begin{equation}
\textstyle
    \mathbb{E}\left[\left\|\nabla f\left(\widehat{x}^K\right)\right\|^2\right] \leq \frac{2 \Phi^0}{\gamma(K+1)}+\frac{2 \widehat{D} \zeta^2}{p}, \label{eq:main_result_non_convex-q}
\end{equation}
where $\widehat{D} = \frac{ 2\delta\mathcal{P}_{\mathcal{G}^k_{\widehat{C}}} }{1-\delta} \left(\frac{6cG}{\widehat C} + p \right)$ and $\widehat{x}^K$ is chosen uniformly at random from $x^0, x^1, \ldots, x^K$, and $\Phi^0=$ $f\left(x^0\right)-f^*+\frac{\gamma}{p}\left\|g^0-\nabla f\left(x^0\right)\right\|^2$ . If, in addition, Assumption~\ref{assm:PL} holds and $0<\gamma \leq \nicefrac{1}{\cL(1+\sqrt{2 A})},$ then for all $K \geq 0$ the iterates produced by \algname{Byz-VR-MARINA-PP} (Algorithm \ref{alg:byz_vr_marina}) satisfy with $\rho = \min\left\{\gamma\mu, \frac{p}{4}\right\}$
\begin{equation}
\textstyle
    \mathbb{E}\left[f\left(x^K\right)-f\left(x^*\right)\right] \leq\left(1-\rho\right)^K \Phi^0+\frac{2\widehat{D}\zeta^2\gamma}{p\rho}, \label{eq:main_result_PL-q}
\end{equation}
where $\Phi^0=$ $f\left(x^0\right)-f^*+\frac{2\gamma}{p}\left\|g^0-\nabla f\left(x^0\right)\right\|^2$.
\end{theorem}

With Assumption~\ref{assm:bounded-compressor}, vectors $\{\cQ(\widehat{\Delta}_i(x^{k+1}, x^k))\}_{i\in \cG_C^k}$ can be upper bounded by $D_Q\cL \left\|x^{k+1} - x^k\right\|$. Using this fact, one can take the clipping level sufficiently large such that it is turned off for the regular workers. This allows us to simplify the proof and remove $\nicefrac{1}{p}$ factor in front of the terms not proportional to $\delta$ or to $1-p_G$ in the expression for $A$ that can make the stepsize larger. However, the second term in \eqref{eq:A_bounded_compr} can be larger than \eqref{eq:A_unbounded_compr}, since it depends on potentially large constant $D_Q$. Therefore, the rates of convergence from Theorems \ref{thm:main_result_1} and \ref{thm:main_result_2} cannot be compared directly. We also highlight that the clipping level from Theorem~\ref{thm:main_result_2} is in general larger than the clipping level from Theorem~\ref{thm:main_result_1} and, thus, it is expected that with full participation Theorem~\ref{thm:main_result_2} gives better results than Theorem~\ref{thm:main_result_1}: the bias introduced due to the clipping becomes smaller with the increase of the clipping level. However, in the partial participation regime, the price for this is a decrease of the stepsize to compensate for the increased harm from Byzantine clients in situations when they form a majority. Further discussion of the technical challenges we overcame is deferred to Appendix~\ref{appendix:technical_challenges}.

Nevertheless, in the case of full participation, we have $C=\widehat C=n$, $p_G = \cP_{\cG_{C}^k} = \cP_{\cG_{\widehat C}^k} = 1$, meaning that $A = \Theta(\nicefrac{(1+\omega)}{pn} + \nicefrac{(1+\omega)c\delta}{p^2})$ in Theorem~\ref{thm:main_result_2}. That is, in this case, we recover the result of \algname{Byz-VR-MARINA}. More generally, if $p_G = 1$, which is equivalent to $C \geq \max\{1, \nicefrac{\deltar n}{\delta}\}$, then $\cP_{\cG_C^k} = \PP\{i \in \cG_C^k\} = \min\{1, \nicefrac{C}{G}\}$, $\cP_{\cG_C^k} = \PP\{i \in \cG_{\widehat C}^k\} = \min\{1, \nicefrac{\widehat C}{G}\}$ and we have $A = \Theta(\nicefrac{(1+\omega)}{p C} + \nicefrac{(1+\omega)c\delta}{p^2})$ and $\widehat D = \Theta(c\delta)$. Here, the first term in $A$ is $\nicefrac{n}{C}$ worse than the corresponding term for \algname{Byz-VR-MARINA}. However, the second term in $A$ matches the corresponding term for \algname{Byz-VR-MARINA}. Moreover, this term is the main one if $c\delta \geq \nicefrac{p}{C}$, which is typically the case since parameter $p$ is often small ($p = \min\{\nicefrac{C}{\widehat C}, \nicefrac{b}{m}, \nicefrac{\zeta_{\cQ}}{d}\}$). In such cases, \algname{Byz-VR-MARINA-PP} has the same rate of convergence as \algname{Byz-VR-MARINA} while utilizing, on average, just $\cO(C)$ workers at each step in contrast to \algname{Byz-VR-MARINA} that uses $n$ workers at each step. \emph{That is, in some cases, partial participation is provably beneficial for \algname{Byz-VR-MARINA-PP}.} 

\paragraph{Heuristic extension of \algname{Byz-VR-MARINA-PP}.} In this short remark, we illustrate how the proposed clipping technique can be applied to a general class of Byzantine-robust methods to adapt them to the case of partial participation. Consider the methods having the following update rule: $x^{k+1} = x^k - \gamma \cdot \texttt{Agg}(\{g_i^k\}_{i\in[n]})$, where $\{g_i^k\}_{i\in[n]}$ are the vectors received from workers at iteration $k$ and $\texttt{Agg}$ is some aggregation rule. A vast majority of existing Byzantine-robust methods fit this scheme. In the case of partial participation of clients, we propose to modify the scheme as follows:
\begin{equation}
    x^{k+1} = x^k - \gamma g^k,\quad \text{where } g^k \eqdef g^{k-1} + \texttt{Agg}\left(\left\{\clip_{\lambda_k}(g_i^k - g^{k-1})\right\}_{i\in S_k}\right), \label{eq:heuristic}
\end{equation}
where $S_k \subseteq [n]$ is a subset of clients participating in round $k$ and $\{\lambda_k\}_{k\geq 0}$ is sequence of clipping parameters specified by the server. In particular, \algname{Byz-VR-MARINA-PP} can be seen as an application of scheme \eqref{eq:heuristic} to \algname{Byz-VR-MARINA} (up to a minor modification when $c_k = 1$ in \algname{Byz-VR-MARINA}) with $\lambda_{k+1} = \lambda\|x^{k+1} - x^{k}\|$. We suggest to use $\lambda_{k+1} = \lambda\|x^{k+1} - x^{k}\|$ with tunable parameter $\lambda > 0$ for other methods as well.

\begin{figure*}[t]
\centering
\includegraphics[width=0.32\textwidth]{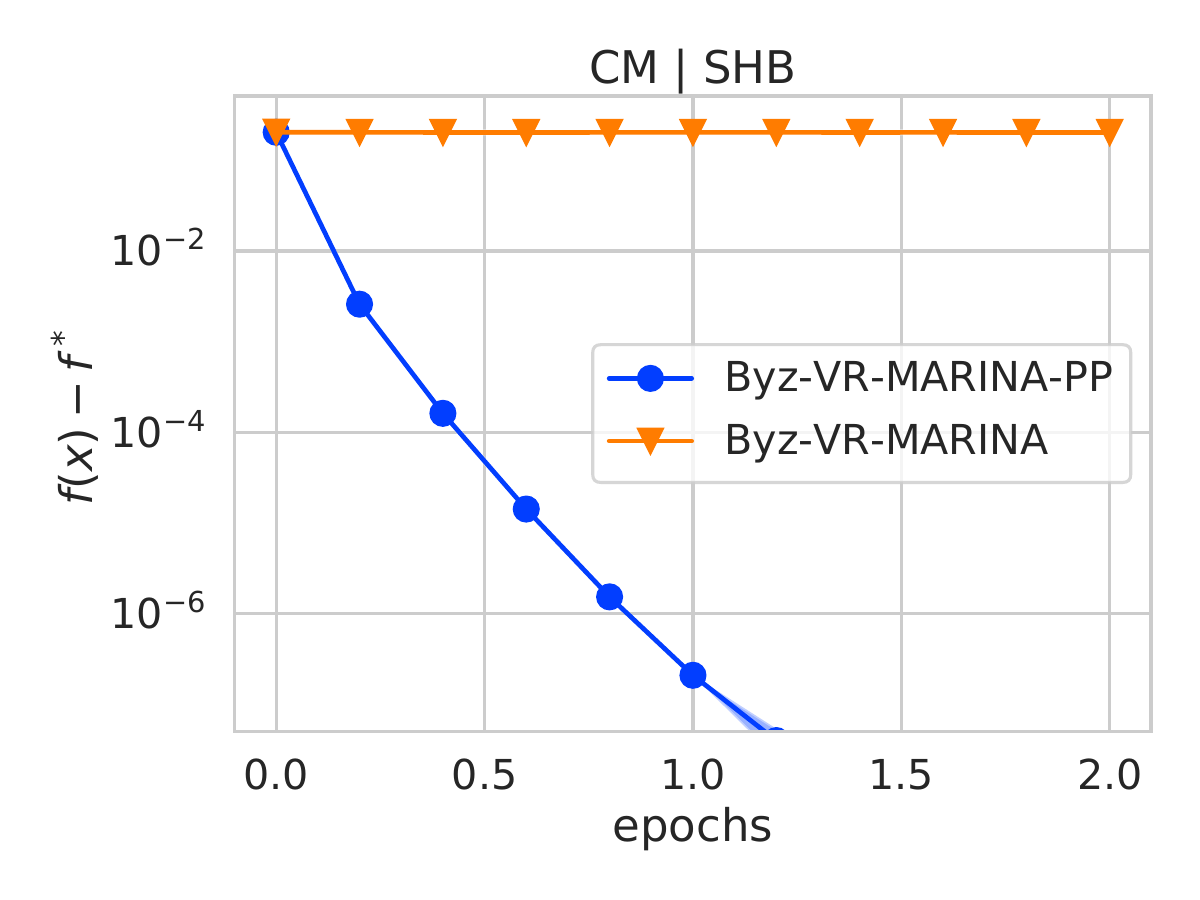}
\includegraphics[width=0.32\textwidth]{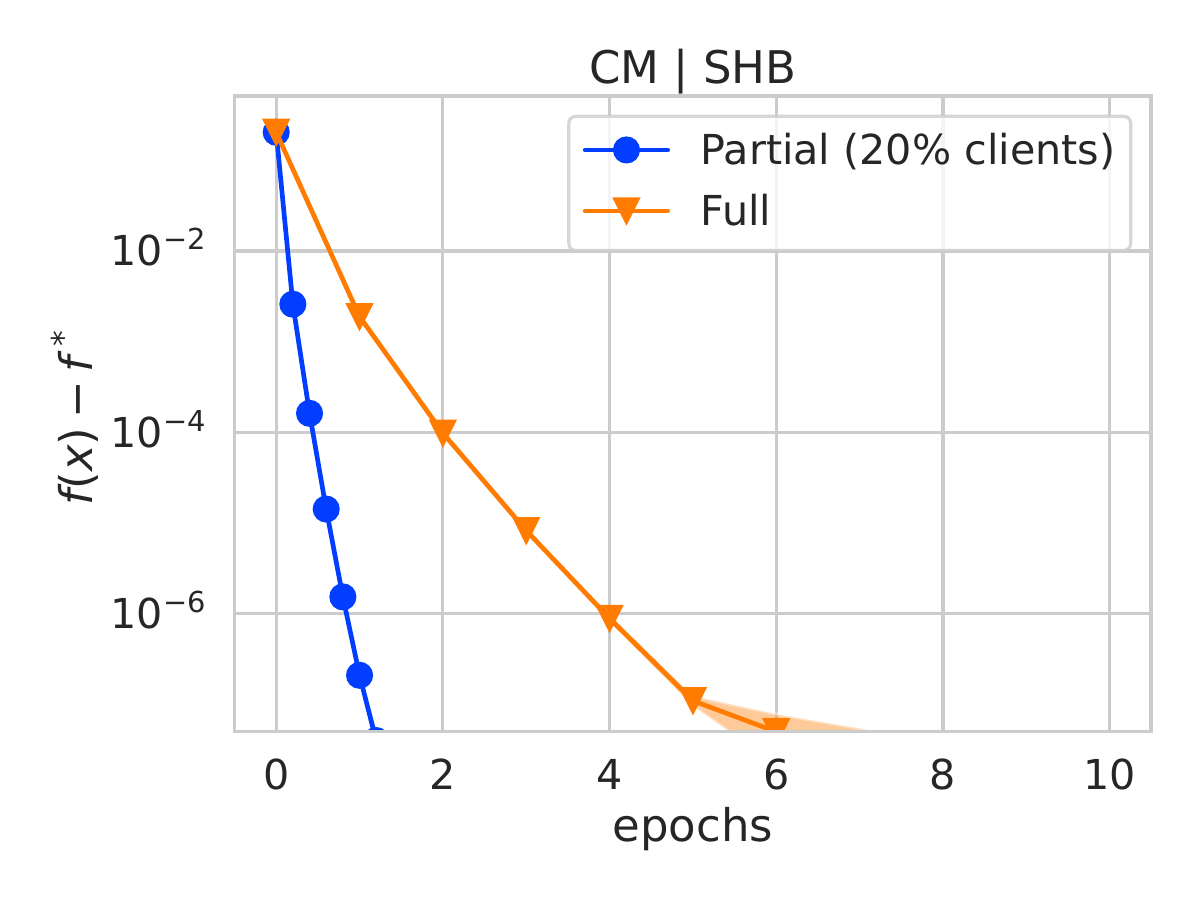}
\includegraphics[width=0.32\textwidth]{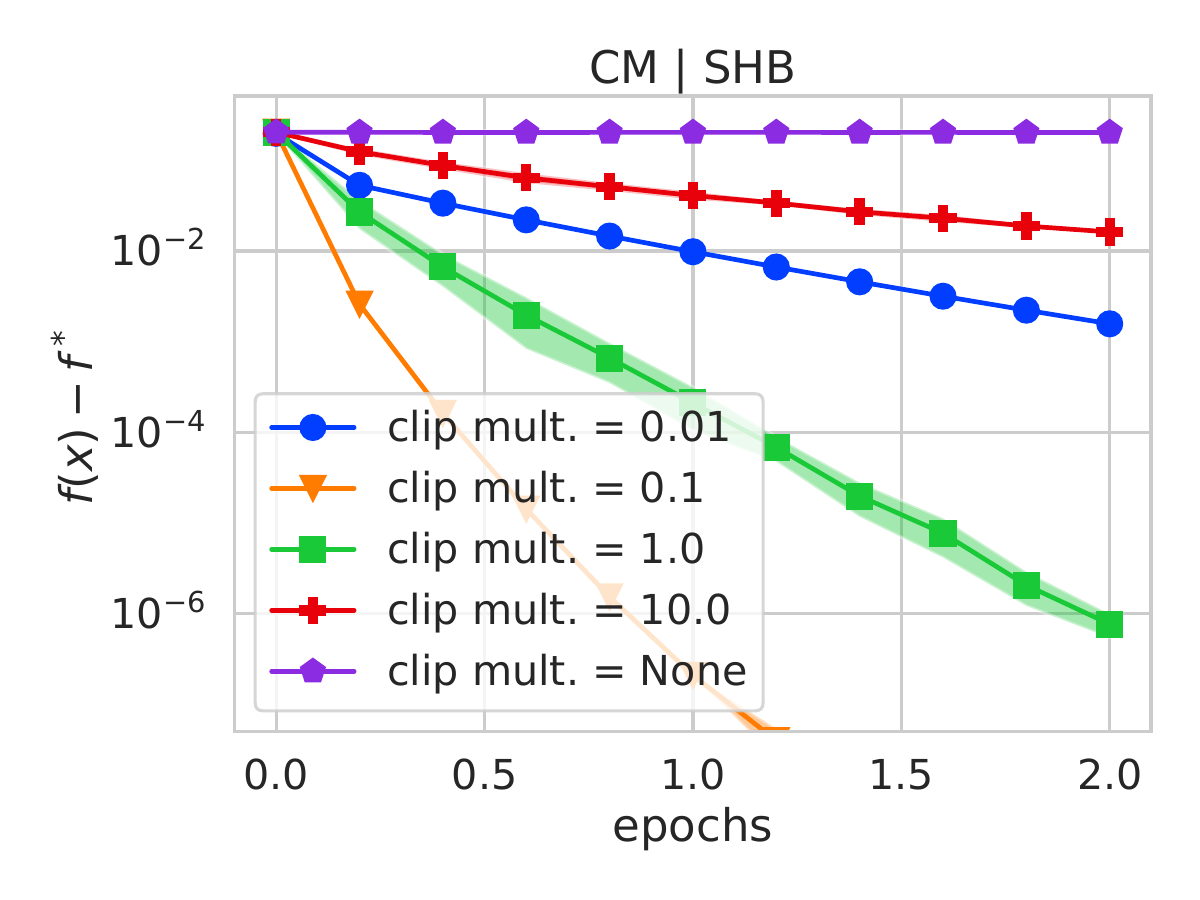}
\caption{
The optimality gap $f(x^k) - f(x^*)$ for 3 different scenarios. We use coordinate-wise mean with bucketing equal to 2 as an aggregation and shift-back as an attack.  We use the a9a dataset, where each worker accesses the full dataset with 15 good and 5 Byzantine workers. We do not use any compression. In each step, we sample 20\% of clients uniformly at random to participate in the given round unless we specifically mention that we use full participation. Left: Linear convergence of \algname{Byz-VR-MARINA-PP} with clipping versus non-convergence without clipping. Middle: Full versus partial participation, showing faster convergence with clipping. Right: Clipping multiplier $\lambda$ sensitivity, demonstrating consistent linear convergence across varying $\lambda$ values.} 
\label{fig:a9a_full}
\end{figure*}

\vspace{-0.5em}
\section{Numerical Experiments}\label{section:experiments}
\vspace{-0.5em}

Firstly, we showcase the benefits of employing clipping to remedy the presence of Byzantine workers and partial participation. For this task, we consider the standard logistic regression model with $\ell_2$-regularization, i.e., $f_{i,j} (x) =  - y_{i,j}\log(h(x, a_{i,j})) - (1 - y_{i,j}) \log(1 - h(x, a_{i,j})) + \eta \|x\|^2,$ 
where $y_{i,j} \in \{0, 1\}$ is the label, $a_{i,j} \in \R^d$ represents the feature vector, $\eta$ is the regularization parameter, and $h(x, a) = \nicefrac{1}{(1 + e^{-a^\top x})}$. This objective is smooth, and for $\eta > 0$, it is also strongly convex, satisfying the P\L-condition. We consider the \textit{a9a} LIBSVM dataset~\citep{chang2011libsvm} and set $\eta = 0.01$. In the experiments, we focus on an important feature of \algname{Byz-VR-MARINA-PP}: it has linear convergence for homogeneous datasets across clients even in the presence of Byzantine workers and partial participation, as shown in Theorems~\ref{thm:main_result_1} and \ref{thm:main_result_2}. 

To demonstrate this experimentally, we consider the setup with 15 good workers and 5 Byzantines, \textit{each worker can access the entire dataset}, and the server uses coordinate-wise median with bucketing as the aggregator (see also Appendix~\ref{appendix:justification_of_Assumption_1}). For the attack, we propose a new attack that we refer to as the \textit{shift-back} attack, which acts in the following way. If Byzantine workers are in the majority in the current round $k$, then each Byzantine worker sends $x^0 - x^k$. Otherwise, they follow protocol and act as benign workers. Further experimental details are deferred to Appendix~\ref{app:experiments}.


We compare our \algname{Byz-VR-MARINA-PP} with its version without clipping. We note that the setup that we consider is the most favorable in terms of minimized variance in terms of data and gradient heterogeneity. We show that even in this simplest setup, the method without clipping does not converge since there is no method that can withstand the Byzantine majority. Therefore, any more complex scenario would also fall short using our simple attack. On the other hand, we show that once clipping is applied, \algname{Byz-VR-MARINA-PP} is able to converge linearly to the exact solution, complementing our theoretical results.

Figure~\ref{fig:a9a_full} showcases these observations. On the left, we can see \algname{Byz-VR-MARINA-PP} converges linearly to the optimal solution, while the version without clipping remains stuck at the starting point since Byzantines are always able to push the solution back to the origin since they can create the majority in some rounds. In the middle plot, we compare the full participation scenario in which all the clients participate in each round, which does not require clipping since, in each step, we are guaranteed that Byzantines are not in the majority, to partial participation with clipping. We can see, when we compare the total number of computations (measured in epochs), \algname{Byz-VR-MARINA-PP} leads to faster convergence even though we need to employ clipping. Finally, in the right plot, we measure the sensitivity of clipping multiplier $\lambda$. We can see that \algname{Byz-VR-MARINA-PP} is not very sensitive to $\lambda$ in terms of convergence, i.e., for all the values of $\lambda$, we still converge linearly. However, the suboptimal choice of $\lambda$ leads to slower convergence.

\begin{figure*}[t]
\centering
\includegraphics[width=0.245\textwidth]{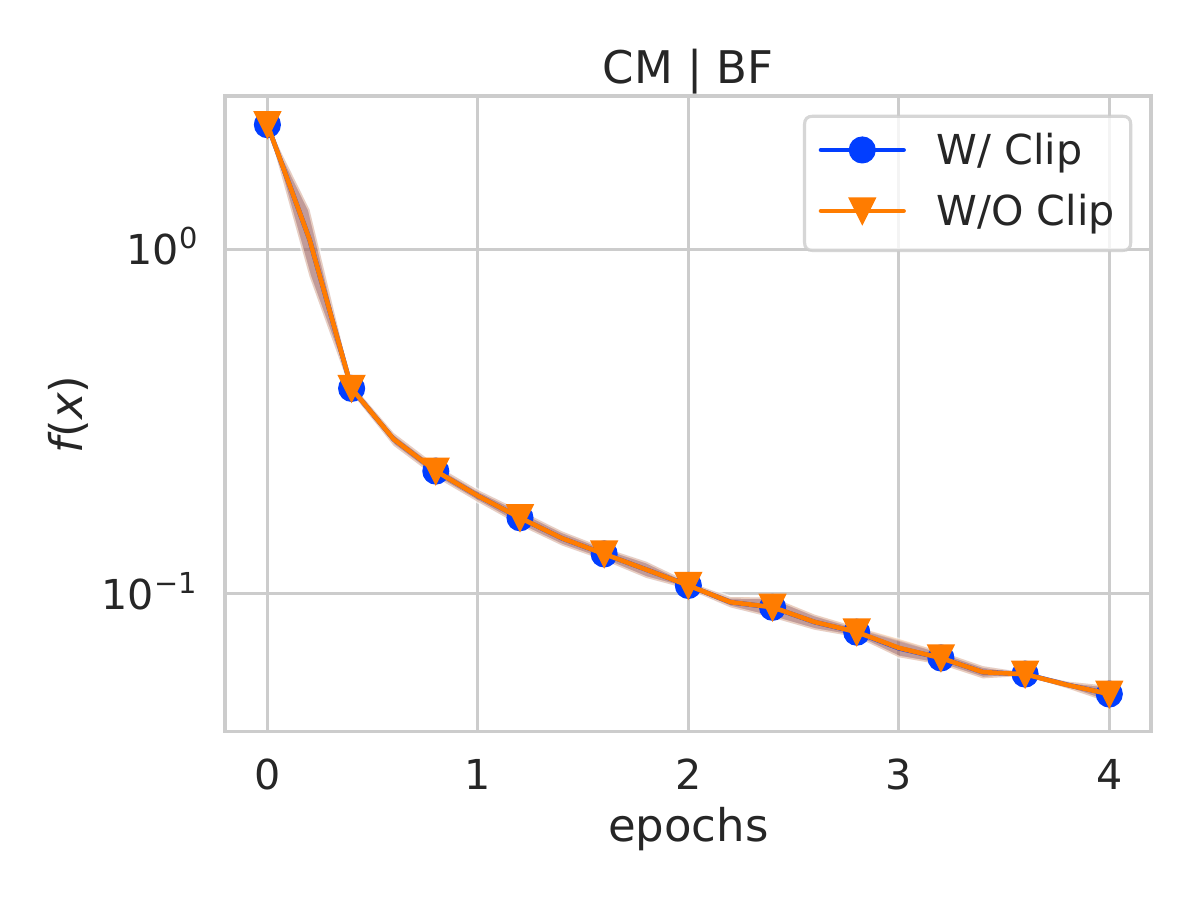}
\includegraphics[width=0.245\textwidth]{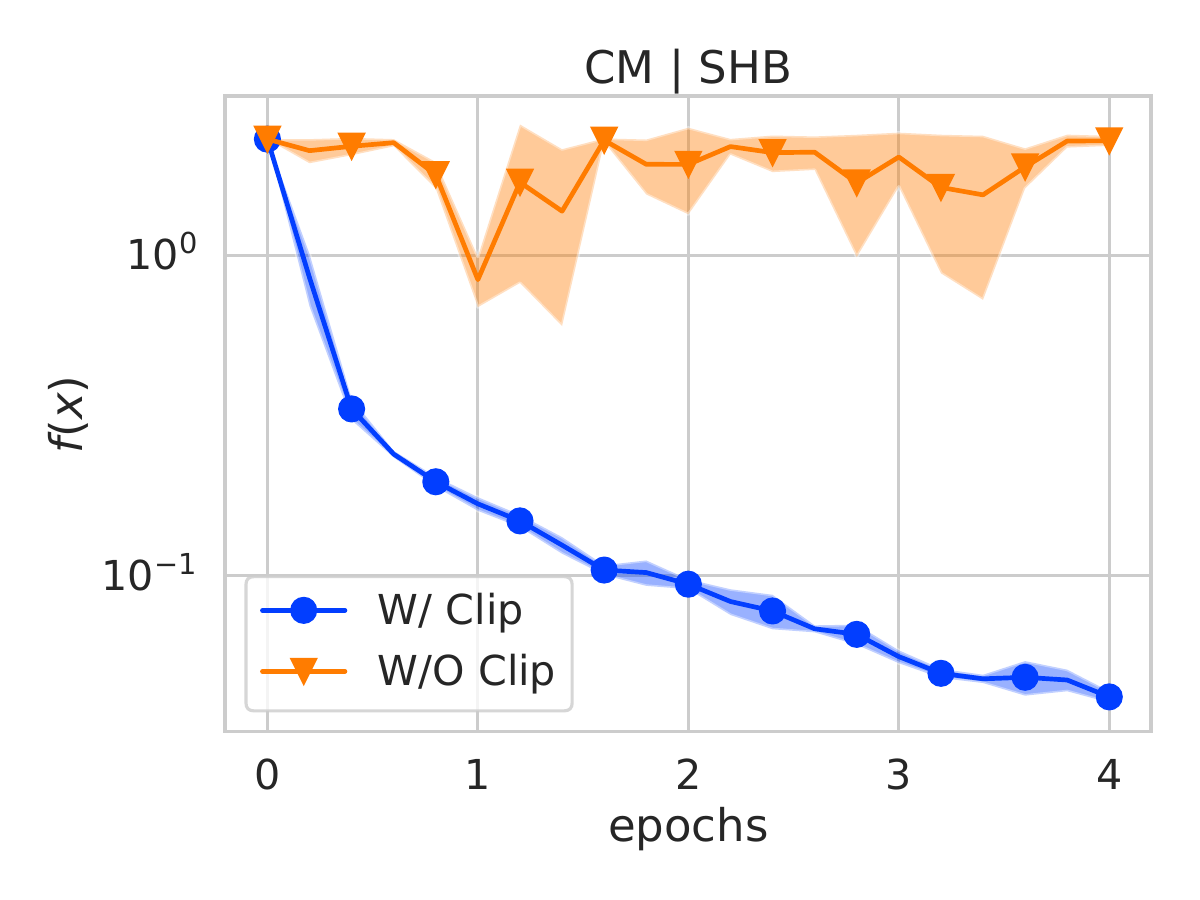}
\includegraphics[width=0.245\textwidth]{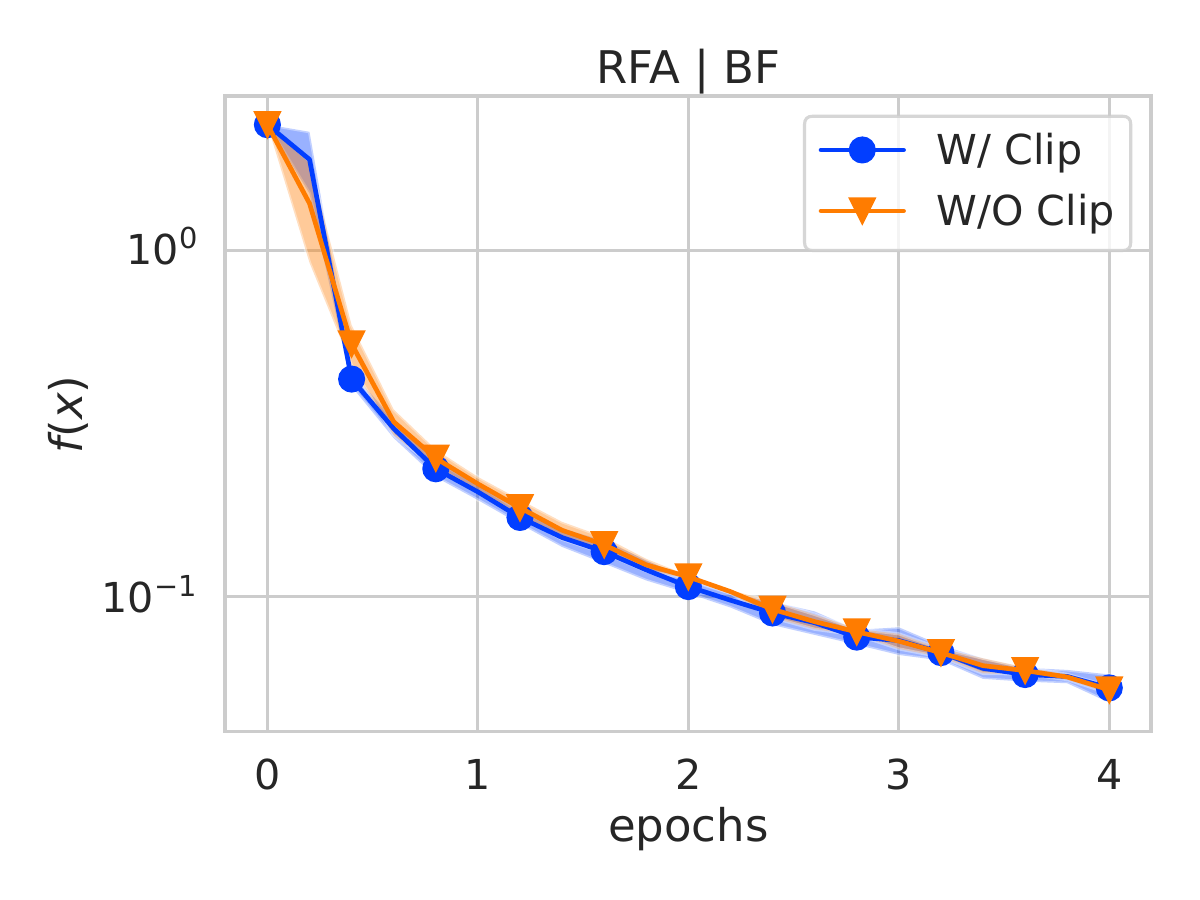}
\includegraphics[width=0.245\textwidth]{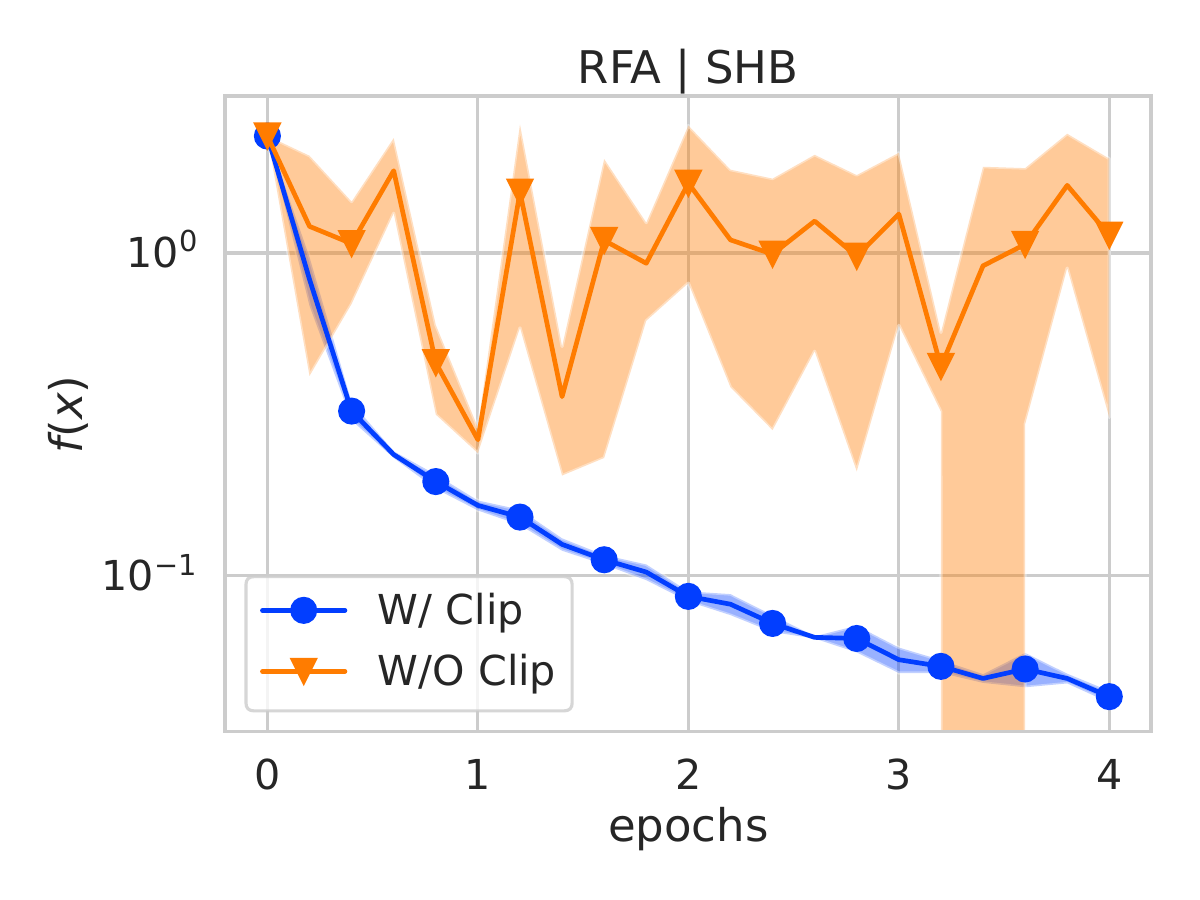}
\caption{
Training loss of 2 aggregation rules (CM, RFA) under 2 attacks (BF, SHB) on the MNIST dataset under heterogeneous data split with 20 clients, 
5 of which are malicious. Complete experiments with 4 attacks (BF, LF, ALIE, SHB) and test accuracy are provided in Appendix~\ref{app:experiments}.} 
\label{fig:nn}
\vspace{-1.5em}
\end{figure*}

Furthermore, we also realize that other attacks and more complicated experiments could potentially damage clipping more than methods not using clipping. Therefore, we provide additional experiments with neural networks and different attacks in heterogeneous settings. For our experimental setup, we follow \citep{karimireddy2021learning}. However, when working with neural networks, the choice of standard variance reduction is not effective~\citep{defazio2019ineffectiveness}. Therefore, we use Byzantine Robust Momentum SGD~\citep{karimireddy2021learning} as an underlying optimization method; see \eqref{eq:heuristic}. 

We consider the MNIST dataset~\citep{mnist} with heterogeneous splits with 20 clients, 5 of which are malicious. For the attacks, we consider A Little is Enough (ALIE)~\citep{baruch2019little}, Bit Flipping (BF), Label Flipping (LF), and aforementioned Shift-Back (SHB). For the aggregations, we consider coordinate median (CM)~\citep{chen2017distributed} and robust federated averaging (RFA)~\citep{pillutla2022robust} with bucketing. 

From Figure~\ref{fig:nn}, we can see that clipping does not lead to performance degradation. On the contrary, clipping performs on par or better than its variant without clipping. Furthermore, we can see that no robust aggregator is able to withstand the shift-back attack without clipping.

\vspace{-0.5em}
\section{Conclusion and Future Work}
\vspace{-0.5em}

This work makes an important first step in the direction of achieving Byzantine robustness under the partial participation of clients. 
However, some important questions remain open. 
First of all, it will be interesting to understand whether the derived bounds can be further improved in terms of the dependence on $\omega, m$, and $C$. Next, it would be interesting to rigorously prove that our heuristic works for \algname{SGD} with client momentum \citep{karimireddy2021learning, karimireddy2020byzantine} and other Byzantine-robust methods. Finally, studying other participation patterns (non-uniform sampling/arbitrary client participation) is also a very prominent direction for future research.

\bibliography{refs}

\newpage
\appendix

\tableofcontents

\clearpage

\section{Extra Related Work}\label{appendix:extra_related_work}

\paragraph{Further Comparison with \citet{data2021byzantine}.} As we mention in the main text, \citet{data2021byzantine} assume that $3B$ is smaller than $C$. More precisely, \citet{data2021byzantine} assume that $B \leq \epsilon C$, where $\epsilon \leq \frac{1}{3} - \epsilon’$ for some parameter $\epsilon’ > 0$ that will be explained later. That is, the results from \citet{data2021byzantine} do not hold when $C$ is smaller than $3B$, and, in particular, their algorithm cannot tolerate the situation when the server samples only Byzantine workers at some particular communication round. We also notice that when $C \geq 4B$, then existing methods such as \algname{Byz-VR-MARINA} \citep{gorbunov2023variance} or Client Momentum \citep{karimireddy2021learning, karimireddy2020byzantine} can be applied without any changes to get a provable convergence.

Next, \citet{data2021byzantine} derive the upper bounds for the expected squared distance to the solution (in the strongly convex case) and the averaged expected squared norm of the gradient (in the non-convex case), where the expectation is taken w.r.t.\ the sampling of stochastic gradients only and the bounds itself hold with probability at least $1 - \frac{K}{H}\exp\left( - \frac{\epsilon’^2(1 - \epsilon)C}{16}\right)$, where $H$ is the number of local steps. For simplicity consider the best-case scenario: $H = 1$ (local steps deteriorate the results from \citet{data2021byzantine}). Then, the lower bound for this probability becomes negative when either $C$ is not large enough or when $K$ is large or when $\epsilon$ is close to $\frac{1}{3}$, e.g., for $K = 10^6, \epsilon = \epsilon’ = \frac{1}{6}, C = 5000$ this lower bound is smaller than $-720$, meaning that in this case, the result does not guarantee convergence. In contrast, our results have classical convergence criteria, where the expectations are taken w.r.t.\ the all randomness.

Finally, the bounds from \citet{data2021byzantine} have non-reduceable terms even for homogeneous data case: these terms are proportional to $\frac{\sigma^2}{b}$, where $\sigma^2$ is the upper bound for the variance of the stochastic estimator on regular clients and $b$ is the batchsize. In contrast, our results have only decreasing terms in the upper bounds when the data is homogeneous.

\paragraph{Byzantine robustness.} There exist various approaches to achieving Byzantine robustness \citep{lyu2020privacy}. \citet{alistarh2018byzantine, allen2020byzantine} rely on the concentration inequalities for the stochastic gradients with bounded noise to iteratively remove them from the training. \citet{karimireddy2021learning} formalize the definition of robust aggregation and propose the first provably robust aggregation rule called \algname{CenteredClip} and the first provably Byzantine robust method under bounded variance assumption for homogeneous problems, i.e., when all good workers share one dataset. In particular, the method from \citep{karimireddy2021learning} uses client momentum on the clients that helps to memorize previous steps for good workers and withstand time-coupled attacks. This approach is extended by \citet{he2022byzantine} to the setup of decentralized learning. \citet{allouah2023fixing} develop an alternative definition for robust aggregation and propose a new aggregation rule satisfying their definition. \citet{karimireddy2020byzantine} generalize these results to the heterogeneous data case and derive lower bounds for the optimization error that one can achieve in the heterogeneous case. Based on the formalism from \citet{karimireddy2021learning}, \citet{gorbunov2021secure} propose a server-free approach that uses random checks of computations and bans of peers. This trick allows the elimination of all Byzantine workers after a finite number of steps on average. There are also many other approaches, e.g., one can use redundant computations of the stochastic gradients \citep{chen2018draco, rajput2019detox} or introduce reputation metrics \citep{rodriguez2020dynamic, regatti2020bygars, xu2020towards} to achieve some robustness, see also a recent survey by \citet{lyu2020privacy}.

\paragraph{Variance reduction.} The literature on variance-reduced methods is very rich \citep{gower2020variance}. The first variance-reduced methods are designed to fix the convergence of standard Stochastic Gradient Descent (\algname{SGD}) and make it convergent to any predefined accuracy even with constant stepsizes. Such methods as \algname{SAG} \citep{schmidt2017minimizing}, \algname{SVRG} \citep{johnson2013accelerating}, \algname{SAGA} \citep{defazio2014saga} are developed mainly for (strongly) convex smooth optimization problems, while methods like \algname{SARAH} \citep{nguyen2017sarah}, \algname{STORM} \citep{cutkosky2019momentum}, \algname{GeomSARAH} \citep{horvath2019stochastic}, \algname{PAGE} \citep{li2021page} are designed for general smooth non-convex problems. In this paper, we use \algname{GeomSARAH}/\algname{PAGE}-type variance reduction as the main building block of the method that makes the method robust to Byzantine attacks.

\paragraph{Partial participation and client sampling.} In the context of Byzantine-robust learning, there exists one work that develops and analyzes the method with partial participation \citep{data2021byzantine}. However, this work relies on the restrictive assumption that the number of participating clients at each round is at least three times larger than the number of Byzantine workers. In this case, Byzantines cannot form a majority, and standard methods can be applied without any changes. In contrast, our method converges in more challenging scenarios, e.g., \algname{Byz-VR-MARINA-PP} provably converges even when the server samples one client, which can be Byzantine. The results from \citet{data2021byzantine} have some other noticeable limitations that we discuss in Appendix~\ref{appendix:extra_related_work}.

\paragraph{Communication compression.} The literature on communication compression can be roughly divided into two huge groups. The first group studies the methods with unbiased communication compression. Different compression operators in the application to Distributed \algname{SGD}/\algname{GD} are studied in \citep{alistarh2017qsgd, wen2017terngrad, khirirat2018distributed}. To improve the convergence rate by fixing the error coming from the compression \citet{mishchenko2019distributed} propose to apply compression to the special gradient differences. Multiple extensions and generalizations of mentioned techniques are proposed and analyzed in the literature, e.g., see \citep{horvath2019stochastic, gorbunov2021marina, li2020acceleration, qian2021error, basu2019qsparse, haddadpour2021federated, sadiev2022federated, islamov2021distributed, safaryan2021fednl}. 

Another large part of the literature on compressed communication is devoted to biased compression operators \citep{ajalloeian2020convergence, demidovich2023guide}. Typically, such compression operators require more algorithmic changes than unbiased compressors since na\"ive combinations of biased compression with standard methods (e.g., Distributed \algname{GD}) can diverge \citep{beznosikov2020biased}. Error feedback is one of the most popular ways of utilizing biased compression operators in practice \citep{seide20141, stich2018sparsified, vogels2019powersgd}, see also \citep{richtarik2021ef21, fatkhullin2021ef21} for the modern version of error feedback with better theoretical guarantees for non-convex problems.

In the context of Byzantine robustness, methods with communication compression are also studied. The existing approaches are based on aggregation rules based on the norms of the updates \citep{ghosh2020distributed, ghosh2021communication}, \algname{SignSGD} and majority vote \citep{bernstein2018signsgd}, \algname{SAGA}-type variance reduction coupled with unbiased compression \citep{zhu2021broadcast}, and \algname{GeomSARAH}/\algname{PAGE}-type variance reduction combined with unbiased compression \citep{gorbunov2023variance}.

\paragraph{Gradient clipping.} Gradient clipping has multiple useful properties and applications. Originally it was used by \citet{pascanu2013difficulty} to reduce the effect of exploding gradients during the training of RNNs. Gradient clipping is also a popular tool for achieving provable differential privacy \citep{abadi2016deep, chen2020understanding}, convergence under generalized notions of smoothness \citep{zhang2019gradient, mai2021stability} and better (high-probability) convergence under heavy-tailed noise assumption \citep{zhang2020adaptive, nazin2019algorithms, gorbunov2020stochastic, sadiev2023high, nguyen2023improved}. In the context of Byzantine-robust learning, gradient clipping is also utilized to design provably robust aggregation \citep{karimireddy2021learning}. Our work proposes a novel useful application of clipping, i.e., we utilize clipping to achieve Byzantine robustness with partial participation of clients.

\clearpage

\section{Useful Facts}

For all $a, b \in \mathbb{R}^d$ and $\alpha>0, p \in(0,1]$ the following relations hold:

\begin{align}
\label{eq:quadratic} 2\langle a, b\rangle & =\|a\|^2+\|b\|^2-\|a-b\|^2 \\
\label{eq:yung-1}\|a+b\|^2 & \leq(1+\alpha)\|a\|^2+\left(1+\alpha^{-1}\right)\|b\|^2 \\
\label{eq:yung-2}-\|a-b\|^2 & \leq-\frac{1}{1+\alpha}\|a\|^2+\frac{1}{\alpha}\|b\|^2, \\
\label{eq:contract}(1-p)\left(1+\frac{p}{2}\right) & \leq 1-\frac{p}{2}, \quad p\geq 0 \\
\label{eq:contract-2}(1-p)\left(1+\frac{p}{2}\right)\left(1+\frac{p}{4}\right) & \leq 1-\frac{p}{4}\quad p\geq 0 .
\end{align}

\begin{lemma} (Lemma 5 from \citep{richtarik2021ef21}).
\label{lemma:peter}
Let $a, b>0$. If $0 \leq \gamma \leq \frac{1}{\sqrt{a}+b}$, then $a \gamma^2+b \gamma \leq 1$. The bound is tight up to the factor of 2 since $\frac{1}{\sqrt{a}+b} \leq \min \left\{\frac{1}{\sqrt{a}}, \frac{1}{b}\right\} \leq \frac{2}{\sqrt{a}+b}$.
\end{lemma}

\clearpage

\section{Justification of Assumption~\ref{assm:bounded-aggr}}\label{appendix:justification_of_Assumption_1}

\begin{algorithm}[h]
   \caption{Bucketing Algorithm \citep{karimireddy2020byzantine}}\label{alg:bucketing}
\begin{algorithmic}[1]
   \STATE {\bfseries Input:} $\{x_1,\ldots,x_n\}$, $s \in \mathbb{N}$ -- bucket size, \texttt{Aggr} -- aggregation rule 
   \STATE Sample random permutation $\pi = (\pi(1),\ldots, \pi(n))$ of $[n]$
   \STATE Compute $y_i = \frac{1}{s}\sum_{k = s(i-1)+1}^{\min\{si, n\}} x_{\pi(k)}$ for $i = 1, \ldots, \lceil \nicefrac{n}{s} \rceil$
   \STATE {\bfseries Return:} $\widehat x = \texttt{Aggr}(y_1, \ldots, y_{\lceil \nicefrac{n}{s} \rceil})$
\end{algorithmic}
\end{algorithm}

\paragraph{Krum and Krum $\circ$ Bucketing.} Krum aggregation rule is defined as
\begin{equation}
    \text{Krum}(x_1,\ldots, x_n) = \argmin\limits_{x_i \in \{x_1,\ldots, x_n\}} \sum\limits_{j \in S_i}\|x_j - x_i\|^2, \notag
\end{equation}
where $S_i \subset \{x_1,\ldots, x_n\}$ is the subset of $n-B-2$ closest vectors to $x_i$. By definition, $\text{Krum}(x_1,\ldots, x_n) \in \{x_1,\ldots, x_n\}$ and, thus $\|\text{Krum}(x_1,\ldots, x_n)\| \leq \max_{i\in [n]} \|x_i\|$, i.e., Assumption~\ref{assm:bounded-aggr} holds with $F_{\cA} = 1$. Since Krum $\circ$ Bucketing applies Krum aggregation to averages $y_i$ over the buckets and $\|y_i\| \leq \frac{1}{s}\sum_{k = s(i-1)+1}^{\min\{si, n\}} \|x_{\pi(k)}\| \leq \max_{i\in [n]} \|x_i\|$, we have that $\|\text{Krum}\circ\text{Bucketing}(x_1,\ldots, x_n)\| \leq \max_{i\in [n]} \|x_i\|$.

\paragraph{Geometric median (GM) and GM $\circ$ Bucketing.} Geometric median is defined as follows:
\begin{equation}
    \text{GM}(x_1,\ldots, x_n) = \argmin\limits_{x \in \R^d} \sum\limits_{i=1}^n \|x - x_i\|. \label{eq:GM}
\end{equation}
One can show that $\text{GM}(x_1,\ldots, x_n) \in \text{Conv}(x_1, \ldots, x_n) \eqdef \{x \in \R^d \mid x = \sum_{i=1}^n \alpha_i x_i \text{ for some } \alpha_1, \ldots, \alpha_n \geq 1 \text{ such that } \sum_{i=1}^n\alpha_i = 1\}$, i.e., geometric median belongs to the convex hull of the inputs. Indeed, let $\text{GM}(x_1,\ldots, x_n) = x = \hat x + \tilde x$, where $\hat x$ is the projection of $x$ on $\text{Conv}(x_1, \ldots, x_n)$ and $\tilde x = x - 
\hat x$. Then, the optimality condition implies that $\langle \hat x - x, y - \hat x \rangle \geq 0$ for all $y \in \text{Conv}(x_1, \ldots, x_n)$. In particular, for all $i \in [n]$ we have $\langle \hat x - x, x_i - \hat x \rangle \geq 0$. Since
\begin{eqnarray*}
\langle \hat x - x, x_i - \hat x \rangle &=& \langle \tilde x, \hat x - x_i \rangle = \frac{1}{2}\|\tilde x + \hat x - x_i\|^2 - \frac{1}{2}\|\tilde x\|^2 - \frac{1}{2}\|\hat x - x_i\|^2\\
&=& \frac{1}{2}\|x - x_i\|^2 - \frac{1}{2}\|\tilde x\|^2 - \frac{1}{2}\|\hat x - x_i\|^2\\
&\leq& \frac{1}{2}\|x - x_i\|^2 - \frac{1}{2}\|\hat x - x_i\|^2,
\end{eqnarray*}
we get that $\|x - x_i\| \geq \|\hat x - x_i\|$ for all $i\in [n]$ and the equality holds if and only if $\tilde x = 0$. Therefore, $\argmin$ from \eqref{eq:GM} is achieved for $x$ such that $x = \hat x$, meaning that $\text{GM}(x_1,\ldots, x_n) \in \text{Conv}(x_1, \ldots, x_n)$. Therefore, there exist some coefficients $\alpha_1, \ldots, \alpha_n \geq 0$ such that $\sum_{i=1}^n \alpha_i = 1$ and $\text{GM}(x_1,\ldots, x_n) = \sum_{i=1}^n \alpha_i x_i$, implying that
\begin{equation}
\|\text{GM}(x_1,\ldots, x_n)\| \leq \sum\limits_{i=1}^n \alpha_i \|x_i\| \leq \max\limits_{i\in [n]} \|x_i\|. \notag
\end{equation}
That is, GM satisfies Assumption~\ref{assm:bounded-aggr} with $F_{\cA} = 1$. Similarly to the case of Krum $\circ$ Bucketing, we also have $\|\text{GM}\circ\text{Bucketing}(x_1,\ldots, x_n)\| \leq \max_{i\in [n]} \|x_i\|$.

\paragraph{Coordinate-wise median (CM) and CM $\circ$ Bucketing.} Coordinate-wise median (CM) is formally defined as
\begin{equation}
    \text{CM}(x_1,\ldots, x_n) = \argmin\limits_{x \in \R^d} \sum\limits_{i=1}^n \|x - x_i\|_1, \label{eq:CM}
\end{equation}
where $\|\cdot\|_1$ denotes $\ell_1$-norm. This is equivalent to geometric median/median applied to vectors $x_1,\ldots, x_n$ component-wise. Therefore, from the above derivations for GM we have
\begin{eqnarray*}
    \|\text{CM}(x_1,\ldots, x_n)\|_{\infty} &\leq& \max_{i\in [n]}\|x_i\|_{\infty},\\
    \|\text{CM}\circ\text{Bucketing}(x_1,\ldots, x_n)\|_{\infty} &\leq& \max_{i\in [n]}\|x_i\|_{\infty},
\end{eqnarray*}
where $\|\cdot\|_{\infty}$ denotes $\ell_{\infty}$-norm. Therefore, due to the standard relations between $\ell_2$- and $\ell_\infty$-norms, i.e., $\|a\|_\infty \leq \|a\| \leq \sqrt{d}\|a\|_{\infty}$ for any $a \in \R^d$, we have
\begin{eqnarray*}
    \|\text{CM}(x_1,\ldots, x_n)\| &\leq& \sqrt{d}\max_{i\in [n]}\|x_i\|,\\
    \|\text{CM}\circ\text{Bucketing}(x_1,\ldots, x_n)\| &\leq& \sqrt{d}\max_{i\in [n]}\|x_i\|,
\end{eqnarray*}
i.e., Assumption~\ref{assm:bounded-aggr} is satisfied with $F_{\cA} = \sqrt{d}$.

\clearpage

\section{General Analysis}

\subsection{Refined Assumptions}\label{appendix:refined_assumptions}

For simplicity, in the main part of our paper, we present simplified versions of our main results. However, our analysis works under more general assumptions presented in this section.

\paragraph{Refined smoothness.} The following assumption is classical for the literature on non-convex optimization.
\begin{assumption}[$L$-smoothness]
\label{assm:L-smoothness}
 We assume that function $f: \mathbb{R}^d \rightarrow \mathbb{R}$ is L-smooth, i.e., for all $x, y \in \mathbb{R}^d$ we have $\|\nabla f(x)-\nabla f(y)\| \leq L\|x-y\|$. Moreover, we assume that $f$ is uniformly lower bounded by $f^* \in \mathbb{R}$, i.e., $f^*\eqdef\inf _{x \in \mathbb{R}^d} f(x)$. In addition, we assume that $f_i$ is $L_i$-smooth for all $i\in\cG$, i.e., for all $x, y \in \mathbb{R}^d$
\begin{equation}
    \|\nabla f_i(x) - \nabla f_i(y)\| \leq L_i \|x - y\|. \label{eq:f_i_smooth}
\end{equation}
\end{assumption}
We notice here that \eqref{eq:f_i_smooth} implies $L$-smoothness of $f$ with $L \leq \frac{1}{G}\sum_{i \in \cG} L_i$, i.e., smoothness constant of $f$ can be better than the averaged smoothness constant of the local loss functions on the regular clients.

Following \citet{gorbunov2023variance}, we consider refined assumptions on the smoothness.
\begin{assumption}[Global Hessian variance assumption \citep{szlendak2021permutation}] 
\label{assm:global} We assume that there exists $L_{ \pm} \geq 0$ such that for all $x, y \in \mathbb{R}^d$
\begin{align}
   \frac{1}{G} \sum_{i \in \mathcal{G}}\left\|\nabla f_i(x)-\nabla f_i(y)\right\|^2-\|\nabla f(x)-\nabla f(y)\|^2\leq L_{ \pm}^2\|x-y\|^2. \label{eq:global_hessian_variance}
\end{align}
\end{assumption}

We notice that \eqref{eq:f_i_smooth} implies \eqref{eq:global_hessian_variance} with $L_\pm \leq \max_{i\in \cG} L_i$. \citet{szlendak2021permutation} prove that $L_\pm$ satisfies the following relation: $L_{\mathrm{avg}}^2 - L^2 \leq L_\pm^2 \leq L_{\mathrm{avg}}^2$, where $L_{\mathrm{avg}}^2 \eqdef \frac{1}{G}\sum_{i\in \cG}L_i^2$. In particular, it is possible that $L_\pm = 0$ even if the data on the good workers is heterogeneous.

\begin{assumption}[Local Hessian variance assumption \citep{gorbunov2023variance}]
\label{assm:local}
We assume that there exists $\mathcal{L}_{ \pm} \geq 0$ such that for all $x, y \in \mathbb{R}^d$
$$
\frac{1}{G} \sum_{i \in \mathcal{G}} \mathbb{E}\left\|\widehat{\Delta}_i(x, y)-\Delta_i(x, y)\right\|^2 \leq \frac{\mathcal{L}_{ \pm}^2}{b}\|x-y\|^2,
$$
where $\Delta_i(x, y)\eqdef\nabla f_i(x)-\nabla f_i(y)$ and $\widehat{\Delta}_i(x, y)$ is an unbiased mini-batched estimator of $\Delta_i(x, y)$ with batch size $b$.
\end{assumption}

This assumption incorporates considerations for the smoothness characteristics inherent in all functions $\{f_{i,j}\}_{i\in \cG, j\in [m]}$, the sampling policy, and the similarity among the functions $\{f_{i,j}\}_{i\in \cG, j\in [m]}$. \citet{gorbunov2023variance} have demonstrated that, assuming smoothness of $\{f_{i,j}\}_{i\in \cG, j\in [m]}$, Assumption~\ref{assm:local} holds for various standard sampling strategies, including uniform and importance samplings.

For part of our results, we also need to assume smoothness of all $\{f_{i,j}\}_{i\in \cG, j\in [m]}$ explicitly.

\begin{assumption}[Smoothness of $f_{i,j}$ (optional)]
\label{assm:local_smooth_all}
We assume that for all $i\in \cG$ and $j\in [m]$ there exists $L_{i,j} \geq 0$ such that $f_{i,j}$ is $L_{i,j}$-smooth, i.e., for all $x, y \in \mathbb{R}^d$
\begin{equation}
    \|\nabla f_{i,j}(x) - \nabla f_{i,j}(y)\| \leq L_{i,j} \|x - y\|. \label{eq:f_i_j_smooth}
\end{equation}
\end{assumption}

\paragraph{Refined heterogeneity.} Instead of Assumption~\ref{assm:het_simplified}, we consider a more generalized one. 

\begin{assumption}[$(B, \zeta^2)$-heterogeneity] 
\label{assm:het}
We assume that good clients have $\left(B, \zeta^2\right)$-heterogeneous local loss functions for some $B \geq 0, \zeta \geq 0$, i.e.,
\begin{align*}
\frac{1}{G} \sum_{i \in \mathcal{G}}\left\|\nabla f_i(x)-\nabla f(x)\right\|^2 \leq B\|\nabla f(x)\|^2+\zeta^2 \quad \forall x \in \mathbb{R}^d
\end{align*}
\end{assumption}

When $B = 0$, the above assumption recovers Assumption~\ref{assm:het_simplified}. However, it also covers some situations when the model is over-parameterized \citep{vaswani2019fast} and can hold with smaller values of $\zeta^2$. This assumption is also used in \citep{karimireddy2020byzantine, gorbunov2023variance}.

\subsection{Technical Lemmas}

\begin{lemma} \label{lemma:clipping} Let  $X \text { be a random vector in } \mathbb{R}^d \text { and } \widetilde{X}=\clip_{\lambda}(X).$ Assume that $\mathbb{E}[X]=x \in \mathbb{R}^d$ and $\|x\| \leq \lambda / 2,$ then 

\begin{equation*}
    \mathbb{E}\left[\|\widetilde{X}-x\|^2\right] \leq 10\mathbb{E}\left\| X - x\right\|^2.
\end{equation*}
\end{lemma}
\begin{proof}
The proof follows a similar procedure to that presented in Lemma F.5 from \citep{gorbunov2020stochastic}. To commence the proof, we introduce two indicator random variables:
$$
\chi=\mathbb{I}_{\{X:\|X\|>\lambda\}}=\left\{\begin{array}{ll}
1, & \text { if }\|X\|>\lambda, \\
0, & \text { otherwise }
\end{array}, \eta=\mathbb{I}_{\left\{X:\|X-x\|>\frac{\lambda}{2}\right\}}= \begin{cases}1, & \text { if }\|X-x\|>\frac{\lambda}{2} \\
0, & \text { otherwise }\end{cases}\right..
$$
Moreover, since $\|X\| \leq\|x\|+\|X-x\| \stackrel{\|x\| \leq \lambda / 2}{\leq} \frac{\lambda}{2}+\|X-x\|$, we have $\chi \leq \eta$. Using that we get
$$
\widetilde{X}=\min \left\{1, \frac{\lambda}{\|X\|}\right\} X=\chi \frac{\lambda}{\|X\|} X+(1-\chi) X.
$$
By Markov’s inequality,
\begin{align}
\label{eq:exp_eta}
\mathbb{E}[\eta]&=\mathbb{P}\left\{\|X-x\|>\frac{\lambda}{2}\right\}=\mathbb{P}\left\{\|X-x\|^2>\frac{\lambda^2}{4}\right\} \leq \frac{4}{\lambda^2} \mathbb{E}\left[\|X-x\|^2\right].
\end{align}
$\text { Using }\|\widetilde{X}-x\| \leq\|\widetilde{X}\|+\|x\| \leq \lambda+\frac{\lambda}{2}=\frac{3 \lambda}{2} \text {, we obtain }$
\begin{align*}
\mathbb{E}\left[\|\widetilde{X}-x\|^2\right] &= \mathbb{E}\left[\|\widetilde{X}-x\|^2 \chi+\|\widetilde{X}-x\|^2(1-\chi)\right] \\
& =\mathbb{E}\left[\chi\left\|\frac{\lambda}{\|X\|} X-x\right\|^2+\|X-x\|^2(1-\chi)\right]\\
&\leq\mathbb{E}\left[\chi\left(\left\|\frac{\lambda}{\|X\|} X\right\|+\|x\|\right)^2+\|X-x\|^2(1-\chi)\right]\\
&\stackrel{\|x\| \leq \frac{\lambda}{2}}{\leq}\left(\mathbb{E}\left[\chi\left(\frac{3 \lambda}{2}\right)^2+\|X-x\|^2\right]\right),
\end{align*}
where in the last inequality we applied $1 - \chi \leq 1.$ Using (\ref{eq:exp_eta}) and $\chi \leq \eta$ we get 
\begin{align*}
\mathbb{E}\left[\|\widetilde{X}-x\|^2\right] & \leq \frac{9 \lambda^2}{4}\left(\frac{2 }{\lambda}\right)^2\mathbb{E}\left[\|X-x\|^2\right]+\mathbb{E}\left[\|X-x\|^2\right] \\
& \leq 10\mathbb{E}\left[\|X-x\|^2\right].
\end{align*}
\end{proof}

\begin{lemma}[Lemma 2 from \citet{li2021page}]
\label{lemma:page}
Assume that function $f$ is L-smooth (Assumption~\ref{assm:L-smoothness}) and $x^{k+1}=$ $x^k-\gamma g^k$. Then
$$
f\left(x^{k+1}\right) \leq f\left(x^k\right)-\frac{\gamma}{2}\left\|\nabla f\left(x^k\right)\right\|^2-\left(\frac{1}{2 \gamma}-\frac{L}{2}\right)\left\|x^{k+1}-x^k\right\|^2+\frac{\gamma}{2}\left\|g^k-\nabla f\left(x^k\right)\right\|^2 .
$$
\end{lemma}

\begin{lemma}
\label{lemma:premainA1}
Let Assumptions~\ref{assm:L-smoothness}, \ref{assm:global}, \ref{assm:local} hold and the Compression Operator satisfy Definition~\ref{def:Q}. Let us define "ideal" estimator:
\begin{equation*}
    \overline{g}^{k+1}= \begin{cases}
        \frac{1}{G^k_C}\sum \limits_{i \in \mathcal{G}^k_C} \nabla f_i(x^{k+1}),& c_n=1,\quad\quad\quad\quad\quad\quad\quad\quad\quad~\hspace{0.01cm} [1]\\
        g^k+\nabla f\left(x^{k+1}\right)-\nabla f\left(x^k\right),& c_n=0 \text{ and } G^k_C < (1-\delta)C, \hspace{0.325cm}[2]\\
        g^k+\frac{1}{G^k_C} \sum \limits_{i \in \mathcal{G}^k_C}\clip_{\lambda}\left(\mathcal{Q}\left(\widehat{\Delta}_i\left(x^{k+1}, x^k\right)\right)\right),& c_n=0 \text{ and } G^k_C \geq (1-\delta)C. \hspace{0.375cm}   [3]
    \end{cases}
\end{equation*}
Then for all $k\geq 0$ the iterates produced by \algname{Byz-VR-MARINA-PP} (Algorithm~\ref{alg:byz_vr_marina}) satisfy
\begin{align*}
 A_1 &=    \mathbb{E}\left[\left\|\overline{g}^{k+1}-\nabla f\left(x^{k+1}\right)\right\|^2\right]\\& \leq (1-p) \left(1+ \frac{p}{4}\right)\mathbb{E}\left[\left\|g^{k}-\nabla f(x^{k})\right\|^2\right]+p\frac{ \delta\cdot\mathcal{P }_{\mathcal{G}^k_{\widehat{C}}} }{(1-\delta)}  \mathbb{E}\left[ B\|\nabla f(x)\|^2+\zeta^2 \right]\\
    &+ \text{\scriptsize $(1-p)p_G\left(1+ \frac{4}{p}\right) \frac{2\cdot\mathcal{P}_{\mathcal{G}^k_C}n}{C} \left( 10 \omega L^2+(10\omega+1) L_{ \pm}^2+\frac{10(\omega+1) \mathcal{L}_{ \pm}^2}{b} \right)\mathbb{E}\left[\|x^{k+1} - x^k\|^2\right]$},
\end{align*}
where $p_G = \operatorname{Prob}\left\lbrace G^k_C \geq (1-\delta)C \right\rbrace $ and $\mathcal{P}_{\mathcal{G}^k_C} =  \operatorname{Prob}\left\lbrace i \in \mathcal{G}^k_C \mid G^k_C \geq\left(1-\delta\right) C\right\rbrace$.
\end{lemma}

\begin{proof}

Let us examine the expected value of the squared difference between ideal estimator and full gradient:
\begin{align*}
    A_1 &= \mathbb{E}\left[\left\|\overline{g}^{k+1}-\nabla f\left(x^{k+1}\right)\right\|^2\right]\\
&=\mathbb{E}\left[\mathbb{E}_k\left[\left\|\overline{g}^{k+1}-\nabla f\left(x^{k+1}\right)\right\|^2\right]\right]\\
    & = \left(1-p\right)p_{G}\mathbb{E}\left[\mathbb{E}_k\left[\left\|g^k+\frac{1}{G^k_C} \sum \limits_{i \in \mathcal{G}^k_C}\clip_{\lambda}\left(\mathcal{Q}\left(\widehat{\Delta}_i\left(x^{k+1}, x^k\right)\right)\right) -\nabla f\left(x^{k+1}\right)\right\|^2\right]\mid [3]\right]\\
    & + (1-p)(1-p_{G})\mathbb{E}\left[\mathbb{E}_k\left[\left\|g^k -\nabla f(x^{k})\right\|^2\right]\mid [2]\right] + p\mathbb{E}\left[ \left\|\frac{1}{G^k_{\widehat{C}}}\sum \limits_{i \in \cG_{\widehat{C}}^k} \nabla f_i(x^{k+1}) - \nabla f(x^{k+1})\right\|^2 \right].
\end{align*}
Using (\ref{eq:yung-1}) and $\nabla f\left(x^{k}\right) - \nabla f\left(x^{k}\right) = 0$ we obtain
\begin{align*}
B_1 &= \mathbb{E}\left[\mathbb{E}_k\left[\left\|g^k+\frac{1}{G^k_C} \sum \limits_{i \in \mathcal{G}^k_C}\clip_{\lambda}\left(\mathcal{Q}\left(\widehat{\Delta}_i\left(x^{k+1}, x^k\right)\right)\right) -\nabla f\left(x^{k+1}\right)\right\|^2\right]\mid [3]\right]\\
&\text{\scriptsize$ =\mathbb{E}\left[\mathbb{E}_k\left[\left\|g^k+\frac{1}{G^k_C} \sum \limits_{i \in \mathcal{G}^k_C}\clip_{\lambda}\left(\mathcal{Q}\left(\widehat{\Delta}_i\left(x^{k+1}, x^k\right)\right)\right) -\nabla f\left(x^{k+1}\right)+\nabla f\left(x^{k}\right) - \nabla f\left(x^{k}\right)\right\|^2\right]\mid [3]\right]$}\\
     &\stackrel{(\ref{eq:yung-1})}{\leq} \left(1+ \frac{p}{4}\right)\mathbb{E}\left[\left\|g^{k} -\nabla f\left(x^{k}\right)\right\|^2\right]\\
    &+ \text{\scriptsize$ \left(1+ \frac{4}{p}\right) \mathbb{E}\left[\mathbb{E}_k\left[\left\|\frac{1}{G^k_C} \sum \limits_{i \in \mathcal{G}^k_C}\clip_{\lambda}\left(\mathcal{Q}\left(\widehat{\Delta}_i\left(x^{k+1}, x^k\right)\right)\right)  - \left( \nabla f(x^{k+1}) - \nabla f(x^{k}) \right)\right\|^2\right]\mid [3]\right]$}\\
    &= \left(1+ \frac{p}{4}\right)\mathbb{E}\left[\left\|g^{k}-\nabla f(x^{k})\right\|^2\right]\\
    &+ \left(1+ \frac{4}{p}\right)\mathbb{E}\left[ \mathbb{E}_k\left[\left\|\frac{1}{G^k_C} \sum \limits_{i \in \mathcal{G}^k_C}\clip_{\lambda}\left(\mathcal{Q}\left(\widehat{\Delta}_i\left(x^{k+1}, x^k\right)\right)\right)  - \Delta\left(x^{k+1}, x^k\right)\right\|^2\right]\mid [3]\right].
\end{align*}
Let us consider the last part of the inequality:
\begin{align}
\notag  B^\prime_1 &= \mathbb{E}\left[ \mathbb{E}_k\left[\left\|\frac{1}{G^k_C} \sum \limits_{i \in \mathcal{G}^k_C}\clip_{\lambda}\left(\mathcal{Q}\left(\widehat{\Delta}_i\left(x^{k+1}, x^k\right)\right)\right)  - \Delta\left(x^{k+1}, x^k\right)\right\|^2\right]\mid [3] \right]\\
 \notag &= \mathbb{E}\left[\mathbb{E}_{S_k}\left[\mathbb{E}_k\left[\left\|\frac{1}{G^k_C} \sum \limits_{i \in \mathcal{G}^k_C}\clip_{\lambda}\left(\mathcal{Q}\left(\widehat{\Delta}_i\left(x^{k+1}, x^k\right)\right)\right)  - \Delta\left(x^{k+1}, x^k\right)\right\|^2\right]\mid [3]\right]\right].
 \end{align}
 Note that $G^k_C \geq (1-\delta)C$ in this case:
 \begin{align}
\notag B^\prime_1  &\leq \text{\small $\frac{1}{C(1-\delta)}\mathbb{E}\left[\mathbb{E}_{S_k}\left[ \sum \limits_{i \in \mathcal{G}^k_C}\mathbb{E}_k\left[\left\|\clip_{\lambda}\left(\mathcal{Q}\left(\widehat{\Delta}_i\left(x^{k+1}, x^k\right)\right)\right)  - \Delta\left(x^{k+1}, x^k\right)\right\|^2\right]\mid [3]\right]\right]$}\\
 \notag   &\leq \text{\small $\frac{1}{C(1-\delta)}\mathbb{E}\left[\sum \limits_{i \in \mathcal{G}}\mathbb{E}_{S_k}\left[\mathcal{I}_{\mathcal{G}^k_C}\right]\mathbb{E}_k\left[\left\|\clip_{\lambda}\left(\mathcal{Q}\left(\widehat{\Delta}_i\left(x^{k+1}, x^k\right)\right)\right)  - \Delta\left(x^{k+1}, x^k\right)\right\|^2\right]\mid [3] \right]$}\\
        &\text{\small$= \frac{1}{C(1-\delta)}\mathbb{E}\left[\sum \limits_{i \in \mathcal{G}}\mathcal{P}_{\mathcal{G}^k_C}\cdot\mathbb{E}_k\left[\left\|\clip_{\lambda}\left(\mathcal{Q}\left(\widehat{\Delta}_i\left(x^{k+1}, x^k\right)\right)\right)  - \Delta\left(x^{k+1}, x^k\right)\right\|^2\right]\mid [3]\right]$},\label{eq:bvdjbjdfbvjdf}
\end{align}
where $\mathcal{I}_{\mathcal{G}^k_C}$ is an indicator function for the event $\left\lbrace i \in \mathcal{G}^k_C \mid G^k_C \geq\left(1-\delta\right) C\right\rbrace$ and $\mathcal{P}_{\mathcal{G}^k_C} =  \operatorname{Prob}\left\lbrace i \in \mathcal{G}^k_C \mid G^k_C \geq\left(1-\delta\right) C\right\rbrace$ is probability of such event. Note that 
$\mathbb{E}_{S_k}\left[\mathcal{I}_{\mathcal{G}^k_C}\right] = \mathcal{P}_{\mathcal{G}^k_C}$. In case of uniform sampling of clients we have 
\begin{align*}
    \forall i \in \mathcal{G} \quad \mathcal{P}_{\mathcal{G}^k_C} &=\operatorname{Prob}\left\lbrace i \in \mathcal{G}^k_C \mid G^k_C \geq\left(1-\delta\right) C\right\rbrace\\
    & = \frac{C}{n p_G} \cdot \sum_{(1-\delta)C\leq t \leq C} \left(\left(\begin{array}{l}
G \\
t
\end{array}\right)\left(\begin{array}{l}
n-G \\
C-t
\end{array}\right) \left(\left(\begin{array}{l}
n \\
C
\end{array}\right)\right)^{-1} \right),\\
p_G &= \sum_{(1-\delta)C\leq t \leq C} \left(\left(\begin{array}{l}
G-1 \\
t-1
\end{array}\right)\left(\begin{array}{l}
n-G \\
C-t
\end{array}\right) \left(\left(\begin{array}{l}
n-1 \\
C-1
\end{array}\right)\right)^{-1} \right)
\end{align*}
Now we can continue with inequalities:
\begin{align*}
B_1^\prime &\leq  \text{\small $       \frac{\mathcal{P}_{\mathcal{G}^k_C}}{C(1-\delta)} \mathbb{E}\left[ \sum \limits_{i \in \mathcal{G}}\mathbb{E}_k\left[\left\|\clip_{\lambda}\left(\mathcal{Q}\left(\widehat{\Delta}_i\left(x^{k+1}, x^k\right)\right)\right)  - \Delta\left(x^{k+1}, x^k\right)\right\|^2\right]\mid [3]\right]$}\\
 &\leq     \text{\small $      \frac{\mathcal{P}_{\mathcal{G}^k_C}}{C(1-\delta)}\mathbb{E}\left[\sum \limits_{i \in \mathcal{G}}\mathbb{E}_k\left[\mathbb{E}_{Q}\left[\left\|\clip_{\lambda}\left(\mathcal{Q}\left(\widehat{\Delta}_i\left(x^{k+1}, x^k\right)\right)\right)  - \Delta\left(x^{k+1}, x^k\right)\right\|^2\right] \right]\mid [3]  \right]$}\\
 &\stackrel{(\ref{eq:yung-1})}{\leq}   \text{\small $      \frac{\mathcal{P}_{\mathcal{G}^k_C}}{C(1-\delta)}\mathbb{E}\left[\sum \limits_{i \in \mathcal{G}}2\mathbb{E}_k\left[\mathbb{E}_{Q}\left[\left\|\clip_{\lambda}\left(\mathcal{Q}\left(\widehat{\Delta}_i\left(x^{k+1}, x^k\right)\right)\right)  - \Delta_i\left(x^{k+1}, x^k\right)\right\|^2\right]\right]\mid [3]\right]$}\\ 
& +   \frac{\mathcal{P}_{\mathcal{G}^k_C}}{C(1-\delta)}\mathbb{E}\left[\sum \limits_{i \in \mathcal{G}}2\mathbb{E}_k\left[\left\|  \Delta_i\left(x^{k+1}, x^k\right) - \Delta\left(x^{k+1}, x^k\right)\right\|^2\right]\mid [3]\right].
\end{align*}
Using Lemma \ref{lemma:clipping} we have 
\begin{align*}
B_1^\prime&\stackrel{\text{Lemma \ref{lemma:clipping}}}{\leq}  \text{\small $    \frac{\mathcal{P}_{\mathcal{G}^k_C}}{C(1-\delta)}\mathbb{E}\left[\sum \limits_{i \in \mathcal{G}}20\mathbb{E}_k\left[\mathbb{E}_{Q}\left[\left\|\mathcal{Q}\left(\widehat{\Delta}_i\left(x^{k+1}, x^k\right)\right)  - \Delta_i\left(x^{k+1}, x^k\right)\right\|^2\right]\right]\mid [3] \right]$}\\
&+\frac{\mathcal{P}_{\mathcal{G}^k_C}}{C(1-\delta)}\mathbb{E}\left[\sum \limits_{i \in \mathcal{G}}2\mathbb{E}_k\left[\left\|  \Delta_i\left(x^{k+1}, x^k\right) -  \Delta\left(x^{k+1}, x^k\right)\right\|^2\right]\mid [3]\right]\\
&\leq  \frac{20\cdot\mathcal{P}_{\mathcal{G}^k_C}}{C(1-\delta)} \mathbb{E}\left[ \sum \limits_{i \in \mathcal{G}}\mathbb{E}_k\left[\mathbb{E}_{Q}\left[\left\|\mathcal{Q}\left(\widehat{\Delta}_i\left(x^{k+1}, x^k\right)\right)  - \Delta_i\left(x^{k+1}, x^k\right)\right\|^2\right]\right]\mid [3]\right]\\
&+\frac{2\cdot\mathcal{P}_{\mathcal{G}^k_C}}{C(1-\delta)} \mathbb{E}\left[ \sum \limits_{i \in \mathcal{G}}\mathbb{E}_k\left[\left\|  \Delta_i\left(x^{k+1}, x^k\right) -  \Delta\left(x^{k+1}, x^k\right)\right\|^2\right]\mid [3]\right]\\
&\leq \text{\small $\frac{20\cdot\mathcal{P}_{\mathcal{G}^k_C}}{C(1-\delta)} \mathbb{E}\left[ \sum \limits_{i \in \mathcal{G}}\mathbb{E}_k\left[\mathbb{E}_{Q}\left[\left\|\mathcal{Q}\left(\widehat{\Delta}_i\left(x^{k+1}, x^k\right)\right)\right\|^2\right]\right] - \sum \limits_{i \in \mathcal{G}}\left\|\Delta_i\left(x^{k+1}, x^k\right)\right\|^2\mid [3]\right]$}\\
&+\frac{2\cdot\mathcal{P}_{\mathcal{G}^k_C}}{C(1-\delta)}\mathbb{E}\left[\sum \limits_{i \in \mathcal{G}}\mathbb{E}_k\left[\left\|  \Delta_i\left(x^{k+1}, x^k\right) -  \Delta\left(x^{k+1}, x^k\right)\right\|^2\right]\mid [3]\right].
\end{align*}
Applying Definition~\ref{def:Q} of Unbiased Compressor we have 
\begin{align*}
B_1^\prime&\leq  \frac{20\cdot\mathcal{P}_{\mathcal{G}^k_C}}{C(1-\delta)} \mathbb{E}\left[ \sum \limits_{i \in \mathcal{G}}(1+\omega)\mathbb{E}_k\left\|\widehat{\Delta}_i\left(x^{k+1}, x^k\right)\right\|^2 - \sum \limits_{i \in \mathcal{G}}\left\|\Delta_i\left(x^{k+1}, x^k\right)\right\|^2\mid [3]\right]\\
&+\frac{2\cdot\mathcal{P}_{\mathcal{G}^k_C}}{C(1-\delta)}\mathbb{E}\left[\sum \limits_{i \in \mathcal{G}}\left\|  \Delta_i\left(x^{k+1}, x^k\right) -  \Delta\left(x^{k+1}, x^k\right)\right\|^2\mid [3]\right]\\
&\leq  \frac{20\cdot\mathcal{P}_{\mathcal{G}^k_C}}{C(1-\delta)} \mathbb{E} \left[ \sum \limits_{i \in \mathcal{G}}(1+\omega)\mathbb{E}_k\left\|\widehat{\Delta}_i\left(x^{k+1}, x^k\right) - \Delta_i\left(x^{k+1}, x^k\right) \right\|^2 \right]\\
&+ \frac{20\cdot\mathcal{P}_{\mathcal{G}^k_C}}{C(1-\delta)} \mathbb{E} \left[ \sum \limits_{i \in \mathcal{G}}(1+\omega)\mathbb{E}_k\left\|\Delta_i\left(x^{k+1}, x^k\right)\right\|^2- \sum \limits_{i \in \mathcal{G}}\mathbb{E}_k\left\|\Delta_i\left(x^{k+1}, x^k\right)\right\|^2\mid [3]\right]\\
&+\frac{2\cdot\mathcal{P}_{\mathcal{G}^k_C}}{C(1-\delta)}\mathbb{E}\left[\sum \limits_{i \in \mathcal{G}}\left\|  \Delta_i\left(x^{k+1}, x^k\right) -  \Delta\left(x^{k+1}, x^k\right)\right\|^2\mid [3]\right].
\end{align*}
Now we combine terms and have 
\begin{align*}
    B_1^\prime &\leq  \frac{20\cdot\mathcal{P}_{\mathcal{G}^k_C}}{C(1-\delta)} (1+\omega) \mathbb{E} \left[ \sum \limits_{i \in \mathcal{G}}\mathbb{E}_k\left[\left\|\widehat{\Delta}_i\left(x^{k+1}, x^k\right) - \Delta_i\left(x^{k+1}, x^k\right) \right\|^2\right] \mid [3] \right]\\
&+ \frac{20\cdot\mathcal{P}_{\mathcal{G}^k_C}}{C(1-\delta)}\omega \mathbb{E} \left[ \sum \limits_{i \in \mathcal{G}}\left\|\Delta_i\left(x^{k+1}, x^k\right)\right\|^2\mid [3]\right]\\
&+\frac{2\cdot\mathcal{P}_{\mathcal{G}^k_C}}{C(1-\delta)}\mathbb{E} \left[ \sum \limits_{i \in \mathcal{G}}\left\|  \Delta_i\left(x^{k+1}, x^k\right) -  \Delta\left(x^{k+1}, x^k\right)\right\|^2\mid [3]\right]\\
 &= \frac{20\cdot\mathcal{P}_{\mathcal{G}^k_C}}{C(1-\delta)} (1+\omega) \mathbb{E} \left[ \sum \limits_{i \in \mathcal{G}}\mathbb{E}_k\left[\left\|\widehat{\Delta}_i\left(x^{k+1}, x^k\right) - \Delta_i\left(x^{k+1}, x^k\right) \right\|^2\right] \mid [3] \right] \\
&+ \frac{20\cdot\mathcal{P}_{\mathcal{G}^k_C}}{C(1-\delta)}\omega \mathbb{E} \left[  \sum \limits_{i \in \mathcal{G}}\left\|\Delta_i\left(x^{k+1}, x^k\right) - \Delta\left(x^{k+1}, x^k\right)\right\|^2 + \|\Delta\left(x^{k+1}, x^k\right)\|^2\mid [3]\right]\\
&+\frac{2\cdot\mathcal{P}_{\mathcal{G}^k_C}}{C(1-\delta)} \mathbb{E} \left[ \sum \limits_{i \in \mathcal{G}}\left\|  \Delta_i\left(x^{k+1}, x^k\right) -  \Delta\left(x^{k+1}, x^k\right)\right\|^2 \mid [3] \right].
\end{align*}
Rearranging terms leads to 
\begin{align*}
   B_1^\prime  &\leq \frac{20\cdot\mathcal{P}_{\mathcal{G}^k_C}}{C(1-\delta)} (1+\omega) \mathbb{E} \left[  \sum \limits_{i \in \mathcal{G}}\mathbb{E}_k\left[\left\|\widehat{\Delta}_i\left(x^{k+1}, x^k\right) - \Delta_i\left(x^{k+1}, x^k\right) \right\|^2\right] \mid [3] \right] \\
&+ \frac{2\cdot\mathcal{P}_{\mathcal{G}^k_C}}{C(1-\delta)}(10\omega+1) \mathbb{E}\left[ \sum \limits_{i \in \mathcal{G}}\left\|\Delta_i\left(x^{k+1}, x^k\right) - \Delta\left(x^{k+1}, x^k\right)\right\|^2\mid[3]\right] \\
&+\frac{20\cdot\mathcal{P}_{\mathcal{G}^k_C}}{C(1-\delta)} \omega \mathbb{E}\left[\sum \limits_{i \in \mathcal{G}}\left\|  \Delta\left(x^{k+1}, x^k\right)\right\|^2 \mid [3]\right].
\end{align*}
Now we apply Assumptions \ref{assm:L-smoothness}, \ref{assm:global}, \ref{assm:local}:
\begin{align*}
   B_1^\prime  &\leq \frac{20\cdot\mathcal{P}_{\mathcal{G}^k_C}}{C(1-\delta)} (1+\omega) \mathbb{E} \left[ G \frac{\mathcal{L}_{ \pm}^2}{b} \|x^{k+1} - x^k\|^2\right]\\
&+ \frac{2\cdot\mathcal{P}_{\mathcal{G}^k_C}}{C(1-\delta)}(10\omega+1) \mathbb{E} \left[ G L_{ \pm}^2 \|x^{k+1} - x^k\|^2\right]\\
&+\frac{20\cdot\mathcal{P}_{\mathcal{G}^k_C}}{C(1-\delta)} \omega \mathbb{E} \left[ G L^2\left\|  x^{k+1} - x^k\right\|^2\right].
\end{align*}
Finally, we have
\begin{align*}
B_1^\prime&\leq \frac{2\cdot\mathcal{P}_{\mathcal{G}^k_C}\cdot G}{C(1-\delta)} \left( 10 \omega L^2+(10\omega+1) L_{ \pm}^2+\frac{10(\omega+1) \mathcal{L}_{ \pm}^2}{b} \right)\mathbb{E}\left[ \|x^{k+1} - x^k\|^2\right].
\end{align*}
Let us plug obtained results:
\begin{align*}
  B_1      &\leq \left(1+ \frac{p}{4}\right) \mathbb{E}\left[\left\|g^{k}-\nabla f(x^{k})\right\|^2\right]\\
    &+ \left(1+ \frac{4}{p}\right) \frac{2\cdot\mathcal{P}_{\mathcal{G}^k_C}\cdot G}{C(1-\delta)} \left( 10 \omega L^2+(10\omega+1) L_{ \pm}^2+\frac{10(\omega+1) \mathcal{L}_{ \pm}^2}{b} \right)\mathbb{E}\left[\|x^{k+1} - x^k\|^2\right].
\end{align*}

Let us consider the term $\mathbb{E}\left[ \left\|\frac{1}{G^k_{\widehat{C}}}\sum \limits_{i \in \cG_{\widehat{C}}^k} \nabla f_i(x^{k+1}) - \nabla f(x^{k+1})\right\|^2 \right]$:

\begin{align*}
    \mathbb{E}\left[ \left\|\frac{1}{G^k_{\widehat{C}}}\sum \limits_{i \in \cG_{\widehat{C}}^k} \nabla f_i(x^{k+1}) - \nabla f(x^{k+1})\right\|^2 \right] &\leq     \mathbb{E}\left[ \frac{1}{G^k_{\widehat{C}}}\sum \limits_{i \in \cG_{\widehat{C}}^k} \left\|\nabla f_i(x^{k+1}) - \nabla f(x^{k+1})\right\|^2 \right]\\
    &\leq \frac{1}{(1-\delta)\widehat{C}}  \mathbb{E}\left[ \sum \limits_{i \in \cG_{\widehat{C}}^k} \left\|\nabla f_i(x^{k+1}) - \nabla f(x^{k+1})\right\|^2 \right]\\
    & = \frac{1}{(1-\delta)\widehat{C}}  \mathbb{E}\left[ \sum \limits_{i \in \mathcal{G}} \mathcal{I}_{\mathcal{G}^k_{\widehat{C}}} \left\|\nabla f_i(x^{k+1}) - \nabla f(x^{k+1})\right\|^2 \right]
\end{align*}
Using definition of $\mathcal{P}_{\mathcal{G}^k_C}$ we get
\begin{align*}
    \mathbb{E}\left[ \left\|\frac{1}{G^k_{\widehat{C}}}\sum \limits_{i \in \cG_{\widehat{C}}^k} \nabla f_i(x^{k+1}) - \nabla f(x^{k+1})\right\|^2 \right] &\leq  \frac{\mathcal{P}_{\mathcal{G}^k_{\widehat{C}}} }{(1-\delta)\widehat{C}}  \mathbb{E}\left[ \sum \limits_{i \in \mathcal{G}} \left\|\nabla f_i(x^{k+1}) - \nabla f(x^{k+1})\right\|^2 \right]\\
    &\leq \frac{ G\cdot\mathcal{P }_{\mathcal{G}^k_{\widehat{C}}} }{(1-\delta)\widehat{C}G}  \mathbb{E}\left[ \sum \limits_{i \in \mathcal{G}} \left\|\nabla f_i(x^{k+1}) - \nabla f(x^{k+1})\right\|^2 \right]\\ 
\end{align*}
Using Assumption \ref{assm:het} we get
\begin{align}
    \mathbb{E}\left[ \left\|\frac{1}{G^k_{\widehat{C}}}\sum \limits_{i \in \cG_{\widehat{C}}^k} \nabla f_i(x^{k+1}) - \nabla f(x^{k+1})\right\|^2 \right] &\leq  \frac{ G\cdot\mathcal{P }_{\mathcal{G}^k_{\widehat{C}}} }{(1-\delta)\widehat{C}}  \mathbb{E}\left[ B\|\nabla f(x)\|^2+\zeta^2 \right] \notag \\
    & \leq  \frac{ \deltar n \cdot\mathcal{P }_{\mathcal{G}^k_{\widehat{C}}} }{(1-\delta)\frac{\deltar n }{\delta}}  \mathbb{E}\left[ B\|\nabla f(x)\|^2+\zeta^2 \right] \notag\\
    &= \frac{ \delta\cdot\mathcal{P }_{\mathcal{G}^k_{\widehat{C}}} }{(1-\delta)}  \mathbb{E}\left[ B\|\nabla f(x)\|^2+\zeta^2 \right] \label{eq:nsjknbvsbicusd}
    \end{align}
Also, we have 
\begin{align*}
    A_1 &= \mathbb{E}\left[\left\|\overline{g}^{k+1}-\nabla f(x^{k+1})\right\|^2\right]\\
    &\leq (1-p)p_G B_1 + (1-p)(1-p_{G})\mathbb{E}\left[\left\|g^k -\nabla f(x^{k})\right\|^2\right]+p\frac{ \delta\cdot\mathcal{P }_{\mathcal{G}^k_{\widehat{C}}} }{(1-\delta)}  \mathbb{E}\left[ B\|\nabla f(x)\|^2+\zeta^2 \right]\\
    &\leq (1-p)p_G \left(1+ \frac{p}{4}\right)\mathbb{E}\left[\left\|g^{k}-\nabla f(x^{k})\right\|^2\right]\\
    &+ \text{\scriptsize $(1-p)p_G\left(1+ \frac{4}{p}\right) \frac{2\cdot\mathcal{P}_{\mathcal{G}^k_C}\cdot G}{C(1-\delta)} \left( 10 \omega L^2+(10\omega+1) L_{ \pm}^2+\frac{10(\omega+1) \mathcal{L}_{ \pm}^2}{b} \right)\mathbb{E}\left[\|x^{k+1} - x^k\|^2\right]$}\\
    &+(1-p)(1-p_{G})\mathbb{E}\left[\left\|g^k -\nabla f(x^{k})\right\|^2\right]+p\frac{ \delta\cdot\mathcal{P }_{\mathcal{G}^k_{\widehat{C}}} }{(1-\delta)}  \mathbb{E}\left[ B\|\nabla f(x)\|^2+\zeta^2 \right].
\end{align*}
To simplify the bound we use $\left(1+\frac{p}{4}>1\right)$ and obtain 
\begin{align*} 
    A_1 
    &\leq (1-p)p_G \left(1+ \frac{p}{4}\right)\mathbb{E}\left[\left\|g^{k}-\nabla f(x^{k})\right\|^2\right]+p\frac{ \delta\cdot\mathcal{P }_{\mathcal{G}^k_{\widehat{C}}} }{(1-\delta)}  \mathbb{E}\left[ B\|\nabla f(x)\|^2+\zeta^2 \right]\\
    &+ \text{\scriptsize $(1-p)p_G\left(1+ \frac{4}{p}\right) \frac{2\cdot\mathcal{P}_{\mathcal{G}^k_C}\cdot G}{C(1-\delta)} \left( 10 \omega L^2+(10\omega+1) L_{ \pm}^2+\frac{10(\omega+1) \mathcal{L}_{ \pm}^2}{b} \right)\mathbb{E}\left[\|x^{k+1} - x^k\|^2\right]$}\\
    &+(1-p)(1-p_{G})\mathbb{E}\left[\left\|g^k -\nabla f(x^{k})\right\|^2\right]\\
    &\leq (1-p)p_G \left(1+ \frac{p}{4}\right)\mathbb{E}\left[\left\|g^{k}-\nabla f(x^{k})\right\|^2\right]\\
    &+ \text{\scriptsize $(1-p)p_G\left(1+ \frac{4}{p}\right) \frac{2\cdot\mathcal{P}_{\mathcal{G}^k_C}\cdot G}{C(1-\delta)} \left( 10 \omega L^2+(10\omega+1) L_{ \pm}^2+\frac{10(\omega+1) \mathcal{L}_{ \pm}^2}{b} \right)\mathbb{E}\left[\|x^{k+1} - x^k\|^2\right]$}\\
    &+(1-p)(1-p_{G})\left(1+\frac{p}{4}\right)\mathbb{E}\left[\left\|g^k -\nabla f(x^{k})\right\|^2\right]+p\frac{ \delta\cdot\mathcal{P }_{\mathcal{G}^k_{\widehat{C}}} }{(1-\delta)}  \mathbb{E}\left[ B\|\nabla f(x)\|^2+\zeta^2 \right]\\
      &\leq (1-p) \left(1+ \frac{p}{4}\right)\mathbb{E}\left[\left\|g^{k}-\nabla f(x^{k})\right\|^2\right]+p\frac{ \delta\cdot\mathcal{P }_{\mathcal{G}^k_{\widehat{C}}} }{(1-\delta)}  \mathbb{E}\left[ B\|\nabla f(x)\|^2+\zeta^2 \right]\\
    &+ \text{\scriptsize $(1-p)p_G\left(1+ \frac{4}{p}\right) \frac{2\cdot\mathcal{P}_{\mathcal{G}^k_C}n}{C} \left( 10 \omega L^2+(10\omega+1) L_{ \pm}^2+\frac{10(\omega+1) \mathcal{L}_{ \pm}^2}{b} \right)\mathbb{E}\left[\|x^{k+1} - x^k\|^2\right]$}.
    \end{align*}
\end{proof}

\begin{lemma}\label{lemma:premainA2}
 Let us define "ideal" estimator:
 \begin{equation*}
    \overline{g}^{k+1}= \begin{cases}
        \frac{1}{G^k_C}\sum \limits_{i \in \mathcal{G}^k_C} \nabla f_i(x^{k+1}),& c_n=1,\quad\quad\quad\quad\quad\quad\quad\quad\quad~\hspace{0.01cm} [1]\\
        g^k+\nabla f\left(x^{k+1}\right)-\nabla f\left(x^k\right),& c_n=0 \text{ and } G^k_C < (1-\delta)C, \hspace{0.325cm}[2]\\
        g^k+\frac{1}{G^k_C} \sum \limits_{i \in \mathcal{G}^k_C}\clip_{\lambda}\left(\mathcal{Q}\left(\widehat{\Delta}_i\left(x^{k+1}, x^k\right)\right)\right),& c_n=0 \text{ and } G^k_C \geq (1-\delta)C. \hspace{0.375cm}   [3]
    \end{cases}
\end{equation*}
Also let us introduce the notation
$$\texttt{ARAgg}_Q^{k+1} = \texttt{ARAgg}\left(\clip_{\lambda_{k+1}}\left(\cQ\left(\widehat{\Delta}_1(x^{k+1}, x^k)\right)\right),\ldots, \clip_{\lambda_{k+1}}\left(\cQ\left(\widehat{\Delta}_C(x^{k+1}, x^k)\right)\right)\right).$$ 
Then for all $k\geq 0$ the iterates produced by \algname{Byz-VR-MARINA-PP} (Algorithm~\ref{alg:byz_vr_marina}) satisfy

    \begin{align*}
        A_2&= \mathbb{E}\left[\left\|g^{k+1}-\overline{g}^{k+1}\right\|^2\right]\\
        &\leq p\mathbb{E}\left[\mathbb{E}_k\left[\left\|\texttt{ARAgg}\left(\nabla f_1(x^{k+1}), \ldots, \nabla f_{\widehat{C}}(x^{k+1})\right) - \nabla f(x^{k+1})\right\|^2\right]\mid [1]\right]\\
   & + (1-p)p_G \mathbb{E}\left[ \mathbb{E}_k\left[\left\| \frac{1}{G^k_C} \sum \limits_{i \in \mathcal{G}^k_C}\clip_{\lambda}\left(\mathcal{Q}\left(\widehat{\Delta}_i\left(x^{k+1}, x^k\right)\right)\right)  -   \texttt{ARAgg}_Q^{k+1}\right\|^2\mid [3]\right]\right]\\
   &+ (1-p)(1-p_G)\mathbb{E}\left[\mathbb{E}_k\left[\left\| \nabla f(x^{k+1}) - \nabla f(x^{k}) - \texttt{ARAgg}_Q^{k+1}\right\|^2\mid [2]\right]\right],
    \end{align*}
    where $p_G = \operatorname{Prob}\left\lbrace G^k_C \geq (1-\delta)C \right\rbrace $. 
\end{lemma}

\begin{proof}
    Using conditional expectations we have 
\begin{align*}
   A_2 &=  \mathbb{E}\left[\mathbb{E}_k\left[\left\|g^{k+1}-\overline{g}^{k+1}\right\|^2\right]\right]\\
   & = p\mathbb{E}\left[\mathbb{E}_k\left[\left\|\texttt{ARAgg}\left(\nabla f_1(x^{k+1}), \ldots, \nabla f_{\widehat{C}}(x^{k+1})\right) - \nabla f(x^{k+1})\right\|^2\right]\mid [1]\right]\\
   & \text{\small $+ (1-p)p_G \mathbb{E}\left[ \mathbb{E}_k\left[\left\| g^k+\frac{1}{G^k_C} \sum \limits_{i \in \mathcal{G}^k_C}\clip_{\lambda}\left(\mathcal{Q}\left(\widehat{\Delta}_i\left(x^{k+1}, x^k\right)\right)\right)  - \left(g^k + \texttt{ARAgg}_Q^{k+1}\right)\right\|^2\right]\mid [3]\right]$}\\
   &+ (1-p)(1-p_G)\mathbb{E}\left[\mathbb{E}_k\left[\left\| g^k+\nabla f(x^{k+1}) - \nabla f(x^{k}) - \left(g^k + \texttt{ARAgg}_Q^{k+1}\right) \right\|^2\right]\mid [2]\right].
   \end{align*}
After simplification we get the following bound:
   \begin{align*}
A_2   &\leq  p\mathbb{E}\left[\mathbb{E}_k\left[\left\|\texttt{ARAgg}\left(\nabla f_1(x^{k+1}), \ldots, \nabla f_{\widehat{C}}(x^{k+1})\right) - \nabla f(x^{k+1})\right\|^2\right]\mid [1]\right]\\
   & + (1-p)p_G \mathbb{E}\left[ \mathbb{E}_k\left[\left\| \frac{1}{G^k_C} \sum \limits_{i \in \mathcal{G}^k_C}\clip_{\lambda}\left(\mathcal{Q}\left(\widehat{\Delta}_i\left(x^{k+1}, x^k\right)\right)\right)  -   \texttt{ARAgg}_Q^{k+1}\right\|^2\mid [3]\right]\right]\\
   &+ (1-p)(1-p_G)\mathbb{E}\left[\mathbb{E}_k\left[\left\| \nabla f(x^{k+1}) - \nabla f(x^{k}) - \texttt{ARAgg}_Q^{k+1}\right\|^2\mid [2]\right]\right].
\end{align*} 
\end{proof}

\begin{lemma} 
\label{lemma:full_aggr}
Let Assumptions~\ref{assm:L-smoothness} and \ref{assm:het} hold and Aggregation Operator ($\texttt{ARAgg}$) satisfy Definition~\ref{def:aragg}. Then for all $k\geq 0$ the iterates produced by \algname{Byz-VR-MARINA-PP} (Algorithm~\ref{alg:byz_vr_marina}) satisfy
\begin{align*}
T_1 &=  \mathbb{E}\left[\mathbb{E}_k\left[\left\|\texttt{ARAgg}\left(\nabla f_1(x^{k+1}), \ldots, \nabla f_{\widehat{C}}(x^{k+1})\right) - \nabla f(x^{k+1})\right\|^2\right]\mid [1] \right]\\ 
&\leq \frac{8 G \mathcal{P}_{\mathcal{G}^k_{\widehat{C}}} c\delta}{(1-\delta)\widehat{C}}  B  \mathbb{E}\left[\left\|\nabla f\left(x^k\right)\right\|^2\right] + \frac{8 G \mathcal{P}_{\mathcal{G}^k_{\widehat{C}}} c\delta}{(1-\delta)\widehat{C}}  B  L^2 \mathbb{E}\left[\left\|x^{k+1}-x^k\right\|^2\right] + \frac{4 G \mathcal{P}_{\mathcal{G}^k_{\widehat{C}}} c\delta}{(1-\delta)\widehat{C}}  \zeta^2.
\end{align*}
\end{lemma}

\begin{proof}
Using the definition of aggregation operator, we have
    \begin{align*}
T_1 &= \mathbb{E}\left[\mathbb{E}_k\left[\left\|\texttt{ARAgg}\left(\nabla f_1(x^{k+1}), \ldots, \nabla f_{\widehat{C}}(x^{k+1})\right) - \nabla f(x^{k+1})\right\|^2\right]\mid [1]\right]\\
& \stackrel{(\text{Def.~\ref{def:aragg}})}{\leq}\mathbb{E}\left[\frac{c\delta}{G_{\widehat{C}}^k(G_{\widehat{C}}^k-1)} \sum_{\substack{i, l \in \mathcal{G}^k_{\widehat{C}} \\
i \neq l}} \mathbb{E}_k\left[\left\|\nabla f_i\left(x^{k+1}\right)-\nabla f_l\left(x^{k+1}\right)\right\|^2\mid [1]\right]\right] \\
& \stackrel{(\ref{eq:yung-1})}{\leq}\mathbb{E}\left[\frac{c\delta}{G_{\widehat{C}}^k(G_{\widehat{C}}^k-1)}\sum_{\substack{i, l \in \mathcal{G}^k_{\widehat{C}} \\
i \neq l}} \mathbb{E}\left[2\left\|\nabla f_i\left(x^{k+1}\right)-\nabla f\left(x^{k+1}\right)\right\|^2+2\left\|\nabla f_l\left(x^{k+1}\right)-\nabla f\left(x^{k+1}\right)\right\|^2\mid [1]\right]\right]\\
&=\mathbb{E}\left[\frac{c\delta}{G^k_{\widehat{C}}} \sum_{i \in \mathcal{G}^k_{\widehat{C}}} 4\mathbb{E}_k\left[\left\|\nabla f_i\left(x^{k+1}\right)-\nabla f\left(x^{k+1}\right)\right\|^2\mid [1]\right]\right]\\
&\leq \frac{\mathcal{P}_{\mathcal{G}^k_{\widehat{C}}} c\delta}{(1-\delta) \widehat{C}} \sum_{i \in \mathcal{G}} 4\mathbb{E}_k\left[\left\|\nabla f_i\left(x^{k+1}\right)-\nabla f\left(x^{k+1}\right)\right\|^2 \right]\\
&\stackrel{(\text{As.\ref{assm:het}})}{\leq} \frac{4 G \mathcal{P}_{\mathcal{G}^k_{\widehat{C}}} c\delta}{(1-\delta) \widehat{C}}   \left(B\mathbb{E}\left[\left\|\nabla f\left(x^{k+1}\right)\right\|^2\right]+ \zeta^2\right)\\
&\stackrel{(\ref{eq:yung-1})}{\leq} \quad \frac{8 G \mathcal{P}_{\mathcal{G}^k_{\widehat{C}}} c\delta}{(1-\delta) \widehat{C}}  B  \mathbb{E}\left[\left\|\nabla f\left(x^k\right)\right\|^2\right] + \frac{8 G \mathcal{P}_{\mathcal{G}^k_{\widehat{C}}} c\delta}{(1-\delta) \widehat{C}}  B \mathbb{E}\left[\left\|\nabla f\left(x^{k+1}\right)-\nabla f\left(x^k\right)\right\|^2\right]\\
& + \frac{4 G \mathcal{P}_{\mathcal{G}^k_{\widehat{C}}} c\delta}{(1-\delta) \widehat{C}}  \zeta^2\\
&\leq \quad  \frac{8 G \mathcal{P}_{\mathcal{G}^k_{\widehat{C}}} c\delta}{(1-\delta) \widehat{C}}  B  \mathbb{E}\left[\left\|\nabla f\left(x^k\right)\right\|^2\right] + \frac{8 G \mathcal{P}_{\mathcal{G}^k_{\widehat{C}}} c\delta}{(1-\delta) \widehat{C}}  B  L^2 \mathbb{E}\left[\left\|x^{k+1}-x^k\right\|^2\right] + \frac{4 G \mathcal{P}_{\mathcal{G}^k_{\widehat{C}}} c\delta}{(1-\delta) \widehat{C}}  \zeta^2.
\end{align*}
\end{proof}
\begin{lemma}
\label{lemma:good_aggr}
Let Assumptions~\ref{assm:L-smoothness}, \ref{assm:global}, \ref{assm:local} hold and the Compression Operator satisfy Definition~\ref{def:Q}. Also let us introduce the notation
$$\texttt{ARAgg}_Q^{k+1} = \texttt{ARAgg}\left(\clip_{\lambda_{k+1}}\left(\cQ\left(\widehat{\Delta}_1(x^{k+1}, x^k)\right)\right),\ldots, \clip_{\lambda_{k+1}}\left(\cQ\left(\widehat{\Delta}_C(x^{k+1}, x^k)\right)\right)\right).$$ 
Then for all $k\geq 0$ the iterates produced by \algname{Byz-VR-MARINA-PP} (Algorithm~\ref{alg:byz_vr_marina}) satisfy

    \begin{align*}
       T_2 &=   \mathbb{E}\left[ \mathbb{E}_k\left[\left\| \frac{1}{G^k_C} \sum \limits_{i \in \mathcal{G}^k_C}\clip_{\lambda}\left(\mathcal{Q}\left(\widehat{\Delta}_i\left(x^{k+1}, x^k\right)\right)\right)  -   \texttt{ARAgg}_Q^{k+1}\right\|^2\mid [3]\right]\right]\\
       &\leq \frac{8G\mathcal{P}_{\mathcal{G}^k_C}}{(1-\delta)C} \left( 10(1+\omega)\frac{\mathcal{L}_{ \pm}^2}{b} + (10\omega+1) L_{ \pm}^2 + 10\omega  L^2 \right)c\delta\mathbb{E}\left[\|x^{k+1} - x^k\|^2\right],
\end{align*}
where $\mathcal{P}_{\mathcal{G}^k_C} =  \operatorname{Prob}\left\lbrace i \in \mathcal{G}^k_C \mid G^k_C \geq\left(1-\delta\right) C\right\rbrace$.
\end{lemma}
\begin{proof}
By the definition of robust aggregation, we have 
\begin{align*}
   T_2 &=   \mathbb{E}\left[ \mathbb{E}_k\left[\left\| \frac{1}{G^k_C} \sum \limits_{i \in \mathcal{G}^k_C}\clip_{\lambda}\left(\mathcal{Q}\left(\widehat{\Delta}_i\left(x^{k+1}, x^k\right)\right)\right)  -   \texttt{ARAgg}_Q^{k+1}\right\|^2\mid [3]\right]\right]\\
   &\text{ \scriptsize $\leq  \mathbb{E}\left[  \frac{c \delta}{D_2} \sum_{\substack{i, l \in \mathcal{G}_C^k \\
i \neq l}}    
\mathbb{E}_k\left[\left\| \clip_{\lambda}\left(\mathcal{Q}\left(\widehat{\Delta}_i\left(x^{k+1}, x^k\right)\right)\right) - \clip_{\lambda}\left(\mathcal{Q}\left(\widehat{\Delta}_l\left(x^{k+1}, x^k\right)\right)\right) \right\|^2\mid [3]\right]\right]$},
\end{align*}
where $D_2 = G^k_C(G^k_C-1)$. Next, we consider pair-wise differences:
\begin{align*}
T_{2}^\prime(i,l) &= \mathbb{E}_k\left[\left\| \clip_{\lambda}\left(\mathcal{Q}\left(\widehat{\Delta}_i\left(x^{k+1}, x^k\right)\right)\right) - \clip_{\lambda}\left(\mathcal{Q}\left(\widehat{\Delta}_l\left(x^{k+1}, x^k\right)\right)\right) \right\|^2\mid [3]\right]\\
& \stackrel{(\ref{eq:yung-1})}{\leq}   \text{\tiny $2\mathbb{E}_k\left[\left\| \clip_{\lambda}\left(\mathcal{Q}\left(\widehat{\Delta}_i\left(x^{k+1}, x^k\right)\right)\right)  - \Delta_i\left(x^{k+1}, x^k\right) + \Delta_l\left(x^{k+1}, x^k\right) -\clip_{\lambda}\left(\mathcal{Q}\left(\widehat{\Delta}_l\left(x^{k+1}, x^k\right)\right)\right) \right\|^2\mid [3]\right]$}\\
&+ 2 \text{$\small  \mathbb{E}_k\left[\left\|  \Delta_i\left(x^{k+1}, x^k\right) - \Delta_l\left(x^{k+1}, x^k\right) \right\|^2\mid [3]\right]$}\\
& \stackrel{(\ref{eq:yung-1})}{\leq}  4 \text{$\small \mathbb{E}_k\left[\left\| \clip_{\lambda}\left(\mathcal{Q}\left(\widehat{\Delta}_i\left(x^{k+1}, x^k\right)\right)\right) -\Delta_i\left(x^{k+1}, x^k\right) \right\|^2\mid [3]\right]$}\\
&+ 4 \text{$\small \mathbb{E}_k\left[\left\| \Delta_l\left(x^{k+1}, x^k\right) - \clip_{\lambda}\left(\mathcal{Q}\left(\widehat{\Delta}_l\left(x^{k+1}, x^k\right)\right)\right) \right\|^2\mid [3]\right]$}\\
&+ 2 \text{$\small \mathbb{E}_k\left[\left\|  \Delta_l\left(x^{k+1}, x^k\right) - \Delta_i\left(x^{k+1}, x^k\right) \right\|^2\mid [3]\right]]$}\\
& \stackrel{(\ref{eq:yung-1})}{\leq}  4 \text{$\small \mathbb{E}_k\left[\left\| \clip_{\lambda}\left(\mathcal{Q}\left(\widehat{\Delta}_i\left(x^{k+1}, x^k\right)\right)\right) -\Delta_i\left(x^{k+1}, x^k\right)  \right\|^2\mid [3]\right]$}\\
&+ 4 \text{$\small \mathbb{E}_k\left[\left\| \Delta_l\left(x^{k+1}, x^k\right) - \clip_{\lambda}\left(\mathcal{Q}\left(\widehat{\Delta}_l\left(x^{k+1}, x^k\right)\right)\right) \right\|^2\mid [3]\right]$}\\
&+ 4 \text{$\small  \mathbb{E}_k\left[\left\|  \Delta_l\left(x^{k+1}, x^k\right) - \Delta\left(x^{k+1}, x^k\right) \right\|^2\mid [3]\right]$}\\
&+ 4\text{$\small \mathbb{E}_k\left[\left\|  \Delta_i\left(x^{k+1}, x^k\right) - \Delta\left(x^{k+1}, x^k\right)\right\|^2\mid [3]\right]$}.
\end{align*}
Now we can combine all parts together:
\begin{align*}
    \widehat{T}_2 & = \mathbb{E}\left[\frac{1}{G^k_C(G^k_C - 1)}  \sum_{\substack{i, l \in \mathcal{G}_C^k \\
i \neq l}} T_{2}^\prime(i,l) \right]  \\
& \leq   \mathbb{E}\left[ \frac{1}{D_2} \sum_{\substack{i, l \in \mathcal{G}_C^k \\
i \neq l}} 4  \mathbb{E}_k\left[\left\| \clip_{\lambda}\left(\mathcal{Q}\left(\widehat{\Delta}_i\left(x^{k+1}, x^k\right)\right)\right) -\Delta_i\left(x^{k+1}, x^k\right)  \right\|^2\mid [3]\right]\right]\\
&+  \mathbb{E}\left[  \frac{1}{D_2}  \sum_{\substack{i, l \in \mathcal{G}_C^k \\
i \neq l}} 4\mathbb{E}_k\left[\left\| \Delta_l\left(x^{k+1}, x^k\right) - \clip_{\lambda}\left(\mathcal{Q}\left(\widehat{\Delta}_l\left(x^{k+1}, x^k\right)\right)\right) \right\|^2\mid [3]\right]\right]\\
&+    \mathbb{E}\left[  \frac{1}{D_2}  \sum_{\substack{i, l \in \mathcal{G}_C^k \\
i \neq l}} 4  \mathbb{E}_k\left[\left\|  \Delta_l\left(x^{k+1}, x^k\right) - \Delta\left(x^{k+1}, x^k\right) \right\|^2\mid [3]\right]\right]\\
&+   \mathbb{E}\left[  \frac{1}{D_2}  \sum_{\substack{i, l \in \mathcal{G}_C^k \\
i \neq l}} 4\small \mathbb{E}_k\left[\left\|  \Delta_i\left(x^{k+1}, x^k\right) - \Delta\left(x^{k+1}, x^k\right)\right\|^2\mid [3]\right]\right].
\end{align*}
Rearranging the terms, we obtain
\begin{align*}
    \widehat{T}_2& \leq   \mathbb{E}\left[ \frac{1}{D_2} \sum_{\substack{i, l \in \mathcal{G}_C^k \\
i \neq l}} 8  \mathbb{E}_k\left[\left\| \clip_{\lambda}\left(\mathcal{Q}\left(\widehat{\Delta}_i\left(x^{k+1}, x^k\right)\right)\right) -\Delta_i\left(x^{k+1}, x^k\right)  \right\|^2\mid [3]\right]\right]\\
&+    \mathbb{E}\left[  \frac{1}{D_2}  \sum_{\substack{i, l \in \mathcal{G}_C^k \\
i \neq l}} 8 \mathbb{E}_k\left[\left\|  \Delta_i\left(x^{k+1}, x^k\right) - \Delta\left(x^{k+1}, x^k\right) \right\|^2\mid [3]\right]\right].
\end{align*}
It leads to 
\begin{align*}
    \widehat{T}_2&\leq \mathbb{E}\left[ \frac{1}{G^k_C} \sum_{i \in \mathcal{G}^k_C} 8  \mathbb{E}_k\left[\left\| \clip_{\lambda}\left(\mathcal{Q}\left(\widehat{\Delta}_i\left(x^{k+1}, x^k\right)\right)\right) -\Delta_i\left(x^{k+1}, x^k\right)  \right\|^2\mid [3]\right] \right]\\
    &+\mathbb{E}\left[ \frac{1}{G^k_C} \sum_{i \in \mathcal{G}^k_C} 8 \mathbb{E}_k\left[\left\|  \Delta_i\left(x^{k+1}, x^k\right) - \Delta\left(x^{k+1}, x^k\right) \right\|^2\mid [3]\right] \right]\\
    &\overset{\text{Lemma~\ref{lemma:clipping}}}{\leq} \mathbb{E}\left[ \frac{1}{G^k_C} \sum_{i \in \mathcal{G}^k_C} 80  \mathbb{E}_k\left[\left\| \mathcal{Q}\left(\widehat{\Delta}_i\left(x^{k+1}, x^k\right)\right) -\Delta_i\left(x^{k+1}, x^k\right)  \right\|^2\mid [3]\right] \right]\\
    &+\mathbb{E}\left[ \frac{1}{G^k_C} \sum_{i \in \mathcal{G}^k_C} 8 \mathbb{E}_k\left[\left\|  \Delta_i\left(x^{k+1}, x^k\right) - \Delta\left(x^{k+1}, x^k\right) \right\|^2\mid [3]\right] \right].
\end{align*}
Using variance decomposition we get
\begin{align*}
        \widehat{T}_2&\leq \mathbb{E}\left[ \frac{1}{G^k_C} \sum_{i \in \mathcal{G}^k_C} 80  \mathbb{E}_k\left[\left\| \mathcal{Q}\left(\widehat{\Delta}_i\left(x^{k+1}, x^k\right)\right) \right\|^2\mid [3] \right]\right]\\
        &-\mathbb{E}\left[ \frac{1}{G^k_C} \sum_{i \in \mathcal{G}^k_C} 80  \mathbb{E}_k\left[\left\| \Delta_i\left(x^{k+1}, x^k\right) \right\|^2\mid [3]\right] \right]\\
&+\mathbb{E}\left[ \frac{1}{G^k_C} \sum_{i \in \mathcal{G}^k_C} 8 \mathbb{E}_k\left[\left\|  \Delta_i\left(x^{k+1}, x^k\right) - \Delta\left(x^{k+1}, x^k\right) \right\|^2\mid [3]\right] \right].
\end{align*}    
Using properties of unbiased compressors (Definition~\ref{def:Q}) we have 
\begin{align*}
        \widehat{T}_2&\leq \mathbb{E}\left[ \frac{1}{G^k_C} \sum_{i \in \mathcal{G}^k_C} 80(1+\omega)  \mathbb{E}_k\left[\left\| \widehat{\Delta}_i\left(x^{k+1}, x^k\right) \right\|^2\mid [3] \right]\right]\\
        &-\mathbb{E}\left[ \frac{1}{G^k_C} \sum_{i \in \mathcal{G}^k_C} 80  \mathbb{E}_k\left[\left\| \Delta_i\left(x^{k+1}, x^k\right) \right\|^2\mid [3]\right] \right]\\
&+\mathbb{E}\left[ \frac{1}{G^k_C} \sum_{i \in \mathcal{G}^k_C} 8 \mathbb{E}_k\left[\left\|  \Delta_i\left(x^{k+1}, x^k\right) - \Delta\left(x^{k+1}, x^k\right) \right\|^2\mid [3]\right] \right].
\end{align*}
Also we have
\begin{align*}
\widehat{T}_2&\leq \mathbb{E}\left[ \frac{1}{G^k_C} \sum_{i \in \mathcal{G}^k_C} 80(1+\omega)  \mathbb{E}_k\left[\left\| \widehat{\Delta}_i\left(x^{k+1}, x^k\right) - \Delta_i\left(x^{k+1}, x^k\right)\right\|^2\mid [3] \right]\right]\\
        &+\mathbb{E}\left[ \frac{1}{G^k_C} \sum_{i \in \mathcal{G}^k_C} 80(1+\omega)  \mathbb{E}_k\left[\left\| \Delta_i\left(x^{k+1}, x^k\right) \right\|^2\mid [3]\right] \right]\\
        &-\mathbb{E}\left[ \frac{1}{G^k_C} \sum_{i \in \mathcal{G}^k_C} 80  \mathbb{E}_k\left[\left\| \Delta_i\left(x^{k+1}, x^k\right) \right\|^2\mid [3]\right] \right]\\
&+\mathbb{E}\left[ \frac{1}{G^k_C} \sum_{i \in \mathcal{G}^k_C} 8 \mathbb{E}_k\left[\left\|  \Delta_i\left(x^{k+1}, x^k\right) - \Delta\left(x^{k+1}, x^k\right) \right\|^2\mid [3]\right] \right].
\end{align*} 
Let us simplify the inequality:
\begin{align*}
    \widehat{T}_2 &\leq \mathbb{E}\left[ \frac{1}{G^k_C} \sum_{i \in \mathcal{G}^k_C} 80(1+\omega)  \mathbb{E}_k\left[\left\| \widehat{\Delta}_i\left(x^{k+1}, x^k\right) - \Delta_i\left(x^{k+1}, x^k\right)\right\|^2\mid [3] \right]\right]\\
        &+\mathbb{E}\left[ \frac{1}{G^k_C} \sum_{i \in \mathcal{G}^k_C} 80\omega  \mathbb{E}_k\left[\left\| \Delta_i\left(x^{k+1}, x^k\right) \right\|^2\mid [3]\right] \right]\\
&+\mathbb{E}\left[ \frac{1}{G^k_C} \sum_{i \in \mathcal{G}^k_C} 8 \mathbb{E}_k\left[\left\|  \Delta_i\left(x^{k+1}, x^k\right) - \Delta\left(x^{k+1}, x^k\right) \right\|^2\mid [3]\right] \right].
\end{align*}
Using a variance decomposition once again, we get
\begin{align*}
    \widehat{T}_2 &\leq \mathbb{E}\left[ \frac{1}{G^k_C} \sum_{i \in \mathcal{G}^k_C} 80(1+\omega)  \mathbb{E}_k\left[\left\| \widehat{\Delta}_i\left(x^{k+1}, x^k\right) - \Delta_i\left(x^{k+1}, x^k\right)\right\|^2\mid [3] \right]\right]\\
        &+\mathbb{E}\left[ \frac{1}{G^k_C} \sum_{i \in \mathcal{G}^k_C} 80\omega  \mathbb{E}_k\left[\left\| \Delta_i\left(x^{k+1}, x^k\right) - \Delta\left(x^{k+1}, x^k\right)  \right\|^2\mid [3]\right] \right]\\
&+\mathbb{E}\left[ \frac{1}{G^k_C} \sum_{i \in \mathcal{G}^k_C} 8 \mathbb{E}_k\left[\left\|  \Delta_i\left(x^{k+1}, x^k\right) - \Delta\left(x^{k+1}, x^k\right) \right\|^2\mid [3]\right] \right]\\
&+ \mathbb{E}\left[ \frac{1}{G^k_C} \sum_{i \in \mathcal{G}^k_C} 80\omega  \mathbb{E}_k\left[\left\| \Delta\left(x^{k+1}, x^k\right)  \right\|^2\mid [3]\right] \right].
\end{align*}
Using a similar argument to the one used in the previous lemma, we obtain
\begin{align*}
    \widehat{T}_2 &\leq \mathbb{E}\left[ \frac{\mathcal{P}_{\mathcal{G}^k_C}}{(1-\delta)C} \sum_{i \in \mathcal{G}} 80(1+\omega)  \mathbb{E}_k\left[\left\| \widehat{\Delta}_i\left(x^{k+1}, x^k\right) - \Delta_i\left(x^{k+1}, x^k\right)\right\|^2\mid [3] \right]\right]\\
        &+\mathbb{E}\left[ \frac{\mathcal{P}_{\mathcal{G}^k_C}}{(1-\delta)C} \sum_{i \in \mathcal{G}} 80\omega  \mathbb{E}_k\left[\left\| \Delta_i\left(x^{k+1}, x^k\right) - \Delta\left(x^{k+1}, x^k\right)  \right\|^2\mid [3]\right] \right]\\
&+\mathbb{E}\left[ \frac{\mathcal{P}_{\mathcal{G}^k_C}}{(1-\delta)C} \sum_{i \in \mathcal{G}} 8 \mathbb{E}_k\left[\left\|  \Delta_i\left(x^{k+1}, x^k\right) - \Delta\left(x^{k+1}, x^k\right) \right\|^2\mid [3]\right] \right]\\
&+ \mathbb{E}\left[ \frac{\mathcal{P}_{\mathcal{G}^k_C}}{(1-\delta)C} \sum_{i \in \mathcal{G}} 80\omega  \mathbb{E}_k\left[\left\| \Delta\left(x^{k+1}, x^k\right)  \right\|^2\mid [3]\right] \right].
\end{align*}
Using Assumptions \ref{assm:L-smoothness}, \ref{assm:global}, \ref{assm:local}: 
    \begin{align*}
    \widehat{T}_2 &\leq \mathbb{E}\left[  \frac{80(1+\omega)  G\mathcal{P}_{\mathcal{G}^k_C}\mathcal{L}_{ \pm}^2}{(1-\delta)Cb}\|x^{k+1}-x^k\|^2 \right] +\mathbb{E}\left[  \frac{8(10\omega+1)G\mathcal{P}_{\mathcal{G}^k_C}   L_{ \pm}^2}{(1-\delta)C} \|x^{k+1}-x^k\|^2\right]\\
        &+\mathbb{E}\left[  \frac{80G\mathcal{P}_{\mathcal{G}^k_C} \omega  L^2}{(1-\delta)C} \|x^{k+1}-x^k\|^2\right].
\end{align*}
Finally, we obtain
\begin{align*}
       T_2 &=   \mathbb{E}\left[ \mathbb{E}_k\left[\left\| \frac{1}{G^k_C} \sum \limits_{i \in \mathcal{G}^k_C}\clip_{\lambda}\left(\mathcal{Q}\left(\widehat{\Delta}_i\left(x^{k+1}, x^k\right)\right)\right)  -   \texttt{ARAgg}_Q^{k+1}\right\|^2\mid [3]\right]\right]\\
       &\leq \frac{8G\mathcal{P}_{\mathcal{G}^k_C}}{(1-\delta)C} \left( 10(1+\omega)\frac{\mathcal{L}_{ \pm}^2}{b} + (10\omega+1) L_{ \pm}^2 + 10\omega  L^2 \right)c\delta \mathbb{E}\left[\|x^{k+1} - x^k\|^2\right].
\end{align*}

\end{proof}

\begin{lemma}\label{lemma:bad_aggr}
Let Assumptions~\ref{assm:bounded-aggr} and \ref{assm:L-smoothness} hold. Also let us introduce the notation
$$\texttt{ARAgg}_Q^{k+1} = \texttt{ARAgg}\left(\clip_{\lambda_{k+1}}\left(\cQ\left(\widehat{\Delta}_1(x^{k+1}, x^k)\right)\right),\ldots, \clip_{\lambda_{k+1}}\left(\cQ\left(\widehat{\Delta}_C(x^{k+1}, x^k)\right)\right)\right).$$ 
Assume that $\lambda_{k+1} = \alpha_{\lambda_{k+1}} \|x^{k+1} - x^k\| $.
Then for all $k\geq 0$ the iterates produced by \algname{Byz-VR-MARINA-PP} (Algorithm~\ref{alg:byz_vr_marina}) satisfy

        \begin{align*}
        T_3 &= \mathbb{E}\left[\mathbb{E}_k\left[\left\| \nabla f(x^{k+1}) - \nabla f(x^{k}) - \texttt{ARAgg}_Q^{k+1}\right\|^2\mid [2]\right]\right]\\
       &\leq 2(L^2+ F_{\cA}^2\alpha^2_{\lambda_{k+1}})\mathbb{E}\left[\left\|  x^{k+1} - x^k\right\|^2\right]
        \end{align*}
\end{lemma}
\begin{proof}
      \begin{align*}
    T_3 &=   \mathbb{E}\left[\mathbb{E}_k\left[\left\| \nabla f(x^{k+1}) - \nabla f(x^{k}) - \texttt{ARAgg}_Q^{k+1}\right\|^2\mid [2]\right]\right]\\
    &\stackrel{(\ref{eq:yung-1})}{\leq} \mathbb{E}\left[\mathbb{E}_k\left[2\left\| \nabla f(x^{k+1}) - \nabla f(x^{k}) \right\|^2+2\left\| \texttt{ARAgg}_Q^{k+1}\right\|^2\mid [2]\right]\right]\\
            \end{align*}
            Using $L$-smoothness and Assumption~\ref{assm:bounded-aggr} we have 

          \begin{align*}
    T_3 
    &\stackrel{(\ref{eq:yung-1})}{\leq} \mathbb{E}\left[\mathbb{E}_k\left[2L^2\left\|  x^{k+1} - x^k\right\|^2+2F_{\cA}^2\lambda^2_{k+1}\mid [2]\right]\right]\\
     &\leq\mathbb{E}\left[\mathbb{E}_k\left[2L^2\left\|  x^{k+1} - x^k\right\|^2+2F_{\cA}^2\alpha^2_{\lambda_{k+1}}\|x^{k+1} - x^k\|^2\mid [2]\right]\right]\\
      &\leq 2(L^2+F_{\cA}^2\alpha^2_{\lambda_{k+1}})\mathbb{E}\left[\left\|  x^{k+1} - x^k\right\|^2\right].
            \end{align*}        
\end{proof}

\begin{lemma}
\label{lemma:final_lemma}
    Let Assumptions ~\ref{assm:bounded-aggr}, \ref{assm:L-smoothness}, \ref{assm:global}, \ref{assm:local}, \ref{assm:het} hold and Compression Operator satisfy Definition~\ref{def:Q}.  Also let us introduce the notation
$$\texttt{ARAgg}_Q^{k+1} = \texttt{ARAgg}\left(\clip_{\lambda_{k+1}}\left(\cQ\left(\widehat{\Delta}_1(x^{k+1}, x^k)\right)\right),\ldots, \clip_{\lambda_{k+1}}\left(\cQ\left(\widehat{\Delta}_C(x^{k+1}, x^k)\right)\right)\right).$$ 
Then for all $k\geq 0$ the iterates produced by \algname{Byz-VR-MARINA-PP} (Algorithm~\ref{alg:byz_vr_marina}) satisfy

    \begin{align*}
    \mathbb{E}\left[\left\|g^{k+1}-\nabla f\left(x^{k+1}\right)\right\|^2\right] &\leq  \left(1-\frac{p}{4}\right) \mathbb{E}\left[\left\|g^{k}-\nabla f\left(x^{k}\right)\right\|^2\right]\\
    & +  \widehat{B}\mathbb{E}\left[\left\|\nabla f\left(x^k\right)\right\|^2\right]  + \widehat{D}\zeta^2+\frac{pA}{4}\|x^{k+1} - x^k\|^2,
\end{align*}
where \begin{align*}
    A& = \frac{4}{p}\left( \frac{80}{p} \frac{p_G \mathcal{P}_{\mathcal{G}^k_C} n}{C} \omega + 24 \frac{ G\mathcal{P}_{\mathcal{G}^k_{\widehat{C}}} c\delta}{(1-\delta)\widehat C} B + \frac{4}{p}(1-p_G)+\frac{160}{p}p_G \frac{G\mathcal{P}_{\mathcal{G}^k_C} }{(1-\delta)C}c\delta\omega  \right) L^2\\
    &+\frac{4}{p}\left( \frac{8}{p} \frac{p_G \mathcal{P}_{\mathcal{G}^k_C} n}{C} \left( 10\omega + 1 \right) + \frac{16}{p}p_G \frac{G\mathcal{P}_{\mathcal{G}^k_C} }{(1-\delta)C}c\delta(10\omega+1) \right)L_{ \pm}^2\\
    &+\frac{4}{p}\left( \frac{160}{p} p_G \frac{G \mathcal{P}_{\mathcal{G}^k_C}}{(1-\delta) C} (1+\omega)c\delta +\frac{80}{p} p_G \mathcal{P}_{\mathcal{G}^k_C} (1+\omega) \frac{n}{C} \right)\frac{\mathcal{L}_{ \pm}^2}{b}\\
    &+\frac{4}{p}\left( \frac{4}{p}(1-p_G)F_{\cA}^2\alpha^2_{\lambda_{k+1}} \right),
\end{align*}
\begin{align*}
    \widehat{B} = \frac{2 B\delta\mathcal{P}_{\mathcal{G}^k_{\widehat{C}}} }{1-\delta}\left(\frac{12cG}{\widehat C} + p \right), \quad \widehat{D} = \frac{2 \delta\mathcal{P}_{\mathcal{G}^k_{\widehat{C}}} }{1-\delta}\left(\frac{6cG}{\widehat C} + p \right)
\end{align*}
and where $p_G = \operatorname{Prob}\left\lbrace G^k_C \geq (1-\delta)C \right\rbrace $ and $\mathcal{P}_{\mathcal{G}^k_C} =  \operatorname{Prob}\left\lbrace i \in \mathcal{G}^k_C \mid G^k_C \geq\left(1-\delta\right) C\right\rbrace$.
\end{lemma}

\begin{proof}

Let us combine bounds for $A_1$ and $A_2$ together:
\begin{align*}
     A_0 &=    \mathbb{E}\left[\left\|g^{k+1}-\nabla f\left(x^{k+1}\right)\right\|^2\right]\\
     & \leq \left(1+\frac{p}{2}\right)\mathbb{E}\left[\left\|\overline{g}^{k+1}-\nabla f\left(x^{k+1}\right)\right\|^2\right]+\left(1+\frac{2}{p}\right)\mathbb{E}\left[\left\|g^{k+1}-\overline{g}^{k+1}\right\|^2\right]\\
     &\leq \left(1+\frac{p}{2}\right)A_1 + \left(1+\frac{2}{p}\right)A_2\\
       & \leq \left(1+\frac{p}{2}\right)(1-p) \left(1+ \frac{p}{4}\right)\mathbb{E}\left[\left\|g^{k}-\nabla f\left(x^{k}\right)\right\|^2\right]\\
    &+ \text{\scriptsize $\left(1+\frac{p}{2}\right)(1-p)p_G\left(1+ \frac{4}{p}\right) \frac{2\cdot\mathcal{P}_{\mathcal{G}^k_C}n}{C} \left( 10 \omega L^2+(10\omega+1) L_{ \pm}^2+\frac{10(\omega+1) \mathcal{L}_{ \pm}^2}{b} \right) \mathbb{E} \left[ \|x^{k+1} - x^k\|^2\right]$}\\
    &+ \left(1+\frac{p}{2}\right)p\left(\frac{ \delta\cdot\mathcal{P }_{\mathcal{G}^k_{\widehat{C}}} }{(1-\delta)}  \mathbb{E}\left[ B\|\nabla f(x)\|^2+\zeta^2 \right]\right)\\
    &+ \left(1+\frac{2}{p}\right)p\mathbb{E}\left[\mathbb{E}_k\left[\left\|\texttt{ARAgg}\left(\nabla f_1(x^{k+1}), \ldots, \nabla f_{\widehat{C}}(x^{k+1})\right) - \nabla f(x^{k+1})\right\|^2\right]\mid [1]\right]\\
   & \text{\small$+ \left(1+\frac{2}{p}\right)(1-p)p_G \mathbb{E}\left[ \mathbb{E}_k\left[\left\| \frac{1}{G^k_C} \sum \limits_{i \in \mathcal{G}^k_C}\clip_{\lambda}\left(\mathcal{Q}\left(\widehat{\Delta}_i\left(x^{k+1}, x^k\right)\right)\right)  -   \texttt{ARAgg}_Q^{k+1}\right\|^2\mid [3]\right]\right]$}\\
   &+ \left(1+\frac{2}{p}\right)(1-p)(1-p_G)\mathbb{E}\left[\mathbb{E}_k\left[\left\| \nabla f(x^{k+1}) - \nabla f(x^{k}) - \texttt{ARAgg}_Q^{k+1}\right\|^2\mid [2]\right]\right].
   \end{align*}
   Finally, we obtain the following bound:
   \begin{align*}
A_0  & \stackrel{(\ref{eq:yung-1})}{\leq}  \left(1-\frac{p}{4}\right) \mathbb{E}\left[\left\|g^{k}-\nabla f\left(x^{k}\right)\right\|^2\right]\\
  &+  \frac{8}{p} \frac{\mathcal{P}_{\mathcal{G}^k_C}n}{C}p_G  \left( 10 \omega L^2+(10\omega+1) L_{ \pm}^2+\frac{10(\omega+1) \mathcal{L}_{ \pm}^2}{b} \right)\mathbb{E}\left[\|x^{k+1} - x^k\|^2\right]\\
  &+2p\left(\frac{ \delta\cdot\mathcal{P }_{\mathcal{G}^k_{\widehat{C}}} }{1-\delta}  \mathbb{E}\left[ B\|\nabla f(x)\|^2+\zeta^2 \right]\right)\\
      &+ \left(p+2\right)\mathbb{E}\left[\mathbb{E}_k\left[\left\|\texttt{ARAgg}\left(\nabla f_1(x^{k+1}), \ldots, \nabla f_n(x^{k+1})\right) - \nabla f(x^{k+1})\right\|^2\right]\mid [1] \right]\\
   & + \frac{2}{p}p_G \mathbb{E}\left[ \mathbb{E}_k\left[\left\| \frac{1}{G^k_C} \sum \limits_{i \in \mathcal{G}^k_C}\clip_{\lambda}\left(\mathcal{Q}\left(\widehat{\Delta}_i\left(x^{k+1}, x^k\right)\right)\right)  -   \texttt{ARAgg}_Q^{k+1}\right\|^2\mid [3]\right]\right]\\
   &+ \frac{2}{p}(1-p_G)\mathbb{E}\left[\mathbb{E}_k\left[\left\| \nabla f(x^{k+1}) - \nabla f(x^{k}) - \texttt{ARAgg}_Q^{k+1}\right\|^2\mid [2]\right]\right]
\end{align*}
Now, we can apply Lemmas \ref{lemma:full_aggr}, \ref{lemma:good_aggr}, \ref{lemma:bad_aggr}:
   \begin{align*}
   A_0 &=  \mathbb{E}\left[\left\|g^{k+1}-\nabla f\left(x^{k+1}\right)\right\|^2\right]\\ 
   & \leq \left(1-\frac{p}{4}\right) \mathbb{E}\left[\left\|g^{k}-\nabla f\left(x^{k}\right)\right\|^2\right]\\
  &+  \frac{8}{p} \frac{\mathcal{P}_{\mathcal{G}^k_C}n}{C}p_G  \left( 10 \omega L^2+(10\omega+1) L_{ \pm}^2+\frac{10(\omega+1) \mathcal{L}_{ \pm}^2}{b} \right)\mathbb{E}\left[\|x^{k+1} - x^k\|^2\right]\\
      &+ \left(p+2\right)\left(\frac{8 G \mathcal{P}_{\mathcal{G}^k_{\widehat{C}}} c\delta}{(1-\delta)\widehat{C}}  B  \mathbb{E}\left[\left\|\nabla f\left(x^k\right)\right\|^2\right] + \frac{8 G \mathcal{P}_{\mathcal{G}^k_{\widehat{C}}} c\delta}{(1-\delta)\widehat{C}}  B  L^2 \mathbb{E}\left[\left\|x^{k+1}-x^k\right\|^2\right]\right)\\
      &+4\left(p+2\right) \frac{G \mathcal{P}_{\mathcal{G}^k_{\widehat{C}}} c\delta}{(1-\delta)\widehat{C}}  \zeta^2\\
   & + \frac{2}{p}p_G\mathbb{E}\left[  80(1+\omega) \frac{G\mathcal{P}_{\mathcal{G}^k_C} }{(1-\delta)C} \frac{\mathcal{L}_{ \pm}^2}{b}c\delta \|x^{k+1}-x^k\|^2 \right] \\
    &    + \frac{2}{p}p_G \mathbb{E}\left[  8(10\omega+1)\frac{G\mathcal{P}_{\mathcal{G}^k_C} }{(1-\delta)C} L_{ \pm}^2c\delta  \|x^{k+1}-x^k\|^2\right]\\
     &   + \frac{2}{p}p_G \mathbb{E}\left[  80 \frac{G\mathcal{P}_{\mathcal{G}^k_C} }{(1-\delta)C} \omega  L^2 c\delta  \|x^{k+1}-x^k\|^2\right]\\
   &+ \frac{2}{p}(1-p_G)2(L^2+F_{\cA}^2\alpha^2_{\lambda_{k+1}})\mathbb{E}\left[\left\|  x^{k+1} - x^k\right\|^2\right]\\
 &  +2p\frac{ \delta\mathcal{P }_{\mathcal{G}^k_{\widehat{C}}} }{1-\delta}  \mathbb{E}\left[ B\|\nabla f(x)\|^2+\zeta^2 \right].
\end{align*}
Finally, we have 
\begin{align*}
    \mathbb{E}\left[\left\|g^{k+1}-\nabla f\left(x^{k+1}\right)\right\|^2\right]   &\leq  \left(1-\frac{p}{4}\right) \mathbb{E}\left[\left\|g^{k}-\nabla f\left(x^{k}\right)\right\|^2\right]\\
    &+\widehat{B}\mathbb{E}\left[\left\|\nabla f\left(x^k\right)\right\|^2\right] + \widehat{D} \zeta^2+\frac{p A}{4} \EE\left[\|x^{k+1} - x^k\|^2\right],
\end{align*}
where 
\begin{align*}
    A &= \text{\small $\frac{32p_G }{p^2} \frac{\mathcal{P}_{\mathcal{G}^k_C}n}{C} \left( 10 \omega L^2+(10\omega+1) L_{ \pm}^2+\frac{10(\omega+1) \mathcal{L}_{ \pm}^2}{b} \right)$}\\
    &+\text{\small $\frac{8}{p^2}\frac{G\mathcal{P}_{\mathcal{G}^k_C} }{(1-\delta)C}p_G c\delta\left( 80(1+\omega)  \frac{\mathcal{L}_{ \pm}^2}{b} + 8(10\omega+1)L_{ \pm}^2  
       + 80\omega  L^2\right)$}\\
       &+\text{\small $\frac{4}{p}\cdot(p+2) \frac{8 G \mathcal{P}_{\mathcal{G}^k_{\widehat{C}}} c\delta}{(1-\delta)\widehat{C}} B  L^2 +\frac{16(1-p_G)}{p^2}(L^2+F_{\cA}^2\alpha^2_{\lambda_{k+1}})$},
\end{align*}
and
\begin{align*}
    \widehat{B} = 2 \frac{ \delta\mathcal{P}_{\mathcal{G}^k_{\widehat{C}}} }{1-\delta} B\left(\frac{12cG}{\widehat C} + p \right), \quad \widehat{D} = 2 \frac{ \delta\mathcal{P}_{\mathcal{G}^k_{\widehat{C}}} }{1-\delta}\left(\frac{6cG}{\widehat C} + p \right).
\end{align*}
Once we simplify the equation, we obtain
\begin{align*}
    A& = \frac{4}{p}\left( \frac{80}{p} \frac{p_G \mathcal{P}_{\mathcal{G}^k_C} n}{C} \omega + 24 \frac{ G\mathcal{P}_{\mathcal{G}^k_{\widehat{C}}} c\delta}{(1-\delta)\widehat C} B + \frac{4}{p}(1-p_G)+\frac{160}{p}p_G \frac{G\mathcal{P}_{\mathcal{G}^k_C} }{(1-\delta)C}c\delta\omega  \right) L^2\\
    &+\frac{4}{p}\left( \frac{8}{p} \frac{p_G \mathcal{P}_{\mathcal{G}^k_C} n}{C} \left( 10\omega + 1 \right) + \frac{16}{p}p_G \frac{G\mathcal{P}_{\mathcal{G}^k_C} }{(1-\delta)C}c\delta(10\omega+1) \right)L_{ \pm}^2\\
    &+\frac{4}{p}\left( \frac{160}{p} p_G \frac{G \mathcal{P}_{\mathcal{G}^k_C}}{(1-\delta) C} (1+\omega)c\delta +\frac{80}{p} p_G \mathcal{P}_{\mathcal{G}^k_C} (1+\omega) \frac{n}{C} \right)\frac{\mathcal{L}_{ \pm}^2}{b}\\
    &+\frac{4}{p}\left( \frac{4}{p}(1-p_G)F_{\cA}^2\alpha^2_{\lambda_{k+1}} \right).
\end{align*}
\end{proof}

\subsection{Main Results}

\begin{theorem}
\label{them:1}
 Let Assumptions \ref{assm:bounded-aggr}, \ref{assm:L-smoothness}, \ref{assm:global}, \ref{assm:local}, \ref{assm:het} hold. Setting $\lambda_{k+1} = 2\max_{i\in \mathcal{G}} L_i \left\|x^{k+1} - x^k\right\|$. Assume that
$$
0<\gamma \leq \frac{1}{L+\sqrt{A}}, \quad 4\widehat{B}<p,
$$
where 
\begin{align*}
    A& = \frac{4}{p}\left( \frac{80}{p} \frac{p_G \mathcal{P}_{\mathcal{G}^k_C} n}{C} \omega + 24 \frac{ G\mathcal{P}_{\mathcal{G}^k_{\widehat{C}}} c\delta}{(1-\delta)\widehat C} B + \frac{4}{p}(1-p_G)+\frac{160}{p}p_G \frac{G\mathcal{P}_{\mathcal{G}^k_C} }{(1-\delta)C}c\delta\omega  \right) L^2\\
    &+\frac{4}{p}\left( \frac{8}{p} \frac{p_G \mathcal{P}_{\mathcal{G}^k_C} n}{C} \left( 10\omega + 1 \right) + \frac{16}{p}p_G \frac{G\mathcal{P}_{\mathcal{G}^k_C} }{(1-\delta)C}c\delta(10\omega+1) \right)L_{ \pm}^2\\
    &+\frac{4}{p}\left( \frac{160}{p} p_G \frac{G \mathcal{P}_{\mathcal{G}^k_C}}{(1-\delta) C} (1+\omega)c\delta +\frac{80}{p} p_G \mathcal{P}_{\mathcal{G}^k_C} (1+\omega) \frac{n}{C} \right)\frac{\mathcal{L}_{ \pm}^2}{b}\\
    &+\frac{4}{p}\left( \frac{4}{p}(1-p_G)F_{\cA}^2\alpha^2_{\lambda_{k+1}} \right),
\end{align*}
\textbf{\begin{align*}
    \widehat{B} = 2 \frac{ \delta\mathcal{P}_{\mathcal{G}^k_{\widehat{C}}} }{1-\delta} B\left(\frac{12cG}{\widehat C} + p \right), \quad \widehat{D} = 2 \frac{ \delta\mathcal{P}_{\mathcal{G}^k_{\widehat{C}}} }{1-\delta} \left(\frac{6cG}{\widehat C} + p \right)
\end{align*}}
and 
\begin{align*}
 \mathcal{P}_{\mathcal{G}^k_C} &=   \frac{C}{np_G} \cdot \sum_{(1-\delta)C\leq t \leq C} \left(\left(\begin{array}{l}
G-1 \\
t-1
\end{array}\right)\left(\begin{array}{l}
n-G \\
C-t
\end{array}\right) \left(\left(\begin{array}{l}
n \\
C
\end{array}\right)\right)^{-1} \right),\\
p_G &=  \PP\left\lbrace G_C^k \geq\left(1-\delta\right) C\right\rbrace\\
&= \sum_{\lceil(1-\delta)C\rceil\leq t \leq C} \left(\left(\begin{array}{l}
G \\
t
\end{array}\right)\left(\begin{array}{l}
n-G \\
C-t
\end{array}\right) \left(\begin{array}{l}
n-1 \\
C-1
\end{array}\right)^{-1} \right).
\end{align*}
Then for all $K \geq 0$ the iterates produced by \algname{Byz-VR-MARINA} (Algorithm \ref{alg:byz_vr_marina}) satisfy
$$
\mathbb{E}\left[\left\|\nabla f\left(\widehat{x}^K\right)\right\|^2\right] \leq \frac{2 \Phi^0}{\gamma\left(1-\frac{4 \widehat{B}}{p}\right)(K+1)}+\frac{4 \widehat{D}  \zeta^2}{p-4 \widehat{B} },
$$
where $\widehat{x}^K$ is choosen uniformly at random from $x^0, x^1, \ldots, x^K$, and $\Phi^0=$ $f\left(x^0\right)-f^*+\frac{2\gamma}{p}\left\|g^0-\nabla f\left(x^0\right)\right\|^2$ . 
\end{theorem}
\begin{proof}[Proof of Theorem~\ref{them:1}]
     For all $k \geq 0$ we introduce $\Phi^k=f\left(x^k\right)-f^*+\frac{2\gamma}{p}\left\|g^k-\nabla f\left(x^k\right)\right\|^2$. Using the results of Lemmas \ref{lemma:final_lemma} and \ref{lemma:page}, we derive
     \begin{align*}
         \mathbb{E}\left[\Phi^{k+1}\right] &\stackrel{(\ref{lemma:page})}{\leq} \mathbb{E}\left[f\left(x^k\right)-f^*-\left(\frac{1}{2 \gamma}-\frac{L}{2}\right)\left\|x^{k+1}-x^k\right\|^2+\frac{\gamma}{2}\left\|g^k-\nabla f\left(x^k\right)\right\|^2\right]\\
         &-\frac{\gamma}{2} \mathbb{E}\left[\left\|\nabla f\left(x^k\right)\right\|^2\right] + \frac{2\gamma}{p}   \mathbb{E}\left[\left\|g^{k+1}-\nabla f\left(x^{k+1}\right)\right\|^2\right]\\
         &\stackrel{(\ref{lemma:final_lemma})}{\leq}  \mathbb{E}\left[f\left(x^k\right)-f^*-\left(\frac{1}{2 \gamma}-\frac{L}{2}\right)\left\|x^{k+1}-x^k\right\|^2+\frac{\gamma}{2}\left\|g^k-\nabla f\left(x^k\right)\right\|^2\right]\\
         &-\frac{\gamma}{2} \mathbb{E}\left[\left\|\nabla f\left(x^k\right)\right\|^2\right] + \frac{2\gamma}{p}     \left(1-\frac{p}{4}\right) \mathbb{E}\left[\left\|g^{k}-\nabla f\left(x^{k}\right)\right\|^2\right]\\
         &+ \frac{2\gamma}{p}\left( \widehat{B}\mathbb{E}\left[\left\|\nabla f\left(x^k\right)\right\|^2\right] + \widehat{D} \zeta^2+\frac{p A}{4}\|x^{k+1} - x^k\|^2\right)\\
         &=\mathbb{E}\left[f\left(x^k\right)-f^*\right] + \frac{2\gamma}{p}\left(\left(1-\frac{p}{4}\right) + \frac{p}{4}\right) \mathbb{E}\left[\left\|g^{k}-\nabla f\left(x^{k}\right)\right\|^2\right]+\frac{2\widehat{D}\zeta^2\gamma}{p}\\
         &+\frac{1}{2\gamma}\left(1-L\gamma-A\gamma^2\right)\mathbb{E}\left[ \|x^{k+1} - x^k\|^2 \right] - \frac{\gamma}{2}\left(1-\frac{4\widehat{B}}{p}\right)\mathbb{E}\left[\left\|\nabla f\left(x^k\right)\right\|^2\right]\\
         &=\mathbb{E}\left[\Phi^{k}\right] +\frac{2 \widehat{D}\zeta^2\gamma}{p}+\frac{1}{2\gamma}\left(1-L\gamma-A\gamma^2\right)\mathbb{E}\left[ \|x^{k+1} - x^k\|^2 \right]\\
         &- \frac{\gamma}{2}\left(1-\frac{4\widehat{B}}{p}\right)\mathbb{E}\left[\left\|\nabla f\left(x^k\right)\right\|^2\right].
     \end{align*}
     Using choice of stepsize and second condition: $
0<\gamma \leq \frac{1}{L+\sqrt{A}}, 4\widehat{B}<p
$ and lemma \ref{lemma:peter} we have
\begin{align*}
       \mathbb{E}\left[\Phi^{k+1}\right]     &\leq\mathbb{E}\left[\Phi^{k}\right] +\frac{2 \widehat{D}\zeta^2\gamma}{p}- \frac{\gamma}{2}\left(1-\frac{4\widehat{B}}{p}\right)\mathbb{E}\left[\left\|\nabla f\left(x^k\right)\right\|^2\right]
     \end{align*}
      Next, we have $\frac{\gamma}{2}\left(1-\frac{4 \widehat{B} }{p}\right)>0$ and $\Phi^{k+1}\geq0$. Therefore, summing up the above inequality for $k=0,1, \ldots, K$ and rearranging the terms, we get
\begin{align*}
\frac{1}{K+1} \sum_{k=0}^K \mathbb{E}\left[\left\|\nabla f\left(x^k\right)\right\|^2\right] \leq & \frac{2}{\gamma\left(1-\frac{4 \widehat{B}}{p}\right)(K+1)} \sum_{k=0}^K\left(\mathbb{E}\left[\Phi^k\right]-\mathbb{E}\left[\Phi^{k+1}\right]\right) \\
& +\frac{4 \widehat{D}  \zeta^2}{p-4\widehat{B} } \\
= & \frac{2\left(\mathbb{E}\left[\Phi^0\right]-\mathbb{E}\left[\Phi^{k+1}\right]\right)}{\gamma\left(1-\frac{4 \widehat{B} }{p}\right)(K+1)}+\frac{4\widehat{D}  \zeta^2}{p-4 \widehat{B} } \\
\leq& \frac{2 \mathbb{E}\left[\Phi^0\right]}{\gamma\left(1-\frac{4\widehat{B} }{p}\right)(K+1)}+\frac{4 \widehat{D}  \zeta^2}{p-4\widehat{B} }.
\end{align*}

\end{proof}

\begin{theorem}
\label{them:2}
Let Assumptions  \ref{assm:bounded-aggr}, \ref{assm:L-smoothness}, \ref{assm:global}, \ref{assm:local}, \ref{assm:het}, \ref{assm:PL} hold. Set $\lambda_{k+1} = \max_{i\in \mathcal{G}} L_i \left\|x^{k+1} - x^k\right\|$. Assume that
$$
0<\gamma \leq \min \left\{\frac{1}{L+\sqrt{2 A}} \right\}, \quad 8\widehat{B}<p
$$
where 
\begin{align*}
    A& = \frac{4}{p}\left( \frac{80}{p} \frac{p_G \mathcal{P}_{\mathcal{G}^k_C} n}{C} \omega + 24 \frac{ G\mathcal{P}_{\mathcal{G}^k_{\widehat{C}}} c\delta}{(1-\delta)\widehat C} B + \frac{4}{p}(1-p_G)+\frac{160}{p}p_G \frac{G\mathcal{P}_{\mathcal{G}^k_C} }{(1-\delta)C}c\delta\omega  \right) L^2\\
    &+\frac{4}{p}\left( \frac{8}{p} \frac{p_G \mathcal{P}_{\mathcal{G}^k_C} n}{C} \left( 10\omega + 1 \right) + \frac{16}{p}p_G \frac{G\mathcal{P}_{\mathcal{G}^k_C} }{(1-\delta)C}c\delta(10\omega+1) \right)L_{ \pm}^2\\
    &+\frac{4}{p}\left( \frac{160}{p} p_G \frac{G \mathcal{P}_{\mathcal{G}^k_C}}{(1-\delta) C} (1+\omega)c\delta +\frac{80}{p} p_G \mathcal{P}_{\mathcal{G}^k_C} (1+\omega) \frac{n}{C} \right)\frac{\mathcal{L}_{ \pm}^2}{b}\\
    &+\frac{4}{p}\left( \frac{4}{p}(1-p_G)F_{\cA}^2\alpha^2_{\lambda_{k+1}} \right),
\end{align*}
\textbf{\begin{align*}
    \widehat{B} = 2 \frac{ \delta\mathcal{P}_{\mathcal{G}^k_{\widehat{C}}} }{1-\delta} B\left(\frac{12cG}{\widehat C} + p \right), \quad \widehat{D} = 2 \frac{ \delta\mathcal{P}_{\mathcal{G}^k_{\widehat{C}}} }{1-\delta} \left(\frac{6cG}{\widehat C} + p \right)
\end{align*}}
and 
\begin{align*}
 \mathcal{P}_{\mathcal{G}^k_C} &=   \frac{C}{np_G} \cdot \sum_{(1-\delta)C\leq t \leq C} \left(\left(\begin{array}{l}
G-1 \\
t-1
\end{array}\right)\left(\begin{array}{l}
n-G \\
C-t
\end{array}\right) \left(\left(\begin{array}{l}
n \\
C
\end{array}\right)\right)^{-1} \right),\\
p_G &=  \PP\left\lbrace G_C^k \geq\left(1-\delta\right) C\right\rbrace\\
&= \sum_{\lceil(1-\delta)C\rceil\leq t \leq C} \left(\left(\begin{array}{l}
G \\
t
\end{array}\right)\left(\begin{array}{l}
n-G \\
C-t
\end{array}\right) \left(\begin{array}{l}
n-1 \\
C-1
\end{array}\right)^{-1} \right).
\end{align*}
Then for all $K \geq 0$ the iterates produced by \algname{Byz-VR-MARINA} (Algorithm \ref{alg:byz_vr_marina}) satisfy
$$
\mathbb{E}\left[f\left(x^K\right)-f\left(x^*\right)\right] \leq\left(1-\rho\right)^K \Phi^0+\frac{4 \widehat{D} \gamma \zeta^2}{p\rho},
$$
where $\rho = \min\left[\gamma\mu\left(1-\frac{8\widehat{B}}{p}\right), \frac{p}{8}\right]$ and $\Phi^0=$ $f\left(x^0\right)-f^*+\frac{4\gamma}{p}\left\|g^0-\nabla f\left(x^0\right)\right\|^2$.
\end{theorem}
\begin{proof}
     For all $k \geq 0$ we introduce $\Phi^k=f\left(x^k\right)-f^*+\frac{4\gamma}{p}\left\|g^k-\nabla f\left(x^k\right)\right\|^2$. Using the results of Lemmas \ref{lemma:final_lemma} and \ref{lemma:page}, we derive
     \begin{align*}
         \mathbb{E}\left[\Phi^{k+1}\right] &\stackrel{(\ref{lemma:page})}{\leq} \mathbb{E}\left[f\left(x^k\right)-f^*-\left(\frac{1}{2 \gamma}-\frac{L}{2}\right)\left\|x^{k+1}-x^k\right\|^2+\frac{\gamma}{2}\left\|g^k-\nabla f\left(x^k\right)\right\|^2\right]\\
         &-\frac{\gamma}{2} \mathbb{E}\left[\left\|\nabla f\left(x^k\right)\right\|^2\right] + \frac{4\gamma}{p}   \mathbb{E}\left[\left\|g^{k+1}-\nabla f\left(x^{k+1}\right)\right\|^2\right]\\
         &\stackrel{(\ref{lemma:final_lemma})}{\leq}  \mathbb{E}\left[f\left(x^k\right)-f^*-\left(\frac{1}{2 \gamma}-\frac{L}{2}\right)\left\|x^{k+1}-x^k\right\|^2+\frac{\gamma}{2}\left\|g^k-\nabla f\left(x^k\right)\right\|^2\right]\\
         &-\frac{\gamma}{2} \mathbb{E}\left[\left\|\nabla f\left(x^k\right)\right\|^2\right] + \frac{4\gamma}{p}     \left(1-\frac{p}{4}\right) \mathbb{E}\left[\left\|g^{k}-\nabla f\left(x^{k}\right)\right\|^2\right]\\
         &+ \frac{4\gamma}{p}\left( \widehat{B}\mathbb{E}\left[\left\|\nabla f\left(x^k\right)\right\|^2\right] + \widehat{D} \zeta^2+\frac{p A}{4}\|x^{k+1} - x^k\|^2\right)\\
         &=\mathbb{E}\left[f\left(x^k\right)-f^*\right] + \frac{4\gamma}{p}\left(\left(1-\frac{p}{4}\right) + \frac{p}{8}\right) \mathbb{E}\left[\left\|g^{k}-\nabla f\left(x^{k}\right)\right\|^2\right]+\frac{4 \widehat{D}\zeta^2\gamma}{p}\\
         &+\frac{1}{2\gamma}\left(1-L\gamma-2A\gamma^2\right)\mathbb{E}\left[ \|x^{k+1} - x^k\|^2 \right] - \frac{\gamma}{2}\left(1-\frac{8\widehat{B}}{p}\right)\mathbb{E}\left[\left\|\nabla f\left(x^k\right)\right\|^2\right].
\end{align*}
Using Assumption~\ref{assm:PL} we obtain
\begin{align*}
\mathbb{E}\left[\Phi^{k+1}\right] &\leq\mathbb{E}\left[f\left(x^k\right)-f^*\right]+\left(1-\frac{p}{8}\right) \frac{4\gamma}{p} \mathbb{E}\left[\left\|g^{k}-\nabla f\left(x^{k}\right)\right\|^2\right]+\frac{4 \widehat{D}\zeta^2\gamma}{p}\\
         &+\frac{1}{2\gamma}\left(1-L\gamma-2A\gamma^2\right)\mathbb{E}\left[ \|x^{k+1} - x^k\|^2 \right]\\ &- \gamma\mu\left(1-\frac{8\widehat{B}}{p}\right)\mathbb{E}\left[f\left(x^k\right)-f^*\right].
\end{align*}
Finally, we have 
\begin{align*}
\mathbb{E}\left[\Phi^{k+1}\right] &\leq\left(1 - \min\left[\gamma\mu\left(1-\frac{\widehat{B}}{p}\right), \frac{p}{8}\right]\right)\mathbb{E}\left[\Phi^k\right]+\frac{4 \widehat{D}\zeta^2\gamma}{p}.
\end{align*}
Unrolling the recurrence with $\rho = \min\left[\gamma\mu\left(1-\frac{8\widehat{B}}{p}\right), \frac{p}{8}\right] $, we obtain
$$
\begin{aligned}
\mathbb{E}\left[\Phi^k\right] & \leq\left(1-\rho\right)^K \mathbb{E}\left[\Phi^0\right]+\frac{4 \widehat{D} \zeta^2 \gamma}{p} \sum_{k=0}^{K-1}\left(1-\rho\right)^k \\
& \leq\left(1-\rho\right)^K \mathbb{E}\left[\Phi^0\right]+\frac{4 \widehat{D} \zeta^2 \gamma}{p} \sum_{k=0}^{\infty}\left(1-\rho\right)^k \\
& =\left(1-\rho\right)^K \mathbb{E}\left[\Phi^0\right]+\frac{4 \widehat{D} \gamma\zeta^2}{p\rho}
\end{aligned}
$$
Taking into account $\Phi^k \geq f\left(x^k\right)-f\left(x^*\right)$, we get the result.
\end{proof}

\clearpage
\section{Analysis for Bounded Compressors}

\subsection{Technical Lemmas}

\begin{lemma}
\label{lemma:premainA1_Q}
Let Assumptions~\ref{assm:L-smoothness}, \ref{assm:global}, \ref{assm:local} and \ref{assm:bounded-compressor} hold and the Compression Operator satisfy Definition~\ref{def:Q}. We set $\lambda_{k+1} = D_Q \max_{i,j} L_{i,j}$.  Let us define "ideal" estimator:
\begin{equation*}
    \overline{g}^{k+1}= \begin{cases}
        \frac{1}{G^k_{\widehat{C}}}\sum \limits_{i \in \mathcal{G}^k_{\widehat{C}}} \nabla f_i(x^{k+1}),& c_n=1,\quad\quad\quad\quad\quad\quad\quad\quad\quad~\hspace{0.01cm} [1]\\
        g^k+\nabla f\left(x^{k+1}\right)-\nabla f\left(x^k\right),& c_n=0 \text{ and } G^k_C < (1-\delta)C, \hspace{0.325cm}[2]\\
        g^k+\frac{1}{G^k_C} \sum \limits_{i \in \mathcal{G}^k_C}\clip_{\lambda}\left(\mathcal{Q}\left(\widehat{\Delta}_i\left(x^{k+1}, x^k\right)\right)\right),& c_n=0 \text{ and } G^k_C \geq (1-\delta)C. \hspace{0.375cm}   [3]
    \end{cases}
\end{equation*}
Then for all $k\geq 0$ the iterates produced by \algname{Byz-VR-MARINA-PP} (Algorithm~\ref{alg:byz_vr_marina}) satisfy
\begin{align*}
   A_1  &= \mathbb{E}\left[\left\|\overline{g}^{k+1}-\nabla f\left(x^{k+1}\right)\right\|^2\right]\\
   &\leq (1-p) \mathbb{E}\left[\left\|g^{k}-\nabla f(x^{k})\right\|^2\right] + p\frac{ \delta\mathcal{P }_{\mathcal{G}^k_{\widehat{C}}} }{(1-\delta)}  \mathbb{E}\left[ B\|\nabla f(x)\|^2+\zeta^2 \right]\\
    &+ \text{\normalsize $(1-p)p_G\frac{\mathcal{P}_{\mathcal{G}^k_C} G}{C^2(1-\delta)^2} \left(  \omega L^2+(\omega+1) L_{ \pm}^2+\frac{(\omega+1) \mathcal{L}_{ \pm}^2}{b} \right)\mathbb{E}\left[\|x^{k+1} - x^k\|^2\right]$}.
    \end{align*}
where $p_G = \operatorname{Prob}\left\lbrace G^k_C \geq (1-\delta)C \right\rbrace $ and $\mathcal{P}_{\mathcal{G}^k_C} =  \operatorname{Prob}\left\lbrace i \in \mathcal{G}^k_C \mid G^k_C \geq\left(1-\delta\right) C\right\rbrace$.
\end{lemma}

\begin{proof}
Similarly to general analysis, we start from conditional expectations:
\begin{align}
    A_1 &= \mathbb{E}\left[\left\|\overline{g}^{k+1}-\nabla f\left(x^{k+1}\right)\right\|^2\right] \notag\\
&=\mathbb{E}\left[\mathbb{E}_k\left[\left\|\overline{g}^{k+1}-\nabla f\left(x^{k+1}\right)\right\|^2\right]\right] \notag\\
    & = \left(1-p\right)p_{G}\mathbb{E}\left[\mathbb{E}_k\left[\left\|g^k+\frac{1}{G^k_C} \sum \limits_{i \in \mathcal{G}^k_C}\clip_{\lambda}\left(\mathcal{Q}\left(\widehat{\Delta}_i\left(x^{k+1}, x^k\right)\right)\right) -\nabla f\left(x^{k+1}\right)\right\|^2\right]\mid [3]\right] \notag\\
    & + (1-p)(1-p_{G})\mathbb{E}\left[\mathbb{E}_k\left[\left\|g^k -\nabla f(x^{k})\right\|^2\right]\mid [2]\right] \notag\\
    &+ p\mathbb{E}\left[ \left\|\frac{1}{G^k_{\widehat{C}}}\sum \limits_{i \in \cG_{\widehat{C}}^k} \nabla f_i(x^{k+1}) - \nabla f(x^{k+1})\right\|^2 \right]. \label{eq:ncsjindisbciusdbhc}
\end{align}
Using (\ref{eq:yung-1}) and $\nabla f\left(x^{k}\right) - \nabla f\left(x^{k}\right) = 0$ we obtain
\begin{align*}
B_1 &= \mathbb{E}\left[\mathbb{E}_k\left[\left\|g^k+\frac{1}{G^k_C} \sum \limits_{i \in \mathcal{G}^k_C}\clip_{\lambda}\left(\mathcal{Q}\left(\widehat{\Delta}_i\left(x^{k+1}, x^k\right)\right)\right) -\nabla f\left(x^{k+1}\right)\right\|^2\right]\mid [3]\right]\\
&\text{\scriptsize$ =\mathbb{E}\left[\mathbb{E}_k\left[\left\|g^k+\frac{1}{G^k_C} \sum \limits_{i \in \mathcal{G}^k_C}\clip_{\lambda}\left(\mathcal{Q}\left(\widehat{\Delta}_i\left(x^{k+1}, x^k\right)\right)\right) -\nabla f\left(x^{k+1}\right)+\nabla f\left(x^{k}\right) - \nabla f\left(x^{k}\right)\right\|^2\right]\mid [3]\right]$}
\end{align*}
Using $\lambda_{k+1} = D_Q\max_{i,j} L_{i,j} \|x^{k+1} - x^k\|$ we can guarantee that clipping operator becomes identical since we have 
\begin{align*}
   \left \|\mathcal{Q}\left(\widehat{\Delta}_i\left(x^{k+1}, x^k\right)\right)\right\| &\leq D_Q\left\|\widehat{\Delta}_i\left(x^{k+1}, x^k\right)\right\|\\
   &\leq D_Q \left\|\frac{1}{b} \sum_{j\in m} \nabla f_{i,j}(x^{k+1})-\nabla f_{i,j}(x^k) \right\|\\
    &\leq D_Q \frac{1}{b}\sum_{j\in m}\left\|  \nabla f_{i,j}(x^{k+1})-\nabla f_{i,j}(x^k) \right\|\\
    &\leq D_Q \max_j L_{i,j}\left\|  x^{k+1}-x^k \right\|\\
    &\leq D_Q \max_{i,j} L_{i,j}\left\|  x^{k+1}-x^k \right\|.
\end{align*}
Therefore, we can continue as follows
\begin{align*}
B_1 &= \mathbb{E}\left[\mathbb{E}_k\left[\left\|g^k+\frac{1}{G^k_C} \sum \limits_{i \in \mathcal{G}^k_C}\mathcal{Q}\left(\widehat{\Delta}_i\left(x^{k+1}, x^k\right)\right) -\nabla f\left(x^{k+1}\right)\right\|^2\right]\mid [3]\right]\\
&\text{\scriptsize$ =\mathbb{E}\left[\mathbb{E}_k\left[\left\|g^k+\frac{1}{G^k_C} \sum \limits_{i \in \mathcal{G}^k_C}\mathcal{Q}\left(\widehat{\Delta}_i\left(x^{k+1}, x^k\right)\right)-\nabla f\left(x^{k+1}\right)+\nabla f\left(x^{k}\right) - \nabla f\left(x^{k}\right)\right\|^2\right]\mid [3]\right]$}.
\end{align*}
Moreover, we can avoid application of Young's inequality and use variance decomposition instead:
 \begin{align}  
     B_1 &\leq \mathbb{E}\left[\left\|g^{k} -\nabla f\left(x^{k}\right)\right\|^2\right] \notag \\
    &+ \text{\scriptsize$ \mathbb{E}\left[\mathbb{E}_k\left[\left\|\frac{1}{G^k_C} \sum \limits_{i \in \mathcal{G}^k_C}\mathcal{Q}\left(\widehat{\Delta}_i\left(x^{k+1}, x^k\right)\right)  - \left( \nabla f(x^{k+1}) - \nabla f(x^{k}) \right)\right\|^2\right]\mid [3]\right]$} \notag\\
    &\leq \mathbb{E}\left[\left\|g^{k}-\nabla f(x^{k})\right\|^2\right] \notag \\
    &+ \mathbb{E}\left[ \mathbb{E}_k\left[\left\|\frac{1}{G^k_C} \sum \limits_{i \in \mathcal{G}^k_C}\mathcal{Q}\left(\widehat{\Delta}_i\left(x^{k+1}, x^k\right)\right) - \Delta\left(x^{k+1}, x^k\right)\right\|^2\right]\mid [3]\right]. \label{eq:djbfjnskcnksbfjvbsd}
\end{align}
Let us consider the last part of the inequality. Note that $G^k_C \geq (1-\delta)C$ in this case and
\begin{align}\label{eq:indicator_1}
\notag  B^\prime_1 &= \mathbb{E}\left[ \mathbb{E}_k\left[\left\|\frac{1}{G^k_C} \sum \limits_{i \in \mathcal{G}^k_C}\mathcal{Q}\left(\widehat{\Delta}_i\left(x^{k+1}, x^k\right)\right)  - \Delta\left(x^{k+1}, x^k\right)\right\|^2\right]\mid [3] \right]\\
 \notag &= \mathbb{E}\left[\mathbb{E}_{S_k}\left[\mathbb{E}_k\left[\left\|\frac{1}{G^k_C} \sum \limits_{i \in \mathcal{G}^k_C}\mathcal{Q}\left(\widehat{\Delta}_i\left(x^{k+1}, x^k\right)\right) - \Delta\left(x^{k+1}, x^k\right)\right\|^2\right]\mid [3]\right]\right]\\
\notag   &\leq \text{\small $\frac{1}{C^2(1-\delta)^2}\mathbb{E}\left[\mathbb{E}_{S_k}\left[ \sum \limits_{i \in \mathcal{G}^k_C}\mathbb{E}_k\left[\left\|\mathcal{Q}\left(\widehat{\Delta}_i\left(x^{k+1}, x^k\right)\right) - \Delta\left(x^{k+1}, x^k\right)\right\|^2\right]\mid [3]\right]\right]$}\\
 \notag   &\leq \text{\small $\frac{1}{C^2(1-\delta)^2}\mathbb{E}\left[\sum \limits_{i \in \mathcal{G}}\mathbb{E}_{S_k}\left[\mathcal{I}_{\mathcal{G}^k_C}\right]\mathbb{E}_k\left[\left\|\mathcal{Q}\left(\widehat{\Delta}_i\left(x^{k+1}, x^k\right)\right)  - \Delta\left(x^{k+1}, x^k\right)\right\|^2\right]\mid [3] \right]$}\\
        &\text{\small$= \frac{1}{C^2(1-\delta)^2}\mathbb{E}\left[\sum \limits_{i \in \mathcal{G}}\mathcal{P}_{\mathcal{G}^k_C}\cdot\mathbb{E}_k\left[\left\|\mathcal{Q}\left(\widehat{\Delta}_i\left(x^{k+1}, x^k\right)\right) - \Delta\left(x^{k+1}, x^k\right)\right\|^2\right]\mid [3]\right]$},
\end{align}
where $\mathcal{I}_{\mathcal{G}^k_C}$ is an indicator function for the event $\left\lbrace i \in \mathcal{G}^k_C \mid G_C^k \geq\left(1-\delta\right) C\right\rbrace$ and $\mathcal{P}_{\mathcal{G}^k_C} =  \operatorname{Prob}\left\lbrace i \in \mathcal{G}^k_C \mid G_C^k \geq\left(1-\delta\right) C\right\rbrace$ is probability of such event. Note that 
$\mathbb{E}_{S_k}\left[\mathcal{I}_{\mathcal{G}^k_C}\right] = \mathcal{P}_{\mathcal{G}^k_C}$. In the case of uniform sampling of clients, we have 
\begin{align*}
    \forall i \in \mathcal{G} \quad \mathcal{P}_{\mathcal{G}^k_C} &=\operatorname{Prob}\left\lbrace i \in \mathcal{G}^k_C \mid G^k_C\geq\left(1-\delta\right) C\right\rbrace\\
    & = \frac{C}{n} \frac{1}{p_G}\cdot \sum_{(1-\delta)C\leq t \leq C} \left(\left(\begin{array}{l}
G-1 \\
t-1
\end{array}\right)\left(\begin{array}{l}
n-G \\
C-t
\end{array}\right) \left(\left(\begin{array}{l}
n-1 \\
C-1
\end{array}\right)\right)^{-1} \right). 
\end{align*}
Now, we can continue with inequalities:
\begin{align*}
B_1^\prime &\leq  \text{\small $       \frac{\mathcal{P}_{\mathcal{G}^k_C}}{C^2(1-\delta)^2} \mathbb{E}\left[ \sum \limits_{i \in \mathcal{G}}\mathbb{E}_k\left[\left\|\mathcal{Q}\left(\widehat{\Delta}_i\left(x^{k+1}, x^k\right)\right)  - \Delta\left(x^{k+1}, x^k\right)\right\|^2\right]\mid [3]\right]$}\\
 &\leq     \text{\small $      \frac{\mathcal{P}_{\mathcal{G}^k_C}}{C^2(1-\delta)^2}\mathbb{E}\left[\sum \limits_{i \in \mathcal{G}}\mathbb{E}_k\left[\mathbb{E}_{Q}\left[\left\|\mathcal{Q}\left(\widehat{\Delta}_i\left(x^{k+1}, x^k\right)\right)  - \Delta\left(x^{k+1}, x^k\right)\right\|^2\right] \right]\mid [3]  \right]$}\\
 &\leq     \text{\small $      \frac{\mathcal{P}_{\mathcal{G}^k_C}}{C^2(1-\delta)^2}\mathbb{E}\left[\sum \limits_{i \in \mathcal{G}}\mathbb{E}_k\left[\mathbb{E}_{Q}\left[\left\|\mathcal{Q}\left(\widehat{\Delta}_i\left(x^{k+1}, x^k\right)\right)  - \Delta_i\left(x^{k+1}, x^k\right)\right\|^2\right]\right]\mid [3]\right]$}\\ 
& +   \frac{\mathcal{P}_{\mathcal{G}^k_C}}{C^2(1-\delta)^2}\mathbb{E}\left[\sum \limits_{i \in \mathcal{G}}\mathbb{E}_k\left[\left\|  \Delta_i\left(x^{k+1}, x^k\right) - \Delta\left(x^{k+1}, x^k\right)\right\|^2\right]\mid [3]\right].
\end{align*}
Using variance decomposition, we have 
\begin{align*}
B_1^\prime &\leq \text{\small $\frac{\mathcal{P}_{\mathcal{G}^k_C}}{C^2(1-\delta)^2} \mathbb{E}\left[ \sum \limits_{i \in \mathcal{G}}\mathbb{E}_k\left[\mathbb{E}_{Q}\left[\left\|\mathcal{Q}\left(\widehat{\Delta}_i\left(x^{k+1}, x^k\right)\right)\right\|^2\right]\right] - \sum \limits_{i \in \mathcal{G}}\left\|\Delta_i\left(x^{k+1}, x^k\right)\right\|^2\mid [3]\right]$}\\
&+\frac{\mathcal{P}_{\mathcal{G}^k_C}}{C^2(1-\delta)^2}\mathbb{E}\left[\sum \limits_{i \in \mathcal{G}}\mathbb{E}_k\left[\left\|  \Delta_i\left(x^{k+1}, x^k\right) -  \Delta\left(x^{k+1}, x^k\right)\right\|^2\right]\mid [3]\right].
\end{align*}

Applying the definition of unbiased compressor, we get
\begin{align*}
B_1^\prime&\leq  \frac{\mathcal{P}_{\mathcal{G}^k_C}}{C^2(1-\delta)^2} \mathbb{E}\left[ \sum \limits_{i \in \mathcal{G}}(1+\omega)\mathbb{E}_k\left\|\widehat{\Delta}_i\left(x^{k+1}, x^k\right)\right\|^2 - \sum \limits_{i \in \mathcal{G}}\left\|\Delta_i\left(x^{k+1}, x^k\right)\right\|^2\mid [3]\right]\\
&+\frac{\mathcal{P}_{\mathcal{G}^k_C}}{C^2(1-\delta)^2}\mathbb{E}\left[\sum \limits_{i \in \mathcal{G}}\left\|  \Delta_i\left(x^{k+1}, x^k\right) -  \Delta\left(x^{k+1}, x^k\right)\right\|^2\mid [3]\right]\\
&\leq  \frac{\mathcal{P}_{\mathcal{G}^k_C}}{C^2(1-\delta)^2} \mathbb{E} \left[ \sum \limits_{i \in \mathcal{G}}(1+\omega)\mathbb{E}_k\left\|\widehat{\Delta}_i\left(x^{k+1}, x^k\right) - \Delta_i\left(x^{k+1}, x^k\right) \right\|^2 \right]\\
&+ \frac{\mathcal{P}_{\mathcal{G}^k_C}}{C^2(1-\delta)^2} \mathbb{E} \left[ \sum \limits_{i \in \mathcal{G}}(1+\omega)\mathbb{E}_k\left\|\Delta_i\left(x^{k+1}, x^k\right)\right\|^2- \sum \limits_{i \in \mathcal{G}}\mathbb{E}_k\left\|\Delta_i\left(x^{k+1}, x^k\right)\right\|^2\mid [3]\right]\\
&+\frac{\mathcal{P}_{\mathcal{G}^k_C}}{C^2(1-\delta)^2}\mathbb{E}\left[\sum \limits_{i \in \mathcal{G}}\left\|  \Delta_i\left(x^{k+1}, x^k\right) -  \Delta\left(x^{k+1}, x^k\right)\right\|^2\mid [3]\right].
\end{align*}
Next, we rearrange terms and derive 
\begin{align*}
    B_1^\prime &\leq  \frac{\mathcal{P}_{\mathcal{G}^k_C}}{C^2(1-\delta)^2} (1+\omega) \mathbb{E} \left[ \sum \limits_{i \in \mathcal{G}}\mathbb{E}_k\left[\left\|\widehat{\Delta}_i\left(x^{k+1}, x^k\right) - \Delta_i\left(x^{k+1}, x^k\right) \right\|^2\right] \mid [3] \right]\\
&+ \frac{\mathcal{P}_{\mathcal{G}^k_C}}{C^2(1-\delta)^2}\omega \mathbb{E} \left[ \sum \limits_{i \in \mathcal{G}}\left\|\Delta_i\left(x^{k+1}, x^k\right)\right\|^2\mid [3]\right]\\
&+\frac{\mathcal{P}_{\mathcal{G}^k_C}}{C^2(1-\delta)^2}\mathbb{E} \left[ \sum \limits_{i \in \mathcal{G}}\left\|  \Delta_i\left(x^{k+1}, x^k\right) -  \Delta\left(x^{k+1}, x^k\right)\right\|^2\mid [3]\right]\\
 &= \frac{\mathcal{P}_{\mathcal{G}^k_C}}{C^2(1-\delta)^2} (1+\omega) \mathbb{E} \left[ \sum \limits_{i \in \mathcal{G}}\mathbb{E}_k\left[\left\|\widehat{\Delta}_i\left(x^{k+1}, x^k\right) - \Delta_i\left(x^{k+1}, x^k\right) \right\|^2\right] \mid [3] \right] \\
&+ \frac{\mathcal{P}_{\mathcal{G}^k_C}}{C^2(1-\delta)^2}\omega \mathbb{E} \left[  \sum \limits_{i \in \mathcal{G}}\left\|\Delta_i\left(x^{k+1}, x^k\right) - \Delta\left(x^{k+1}, x^k\right)\right\|^2 + \|\Delta\left(x^{k+1}, x^k\right)\|^2\mid [3]\right]\\
&+\frac{\mathcal{P}_{\mathcal{G}^k_C}}{C^2(1-\delta)^2} \mathbb{E} \left[ \sum \limits_{i \in \mathcal{G}}\left\|  \Delta_i\left(x^{k+1}, x^k\right) -  \Delta\left(x^{k+1}, x^k\right)\right\|^2 \mid [3] \right].
\end{align*}
Rearranging terms leads to 
\begin{align*}
   B_1^\prime  &\leq \frac{\mathcal{P}_{\mathcal{G}^k_C}}{C^2(1-\delta)^2} (1+\omega) \mathbb{E} \left[  \sum \limits_{i \in \mathcal{G}}\mathbb{E}_k\left[\left\|\widehat{\Delta}_i\left(x^{k+1}, x^k\right) - \Delta_i\left(x^{k+1}, x^k\right) \right\|^2\right] \mid [3] \right] \\
&+ \frac{\mathcal{P}_{\mathcal{G}^k_C}}{C^2(1-\delta)^2}(\omega+1) \mathbb{E}\left[ \sum \limits_{i \in \mathcal{G}}\left\|\Delta_i\left(x^{k+1}, x^k\right) - \Delta\left(x^{k+1}, x^k\right)\right\|^2\mid[3]\right] \\
&+\frac{\mathcal{P}_{\mathcal{G}^k_C}}{C^2(1-\delta)^2} \omega \mathbb{E}\left[\sum \limits_{i \in \mathcal{G}}\left\|  \Delta\left(x^{k+1}, x^k\right)\right\|^2 \mid [3]\right].
\end{align*}
Now we apply Assumptions \ref{assm:L-smoothness}, \ref{assm:global}, \ref{assm:local}:
\begin{align*}
   B_1^\prime  &\leq \frac{\mathcal{P}_{\mathcal{G}^k_C}}{C^2(1-\delta)^2} (1+\omega) \mathbb{E} \left[ G \frac{\mathcal{L}_{ \pm}^2}{b} \|x^{k+1} - x^k\|^2\right]\\
&+ \frac{\mathcal{P}_{\mathcal{G}^k_C}}{C^2(1-\delta)^2}(\omega+1) \mathbb{E} \left[ G L_{ \pm}^2 \|x^{k+1} - x^k\|^2\right]
+\frac{\mathcal{P}_{\mathcal{G}^k_C}}{C^2(1-\delta)^2} \omega \mathbb{E} \left[ G L^2\left\|  x^{k+1} - x^k\right\|^2\right].
\end{align*}
Finally, we have
\begin{align*}
B_1^\prime&\leq \frac{\mathcal{P}_{\mathcal{G}^k_C}\cdot G}{C^2(1-\delta)^2} \left(  \omega L^2+(\omega+1) L_{ \pm}^2+\frac{(\omega+1) \mathcal{L}_{ \pm}^2}{b} \right)\mathbb{E}\left[ \|x^{k+1} - x^k\|^2\right].
\end{align*}
Let us plug the obtained results in \eqref{eq:djbfjnskcnksbfjvbsd}:
\begin{align*}
  B_1      &\leq \mathbb{E}\left[\left\|g^{k}-\nabla f(x^{k})\right\|^2\right]\\
    &+ \frac{\mathcal{P}_{\mathcal{G}^k_C}\cdot G}{C^2(1-\delta)^2} \left(  \omega L^2+(\omega+1) L_{ \pm}^2+\frac{(\omega+1) \mathcal{L}_{ \pm}^2}{b} \right)\mathbb{E}\left[\|x^{k+1} - x^k\|^2\right].
\end{align*}
Also, we have 
\begin{align*}
    A_1 &= \mathbb{E}\left[\left\|\overline{g}^{k+1}-\nabla f(x^{k+1})\right\|^2\right]\\
    &\overset{\eqref{eq:ncsjindisbciusdbhc}, \eqref{eq:nsjknbvsbicusd}}{\leq} (1-p)p_G B_1 + (1-p)(1-p_{G})\mathbb{E}\left[\left\|g^k -\nabla f(x^{k})\right\|^2\right] \\
    &+ p\frac{ \delta\cdot\mathcal{P }_{\mathcal{G}^k_{\widehat{C}}} }{(1-\delta)}  \mathbb{E}\left[ B\|\nabla f(x)\|^2+\zeta^2 \right]\\
    &\leq (1-p)p_G \mathbb{E}\left[\left\|g^{k}-\nabla f(x^{k})\right\|^2\right] + p\frac{ \delta\cdot\mathcal{P }_{\mathcal{G}^k_{\widehat{C}}} }{(1-\delta)}  \mathbb{E}\left[ B\|\nabla f(x)\|^2+\zeta^2 \right]\\
    &+ \text{\small $(1-p)p_G\frac{\mathcal{P}_{\mathcal{G}^k_C}\cdot G}{C^2(1-\delta)^2} \left(  \omega L^2+(\omega+1) L_{ \pm}^2+\frac{(\omega+1) \mathcal{L}_{ \pm}^2}{b} \right)\mathbb{E}\left[\|x^{k+1} - x^k\|^2\right]$}\\
    &+(1-p)(1-p_{G})\mathbb{E}\left[\left\|g^k -\nabla f(x^{k})\right\|^2\right].
\end{align*}
Rearranging the terms, we get
\begin{align*}
   A_1  &\leq (1-p) \mathbb{E}\left[\left\|g^{k}-\nabla f(x^{k})\right\|^2\right] + \frac{ \delta\mathcal{P }_{\mathcal{G}^k_{\widehat{C}}} }{(1-\delta)}  \mathbb{E}\left[ B\|\nabla f(x)\|^2+\zeta^2 \right]\\
    &+ \text{\normalsize $(1-p)p_G\frac{\mathcal{P}_{\mathcal{G}^k_C}G}{C^2(1-\delta)^2} \left(  \omega L^2+(\omega+1) L_{ \pm}^2+\frac{(\omega+1) \mathcal{L}_{ \pm}^2}{b} \right)\mathbb{E}\left[\|x^{k+1} - x^k\|^2\right]$}.
    \end{align*}

\end{proof}

\begin{lemma}
\label{lemma:good_aggr_new}
Let Assumptions~\ref{assm:L-smoothness}, \ref{assm:global}, \ref{assm:local}, \ref{assm:bounded-compressor} hold and the compression operator satisfy Definition~\ref{def:Q}. Also, let us introduce the notation
$$\texttt{ARAgg}_Q^{k+1} = \texttt{ARAgg}\left(\clip_{\lambda_{k+1}}\left(\cQ\left(\widehat{\Delta}_1(x^{k+1}, x^k)\right)\right),\ldots, \clip_{\lambda_{k+1}}\left(\cQ\left(\widehat{\Delta}_C(x^{k+1}, x^k)\right)\right)\right).$$ 
Then for all $k\geq 0$ the iterates produced by \algname{Byz-VR-MARINA-PP} (Algorithm~\ref{alg:byz_vr_marina}) satisfy

\begin{align*}
       T_2 &=   \mathbb{E}\left[ \mathbb{E}_k\left[\left\| \frac{1}{G^k_C} \sum \limits_{i \in \mathcal{G}^k_C}\clip_{\lambda}\left(\mathcal{Q}\left(\widehat{\Delta}_i\left(x^{k+1}, x^k\right)\right)\right)  -   \texttt{ARAgg}_Q^{k+1}\right\|^2\mid [3]\right]\right]\\
       &\leq 4\frac{G\mathcal{P}_{\mathcal{G}^k_C}}{C(1-\delta)} c\delta \left( (1+\omega)\frac{\mathcal{L}_{ \pm}^2}{b} + (\omega+1) L_{ \pm}^2 + \omega  L^2 \right) \mathbb{E}\left[\|x^{k+1} - x^k\|^2\right],
\end{align*}
where $\mathcal{P}_{\mathcal{G}^k_C} =  \operatorname{Prob}\left\lbrace i \in \mathcal{G}^k_C \mid G^k_C \geq\left(1-\delta\right) C\right\rbrace$.
\end{lemma}
\begin{proof}
By definition of the robust aggregation, we have 
\begin{align*}
   T_2 &=   \mathbb{E}\left[ \mathbb{E}_k\left[\left\| \frac{1}{G^k_C} \sum \limits_{i \in \mathcal{G}^k_C}\clip_{\lambda}\left(\mathcal{Q}\left(\widehat{\Delta}_i\left(x^{k+1}, x^k\right)\right)\right)  -   \texttt{ARAgg}_Q^{k+1}\right\|^2\mid [3]\right]\right]\\
   &\text{ \scriptsize $\leq  \mathbb{E}\left[  \frac{c \delta}{D_2} \sum_{\substack{i, l \in \mathcal{G}_C^k \\
i \neq l}}    
\mathbb{E}_k\left[\left\| \clip_{\lambda}\left(\mathcal{Q}\left(\widehat{\Delta}_i\left(x^{k+1}, x^k\right)\right)\right) - \clip_{\lambda}\left(\mathcal{Q}\left(\widehat{\Delta}_l\left(x^{k+1}, x^k\right)\right)\right) \right\|^2\mid [3]\right]\right]$},
\end{align*}
where $D_2 = G^k_C(G^k_C-1)$.

Using $\lambda_{k+1} = D_Q\max_{i,j} L_{i,j} \|x^{k+1} - x^k\|$ we can guarantee that clipping operator becomes identical since we have $\forall i\in \cG$ 
\begin{align*}
   \left \|\mathcal{Q}\left(\widehat{\Delta}_i\left(x^{k+1}, x^k\right)\right)\right\| &\leq D_Q\left\|\widehat{\Delta}_i\left(x^{k+1}, x^k\right)\right\|\\
   &\leq D_Q \left\|\frac{1}{b} \sum_{j\in m} \nabla f_{i,j}(x^{k+1})-\nabla f_{i,j}(x^k) \right\|\\
    &\leq D_Q \frac{1}{b}\sum_{j\in m}\left\|  \nabla f_{i,j}(x^{k+1})-\nabla f_{i,j}(x^k) \right\|\\
    &\leq D_Q \max_j L_{i,j}\left\|  x^{k+1}-x^k \right\|
\end{align*}

Let us consider pair-wise differences: $\forall i,l \in \cG$
\begin{align*}
T_{2}^\prime(i,l) &= \mathbb{E}_k\left[\left\| \clip_{\lambda}\left(\mathcal{Q}\left(\widehat{\Delta}_i\left(x^{k+1}, x^k\right)\right)\right) - \clip_{\lambda}\left(\mathcal{Q}\left(\widehat{\Delta}_l\left(x^{k+1}, x^k\right)\right)\right) \right\|^2\mid [3]\right]\\
&= \mathbb{E}_k\left[\left\| \mathcal{Q}\left(\widehat{\Delta}_i\left(x^{k+1}, x^k\right)\right) - \mathcal{Q}\left(\widehat{\Delta}_l\left(x^{k+1}, x^k\right)\right) \right\|^2\mid [3]\right]\\
& =   \text{\tiny $\mathbb{E}_k\left[\left\| \mathcal{Q}\left(\widehat{\Delta}_i\left(x^{k+1}, x^k\right)\right)  - \Delta_i\left(x^{k+1}, x^k\right) + \Delta_l\left(x^{k+1}, x^k\right) -\mathcal{Q}\left(\widehat{\Delta}_l\left(x^{k+1}, x^k\right)\right) \right\|^2\mid [3]\right]$}\\
&+  \text{$\small  \mathbb{E}_k\left[\left\|  \Delta_i\left(x^{k+1}, x^k\right) - \Delta_l\left(x^{k+1}, x^k\right) \right\|^2\mid [3]\right]$}\\
& \stackrel{(\ref{eq:yung-1})}{\leq}  2 \text{$\small \mathbb{E}_k\left[\left\| \mathcal{Q}\left(\widehat{\Delta}_i\left(x^{k+1}, x^k\right)\right) -\Delta_i\left(x^{k+1}, x^k\right) \right\|^2\mid [3]\right]$}\\
&+ 2 \text{$\small \mathbb{E}_k\left[\left\| \Delta_l\left(x^{k+1}, x^k\right) - \mathcal{Q}\left(\widehat{\Delta}_l\left(x^{k+1}, x^k\right)\right) \right\|^2\mid [3]\right]$}\\
&+  \text{$\small \mathbb{E}_k\left[\left\|  \Delta_l\left(x^{k+1}, x^k\right) - \Delta_i\left(x^{k+1}, x^k\right) \right\|^2\mid [3]\right]]$}\\
& \stackrel{(\ref{eq:yung-1})}{\leq}  2 \text{$\small \mathbb{E}_k\left[\left\| \mathcal{Q}\left(\widehat{\Delta}_i\left(x^{k+1}, x^k\right)\right) -\Delta_i\left(x^{k+1}, x^k\right)  \right\|^2\mid [3]\right]$}\\
&+ 2 \text{$\small \mathbb{E}_k\left[\left\| \Delta_l\left(x^{k+1}, x^k\right) - \mathcal{Q}\left(\widehat{\Delta}_l\left(x^{k+1}, x^k\right)\right) \right\|^2\mid [3]\right]$}\\
&+ 2 \text{$\small  \mathbb{E}_k\left[\left\|  \Delta_l\left(x^{k+1}, x^k\right) - \Delta\left(x^{k+1}, x^k\right) \right\|^2 + \left\|  \Delta_i\left(x^{k+1}, x^k\right) - \Delta\left(x^{k+1}, x^k\right)\right\|^2\mid [3]\right]$}.
\end{align*}
Now we can combine all the parts together:
\begin{align*}
    \widehat{T}_2 & = \mathbb{E}\left[\frac{1}{G^k_C(G^k_C - 1)}  \sum_{\substack{i, l \in \mathcal{G}_C^k \\
i \neq l}} T_{2}^\prime(i,l) \right]  \\
& \leq   \mathbb{E}\left[ \frac{1}{D_2} \sum_{\substack{i, l \in \mathcal{G}_C^k \\
i \neq l}} 2  \mathbb{E}_k\left[\left\| \mathcal{Q}\left(\widehat{\Delta}_i\left(x^{k+1}, x^k\right)\right) -\Delta_i\left(x^{k+1}, x^k\right)  \right\|^2\mid [3]\right]\right]\\
&+  \mathbb{E}\left[  \frac{1}{D_2}  \sum_{\substack{i, l \in \mathcal{G}_C^k \\
i \neq l}} 2 \mathbb{E}_k\left[\left\| \Delta_l\left(x^{k+1}, x^k\right) - \mathcal{Q}\left(\widehat{\Delta}_l\left(x^{k+1}, x^k\right)\right) \right\|^2\mid [3]\right]\right]\\
&+    \mathbb{E}\left[  \frac{1}{D_2}  \sum_{\substack{i, l \in \mathcal{G}_C^k \\
i \neq l}} 2  \mathbb{E}_k\left[\left\|  \Delta_l\left(x^{k+1}, x^k\right) - \Delta\left(x^{k+1}, x^k\right) \right\|^2\mid [3]\right]\right]\\
&+   \mathbb{E}\left[  \frac{1}{D_2}  \sum_{\substack{i, l \in \mathcal{G}_C^k \\
i \neq l}} 2\small \mathbb{E}_k\left[\left\|  \Delta_i\left(x^{k+1}, x^k\right) - \Delta\left(x^{k+1}, x^k\right)\right\|^2\mid [3]\right]\right].
\end{align*}
Rearranging the terms, we get
\begin{align*}
    \widehat{T}_2& \leq   \mathbb{E}\left[ \frac{4}{G_C^k} \sum_{i \in \mathcal{G}_C^k}   \mathbb{E}_k\left[\left\| \mathcal{Q}\left(\widehat{\Delta}_i\left(x^{k+1}, x^k\right)\right) -\Delta_i\left(x^{k+1}, x^k\right)  \right\|^2\mid [3]\right]\right]\\
&+    \mathbb{E}\left[  \frac{4}{G_C^k} \sum_{i \in \mathcal{G}_C^k} \mathbb{E}_k\left[\left\|  \Delta_i\left(x^{k+1}, x^k\right) - \Delta\left(x^{k+1}, x^k\right) \right\|^2\mid [3]\right]\right].
\end{align*}
Using variance decomposition, we get
\begin{align*}
        \widehat{T}_2&\leq \mathbb{E}\left[ \frac{1}{G^k_C} \sum_{i \in \mathcal{G}^k_C} 4  \mathbb{E}_k\left[\left\| \mathcal{Q}\left(\widehat{\Delta}_i\left(x^{k+1}, x^k\right)\right) \right\|^2\mid [3] \right]\right]\\
        &-\mathbb{E}\left[ \frac{1}{G^k_C} \sum_{i \in \mathcal{G}^k_C} 4  \mathbb{E}_k\left[\left\| \Delta_i\left(x^{k+1}, x^k\right) \right\|^2\mid [3]\right] \right]\\
&+\mathbb{E}\left[ \frac{1}{G^k_C} \sum_{i \in \mathcal{G}^k_C} 4 \mathbb{E}_k\left[\left\|  \Delta_i\left(x^{k+1}, x^k\right) - \Delta\left(x^{k+1}, x^k\right) \right\|^2\mid [3]\right] \right].
\end{align*}    
Using the properties of unbiased compressors, we obtain 
\begin{align*}
        \widehat{T}_2&\leq \mathbb{E}\left[ \frac{1}{G^k_C} \sum_{i \in \mathcal{G}^k_C} 4(1+\omega)  \mathbb{E}_k\left[\left\| \widehat{\Delta}_i\left(x^{k+1}, x^k\right) \right\|^2\mid [3] \right]\right]\\
        &-\mathbb{E}\left[ \frac{1}{G^k_C} \sum_{i \in \mathcal{G}^k_C} 4  \mathbb{E}_k\left[\left\| \Delta_i\left(x^{k+1}, x^k\right) \right\|^2\mid [3]\right] \right]\\
&+\mathbb{E}\left[ \frac{1}{G^k_C} \sum_{i \in \mathcal{G}^k_C} 4 \mathbb{E}_k\left[\left\|  \Delta_i\left(x^{k+1}, x^k\right) - \Delta\left(x^{k+1}, x^k\right) \right\|^2\mid [3]\right] \right]\\
&\leq \mathbb{E}\left[ \frac{1}{G^k_C} \sum_{i \in \mathcal{G}^k_C} 4(1+\omega)  \mathbb{E}_k\left[\left\| \widehat{\Delta}_i\left(x^{k+1}, x^k\right) - \Delta_i\left(x^{k+1}, x^k\right)\right\|^2\mid [3] \right]\right]\\
        &+\mathbb{E}\left[ \frac{1}{G^k_C} \sum_{i \in \mathcal{G}^k_C} 4(1+\omega)  \mathbb{E}_k\left[\left\| \Delta_i\left(x^{k+1}, x^k\right) \right\|^2\mid [3]\right] \right]\\
        &-\mathbb{E}\left[ \frac{1}{G^k_C} \sum_{i \in \mathcal{G}^k_C} 4  \mathbb{E}_k\left[\left\| \Delta_i\left(x^{k+1}, x^k\right) \right\|^2\mid [3]\right] \right]\\
&+\mathbb{E}\left[ \frac{1}{G^k_C} \sum_{i \in \mathcal{G}^k_C} 4 \mathbb{E}_k\left[\left\|  \Delta_i\left(x^{k+1}, x^k\right) - \Delta\left(x^{k+1}, x^k\right) \right\|^2\mid [3]\right] \right].
\end{align*} 
Let us simplify the inequality:
\begin{align*}
    \widehat{T}_2 &\leq \mathbb{E}\left[ \frac{1}{G^k_C} \sum_{i \in \mathcal{G}^k_C} 4(1+\omega)  \mathbb{E}_k\left[\left\| \widehat{\Delta}_i\left(x^{k+1}, x^k\right) - \Delta_i\left(x^{k+1}, x^k\right)\right\|^2\mid [3] \right]\right]\\
        &+\mathbb{E}\left[ \frac{1}{G^k_C} \sum_{i \in \mathcal{G}^k_C} 4\omega  \mathbb{E}_k\left[\left\| \Delta_i\left(x^{k+1}, x^k\right) \right\|^2\mid [3]\right] \right]\\
&+\mathbb{E}\left[ \frac{1}{G^k_C} \sum_{i \in \mathcal{G}^k_C} 4 \mathbb{E}_k\left[\left\|  \Delta_i\left(x^{k+1}, x^k\right) - \Delta\left(x^{k+1}, x^k\right) \right\|^2\mid [3]\right] \right].
\end{align*}
Using variance decomposition once again, we get
\begin{align*}
    \widehat{T}_2 &\leq \mathbb{E}\left[ \frac{1}{G^k_C} \sum_{i \in \mathcal{G}^k_C} 4(1+\omega)  \mathbb{E}_k\left[\left\| \widehat{\Delta}_i\left(x^{k+1}, x^k\right) - \Delta_i\left(x^{k+1}, x^k\right)\right\|^2\mid [3] \right]\right]\\
        &+\mathbb{E}\left[ \frac{1}{G^k_C} \sum_{i \in \mathcal{G}^k_C} 4\omega  \mathbb{E}_k\left[\left\| \Delta_i\left(x^{k+1}, x^k\right) - \Delta\left(x^{k+1}, x^k\right)  \right\|^2\mid [3]\right] \right]\\
&+\mathbb{E}\left[ \frac{1}{G^k_C} \sum_{i \in \mathcal{G}^k_C} 4 \mathbb{E}_k\left[\left\|  \Delta_i\left(x^{k+1}, x^k\right) - \Delta\left(x^{k+1}, x^k\right) \right\|^2\mid [3]\right] \right]\\
&+ \mathbb{E}\left[ \frac{1}{G^k_C} \sum_{i \in \mathcal{G}^k_C} 4\omega  \mathbb{E}_k\left[\left\| \Delta\left(x^{k+1}, x^k\right)  \right\|^2\mid [3]\right] \right].
\end{align*}
Then, we apply similar arguments to the ones used in deriving \eqref{eq:indicator_1}:
\begin{align*}
    \widehat{T}_2 &\leq \mathbb{E}\left[ \frac{\mathcal{P}_{\mathcal{G}^k_C}}{C(1-\delta)} \sum_{i \in \mathcal{G}} 4(1+\omega)  \mathbb{E}_k\left[\left\| \widehat{\Delta}_i\left(x^{k+1}, x^k\right) - \Delta_i\left(x^{k+1}, x^k\right)\right\|^2\mid [3] \right]\right]\\
        &+\mathbb{E}\left[ \frac{\mathcal{P}_{\mathcal{G}^k_C}}{C(1-\delta)} \sum_{i \in \mathcal{G}} 4\omega  \mathbb{E}_k\left[\left\| \Delta_i\left(x^{k+1}, x^k\right) - \Delta\left(x^{k+1}, x^k\right)  \right\|^2\mid [3]\right] \right]\\
&+\mathbb{E}\left[ \frac{\mathcal{P}_{\mathcal{G}^k_C}}{C(1-\delta)} \sum_{i \in \mathcal{G}} 4 \mathbb{E}_k\left[\left\|  \Delta_i\left(x^{k+1}, x^k\right) - \Delta\left(x^{k+1}, x^k\right) \right\|^2\mid [3]\right] \right]\\
&+ \mathbb{E}\left[ \frac{\mathcal{P}_{\mathcal{G}^k_C}}{C(1-\delta)} \sum_{i \in \mathcal{G}} 4\omega  \mathbb{E}_k\left[\left\| \Delta\left(x^{k+1}, x^k\right)  \right\|^2\mid [3]\right] \right].
\end{align*}
Using Assumptions \ref{assm:L-smoothness}, \ref{assm:global}, \ref{assm:local}: 
    \begin{align*}
    \widehat{T}_2 &\leq \mathbb{E}\left[  4(1+\omega)  \frac{G\mathcal{P}_{\mathcal{G}^k_C}}{C(1-\delta)} \frac{\mathcal{L}_{ \pm}^2}{b}\|x^{k+1}-x^k\|^2 \right] +\mathbb{E}\left[  4(\omega+1)\frac{G\mathcal{P}_{\mathcal{G}^k_C}}{C(1-\delta)} \omega  L_{ \pm}^2 \|x^{k+1}-x^k\|^2\right]\\
        &+\mathbb{E}\left[  4\frac{G\mathcal{P}_{\mathcal{G}^k_C}}{C(1-\delta)} \omega  L^2 \|x^{k+1}-x^k\|^2\right].
\end{align*}
Finally, we obtain
\begin{align*}
       T_2 &=   \mathbb{E}\left[ \mathbb{E}_k\left[\left\| \frac{1}{G^k_C} \sum \limits_{i \in \mathcal{G}^k_C}\clip_{\lambda}\left(\mathcal{Q}\left(\widehat{\Delta}_i\left(x^{k+1}, x^k\right)\right)\right)  -   \texttt{ARAgg}_Q^{k+1}\right\|^2\mid [3]\right]\right]\\
       &\leq 4\frac{G\mathcal{P}_{\mathcal{G}^k_C}}{C(1-\delta)} c\delta \left( (1+\omega)\frac{\mathcal{L}_{ \pm}^2}{b} + (\omega+1) L_{ \pm}^2 + \omega  L^2 \right) \mathbb{E}\left[\|x^{k+1} - x^k\|^2\right].
\end{align*}
\end{proof}

\begin{lemma}
\label{lemma:final_lemma_Q}
    Let Assumptions ~\ref{assm:bounded-aggr}, \ref{assm:L-smoothness}, \ref{assm:global}, \ref{assm:local}, \ref{assm:het}, \ref{assm:bounded-compressor} hold and the compression operator satisfy Definition~\ref{def:Q}. We set $\lambda_{k+1} = D_Q \max_{i,j} L_{i,j}$. Also let us introduce the notation
$$\texttt{ARAgg}_Q^{k+1} = \texttt{ARAgg}\left(\clip_{\lambda_{k+1}}\left(\cQ\left(\widehat{\Delta}_1(x^{k+1}, x^k)\right)\right),\ldots, \clip_{\lambda_{k+1}}\left(\cQ\left(\widehat{\Delta}_C(x^{k+1}, x^k)\right)\right)\right).$$ 
Then for all $k\geq 0$ the iterates produced by \algname{Byz-VR-MARINA-PP} (Algorithm~\ref{alg:byz_vr_marina}) satisfy

    \begin{align*}
    \mathbb{E}\left[\left\|g^{k+1}-\nabla f\left(x^{k+1}\right)\right\|^2\right] &\leq  \left(1-\frac{p}{2}\right) \mathbb{E}\left[\left\|g^{k}-\nabla f\left(x^{k}\right)\right\|^2\right]\\
    &+\widehat{B}\mathbb{E}\left[\left\|\nabla f\left(x^k\right)\right\|^2\right] + \widehat{D} \zeta^2+\frac{pA}{4}\|x^{k+1} - x^k\|^2,
\end{align*}
with
\begin{align*}
    A& = \frac{4}{p}\left(  \frac{p_G \mathcal{P}_{\mathcal{G}^k_C} G}{C^2(1-\delta)^2} \omega + \frac{8G\cP_{\cG_{\widehat{C}}^k}c\delta}{(1-\delta)\widehat{C}} B + \frac{4}{p}(1-p_G) + \frac{8}{p}p_G \frac{G\mathcal{P}_{\mathcal{G}^k_C}}{C(1-\delta)} c \delta \omega  \right) L^2\\
    &+\frac{4}{p}\left( \frac{p_G \mathcal{P}_{\mathcal{G}^k_C} G}{C^2(1-\delta)^2} \left( \omega + 1 \right) + \frac{8}{p}p_G \frac{G\mathcal{P}_{\mathcal{G}^k_C}}{C(1-\delta)} c \delta (\omega+1) \right)\left(L_{ \pm}^2 + \frac{\mathcal{L}_{ \pm}^2}{b}\right)\\
    &+\frac{16}{p^2}(1-p_G)F_{\cA}^2 \left(D_Q \max_{i,j} L_{i,j} \right)^2
\end{align*}
\textbf{\begin{align*}
    \widehat{B} = 2 \frac{ \delta\mathcal{P}_{\mathcal{G}^k_{\widehat{C}}} }{1-\delta} B\left(\frac{12cG}{\widehat C} + p \right), \quad \widehat{D} = 2 \frac{ \delta\mathcal{P}_{\mathcal{G}^k_{\widehat{C}}} }{1-\delta} \left(\frac{6cG}{\widehat C} + p \right),
\end{align*}}
where $p_G = \operatorname{Prob}\left\lbrace G^k_C \geq (1-\delta_{\max})C \right\rbrace $ and $\mathcal{P}_{\mathcal{G}^k_C} =  \operatorname{Prob}\left\lbrace i \in \mathcal{G}^k_C \mid G^k_C \geq\left(1-\delta_{\max}\right) C\right\rbrace$.
\end{lemma}
\begin{proof}
    Let us combine bounds for $A_1$ and $A_2$ together:
\begin{align*}
     A_0 &=    \mathbb{E}\left[\left\|g^{k+1}-\nabla f\left(x^{k+1}\right)\right\|^2\right]\\
     & \leq \left(1+\frac{p}{2}\right)\mathbb{E}\left[\left\|\overline{g}^{k+1}-\nabla f\left(x^{k+1}\right)\right\|^2\right]+\left(1+\frac{2}{p}\right)\mathbb{E}\left[\left\|g^{k+1}-\overline{g}^{k+1}\right\|^2\right]\\
     &\leq \left(1+\frac{p}{2}\right)A_1 + \left(1+\frac{2}{p}\right)A_2\\
       & \leq \left(1+\frac{p}{2}\right)(1-p) \mathbb{E}\left[\left\|g^{k}-\nabla f(x^{k})\right\|^2\right] + \left(1+\frac{p}{2}\right)p\frac{ \delta\mathcal{P }_{\mathcal{G}^k_{\widehat{C}}} }{(1-\delta)}  \mathbb{E}\left[ B\|\nabla f(x)\|^2+\zeta^2 \right]\\
        &+ \text{\scriptsize $\left(1+\frac{p}{2}\right)(1-p)p_G\frac{\mathcal{P}_{\mathcal{G}^k_C} G}{C^2(1-\delta)^2} \left(  \omega L^2+(\omega+1) L_{ \pm}^2+\frac{(\omega+1) \mathcal{L}_{ \pm}^2}{b} \right)\mathbb{E}\left[\|x^{k+1} - x^k\|^2\right]$}\\
    &+ \left(1+\frac{2}{p}\right)p\mathbb{E}\left[\mathbb{E}_k\left[\left\|\texttt{ARAgg}\left(\nabla f_1(x^{k+1}), \ldots, \nabla f_n(x^{k+1})\right) - \nabla f(x^{k+1})\right\|^2\right]\mid [1]\right]\\
   & \text{\small$+ \left(1+\frac{2}{p}\right)(1-p)p_G \mathbb{E}\left[ \mathbb{E}_k\left[\left\| \frac{1}{G^k_C} \sum \limits_{i \in \mathcal{G}^k_C}\clip_{\lambda}\left(\mathcal{Q}\left(\widehat{\Delta}_i\left(x^{k+1}, x^k\right)\right)\right)  -   \texttt{ARAgg}_Q^{k+1}\right\|^2\mid [3]\right]\right]$}\\
   &+ \left(1+\frac{2}{p}\right)(1-p)(1-p_G)\mathbb{E}\left[\mathbb{E}_k\left[\left\| \nabla f(x^{k+1}) - \nabla f(x^{k}) - \texttt{ARAgg}_Q^{k+1}\right\|^2\mid [2]\right]\right].
   \end{align*}
   Using Lemma \ref{lemma:good_aggr_new} and lemmas from General Analysis (Lemmas~\ref{lemma:full_aggr} and \ref{lemma:bad_aggr}) we have 
     \begin{align*}
   A_0 &=  \mathbb{E}\left[\left\|g^{k+1}-\nabla f\left(x^{k+1}\right)\right\|^2\right]\\ 
   & \leq \left(1-\frac{p}{2}\right) \mathbb{E}\left[\left\|g^{k}-\nabla f\left(x^{k}\right)\right\|^2\right] + 2p\frac{ \delta\mathcal{P }_{\mathcal{G}^k_{\widehat{C}}} }{(1-\delta)}  \mathbb{E}\left[ B\|\nabla f(x)\|^2+\zeta^2 \right]\\
            &+ \text{\normalsize $\left(1-\frac{p}{2}\right)p_G\frac{\mathcal{P}_{\mathcal{G}^k_C} G}{C^2(1-\delta)^2} \left(  \omega L^2+(\omega+1) L_{ \pm}^2+\frac{(\omega+1) \mathcal{L}_{ \pm}^2}{b} \right)\mathbb{E}\left[\|x^{k+1} - x^k\|^2\right]$}\\
      &+ \left(p+2\right)\left(\frac{8 G \mathcal{P}_{\mathcal{G}^k_{\widehat{C}}} c\delta}{(1-\delta)\widehat{C}}  B  \mathbb{E}\left[\left\|\nabla f\left(x^k\right)\right\|^2\right] + \frac{8 G \mathcal{P}_{\mathcal{G}^k_{\widehat{C}}} c\delta}{(1-\delta)\widehat{C}}  B  L^2 \mathbb{E}\left[\left\|x^{k+1}-x^k\right\|^2\right] + \frac{4 G \mathcal{P}_{\mathcal{G}^k_{\widehat{C}}} c\delta}{(1-\delta)\widehat{C}}  \zeta^2\right)\\
   & + \frac{2}{p}p_G\mathbb{E}\left[  4(1+\omega)  \frac{G\mathcal{P}_{\mathcal{G}^k_C}}{C(1-\delta)} c\delta \frac{\mathcal{L}_{ \pm}^2}{b} \|x^{k+1}-x^k\|^2 \right] \\
    &    + \frac{2}{p}p_G \mathbb{E}\left[  4(\omega+1)\frac{G\mathcal{P}_{\mathcal{G}^k_C}}{C(1-\delta)} c\delta L_{ \pm}^2  \|x^{k+1}-x^k\|^2\right]\\
     &   + \frac{2}{p}p_G \mathbb{E}\left[  4\omega\frac{G\mathcal{P}_{\mathcal{G}^k_C}}{C(1-\delta)} c\delta   L^2   \|x^{k+1}-x^k\|^2\right]+ \frac{2}{p}(1-p_G)2(L^2+F_{\cA}^2\alpha^2_{\lambda_{k+1}})\mathbb{E}\left[\left\|  x^{k+1} - x^k\right\|^2\right].
\end{align*}
Finally, we have 
    \begin{align*}
    \mathbb{E}\left[\left\|g^{k+1}-\nabla f\left(x^{k+1}\right)\right\|^2\right] &\leq  \left(1-\frac{p}{2}\right) \mathbb{E}\left[\left\|g^{k}-\nabla f\left(x^{k}\right)\right\|^2\right]\\
    &+\widehat{B}\mathbb{E}\left[\left\|\nabla f\left(x^k\right)\right\|^2\right] + \widehat{D}\zeta^2+\frac{pA}{4}\|x^{k+1} - x^k\|^2,
\end{align*}
where 
\begin{align*}
    A& = \frac{4}{p}\left(  \frac{p_G \mathcal{P}_{\mathcal{G}^k_C} G}{C^2(1-\delta)^2} \omega + \frac{8G\cP_{\cG_{\widehat{C}}^k}c\delta}{(1-\delta)\widehat{C}} B + \frac{4}{p}(1-p_G) + \frac{8}{p}p_G \frac{G\mathcal{P}_{\mathcal{G}^k_C}}{C(1-\delta)} c \delta \omega  \right) L^2\\
    &+\frac{4}{p}\left( \frac{p_G \mathcal{P}_{\mathcal{G}^k_C} G}{C^2(1-\delta)^2} \left( \omega + 1 \right) + \frac{8}{p}p_G \frac{G\mathcal{P}_{\mathcal{G}^k_C}}{C(1-\delta)} c \delta (\omega+1) \right)\left(L_{ \pm}^2 + \frac{\mathcal{L}_{ \pm}^2}{b}\right)\\
    &+\frac{16}{p^2}(1-p_G)F_{\cA}^2 \left(D_Q \max_{i,j} L_{i,j} \right)^2
\end{align*}
and
\textbf{\begin{align*}
    \widehat{B} = 2 \frac{ \delta\mathcal{P}_{\mathcal{G}^k_{\widehat{C}}} }{1-\delta} B\left(\frac{12cG}{\widehat C} + p \right), \quad \widehat{D} = 2 \frac{ \delta\mathcal{P}_{\mathcal{G}^k_{\widehat{C}}} }{1-\delta} \left(\frac{6cG}{\widehat C} + p \right).
\end{align*}}
\end{proof}

\subsection{Main Results}

\begin{theorem}
 Let Assumptions \ref{assm:bounded-aggr}, \ref{assm:L-smoothness}, \ref{assm:global}, \ref{assm:local}, \ref{assm:het}, \ref{assm:bounded-compressor} hold. Setting $\lambda_{k+1} = \max_{i,j} L_{i,j} \left\|x^{k+1} - x^k\right\|$. Assume that
$$
0<\gamma \leq \frac{1}{L+\sqrt{A}}, \quad 4\widehat{B}<p,
$$ 
where 
\begin{align*}
    A& = \frac{4}{p}\left(  \frac{p_G \mathcal{P}_{\mathcal{G}^k_C} G}{C^2(1-\delta)^2} \omega + \frac{8G\cP_{\cG_{\widehat{C}}^k}c\delta}{(1-\delta)\widehat{C}} B + \frac{4}{p}(1-p_G) + \frac{8}{p}p_G \frac{G\mathcal{P}_{\mathcal{G}^k_C}}{C(1-\delta)} c \delta \omega  \right) L^2\\
    &+\frac{4}{p}\left( \frac{p_G \mathcal{P}_{\mathcal{G}^k_C} G}{C^2(1-\delta)^2} \left( \omega + 1 \right) + \frac{8}{p}p_G \frac{G\mathcal{P}_{\mathcal{G}^k_C}}{C(1-\delta)} c \delta (\omega+1) \right)\left(L_{ \pm}^2 + \frac{\mathcal{L}_{ \pm}^2}{b}\right)\\
    &+\frac{16}{p^2}(1-p_G)F_{\cA}^2 \left(D_Q \max_{i,j} L_{i,j} \right)^2,
\end{align*}
\textbf{\begin{align*}
    \widehat{B} = 2 \frac{ \delta\mathcal{P}_{\mathcal{G}^k_{\widehat{C}}} }{1-\delta} B\left(\frac{12cG}{\widehat C} + p \right), \quad \widehat{D} = 2 \frac{ \delta\mathcal{P}_{\mathcal{G}^k_{\widehat{C}}} }{1-\delta} \left(\frac{6cG}{\widehat C} + p \right),
\end{align*}}
and 
\begin{align*}
 \mathcal{P}_{\mathcal{G}^k_C} &=   \frac{C}{np_G} \cdot \sum_{(1-\delta)C\leq t \leq C} \left(\left(\begin{array}{l}
G-1 \\
t-1
\end{array}\right)\left(\begin{array}{l}
n-G \\
C-t
\end{array}\right) \left(\left(\begin{array}{l}
n \\
C
\end{array}\right)\right)^{-1} \right),\\
p_G &=  \PP\left\lbrace G_C^k \geq\left(1-\delta\right) C\right\rbrace\\
&= \sum_{\lceil(1-\delta)C\rceil\leq t \leq C} \left(\left(\begin{array}{l}
G \\
t
\end{array}\right)\left(\begin{array}{l}
n-G \\
C-t
\end{array}\right) \left(\begin{array}{l}
n \\
C
\end{array}\right)^{-1} \right).
\end{align*}
Then for all $K \geq 0$ the iterates produced by \algname{Byz-VR-MARINA} (Algorithm \ref{alg:byz_vr_marina}) satisfy
$$
\mathbb{E}\left[\left\|\nabla f\left(\widehat{x}^K\right)\right\|^2\right] \leq \frac{2 \Phi^0}{\gamma\left(1-\frac{4\widehat{B}}{p}\right)(K+1)}+\frac{2 \widehat{D} \zeta^2}{p-4 \widehat{B}},
$$
where $\widehat{x}^K$ is choosen uniformly at random from $x^0, x^1, \ldots, x^K$, and $\Phi^0=$ $f\left(x^0\right)-f^*+\frac{\gamma}{p}\left\|g^0-\nabla f\left(x^0\right)\right\|^2$ . 
\end{theorem}
\begin{proof}
    The proof is analogous to the proof of Theorem~\ref{them:1}.
\end{proof}

\begin{theorem}
Let Assumptions  \ref{assm:bounded-aggr}, \ref{assm:bounded-compressor}, \ref{assm:L-smoothness}, \ref{assm:global}, \ref{assm:local}, \ref{assm:het}, \ref{assm:PL} hold. Setting $\lambda_{k+1} = \max_{i,j} L_{i,j} \left\|x^{k+1} - x^k\right\|$. Assume that
$$
0<\gamma \leq \min \left\{\frac{1}{L+\sqrt{2 A}} \right\}, \quad 8\widehat{B}<p,
$$
where
\begin{align*}
    A& = \frac{4}{p}\left(  \frac{p_G \mathcal{P}_{\mathcal{G}^k_C} G}{C^2(1-\delta)^2} \omega + \frac{8G\cP_{\cG_{\widehat{C}}^k}c\delta}{(1-\delta)\widehat{C}} B + \frac{4}{p}(1-p_G) + \frac{8}{p}p_G \frac{G\mathcal{P}_{\mathcal{G}^k_C}}{C(1-\delta)} c \delta \omega  \right) L^2\\
    &+\frac{4}{p}\left( \frac{p_G \mathcal{P}_{\mathcal{G}^k_C} G}{C^2(1-\delta)^2} \left( \omega + 1 \right) + \frac{8}{p}p_G \frac{G\mathcal{P}_{\mathcal{G}^k_C}}{C(1-\delta)} c \delta (\omega+1) \right)\left(L_{ \pm}^2 + \frac{\mathcal{L}_{ \pm}^2}{b}\right)\\
    &+\frac{16}{p^2}(1-p_G)F_{\cA}^2 \left(D_Q \max_{i,j} L_{i,j} \right)^2,
\end{align*}
\begin{align*}
    \widehat{B} = 2 \frac{ \delta\mathcal{P}_{\mathcal{G}^k_{\widehat{C}}} }{1-\delta} B\left(\frac{12cG}{\widehat C} + p \right), \quad \widehat{D} = 2 \frac{ \delta\mathcal{P}_{\mathcal{G}^k_{\widehat{C}}} }{1-\delta} \left(\frac{6cG}{\widehat C} + p \right),
\end{align*}
and where $p_G = \operatorname{Prob}\left\lbrace G^k_C \geq (1-\delta)C \right\rbrace $ and $\mathcal{P}_{\mathcal{G}^k_C} =  \operatorname{Prob}\left\lbrace i \in \mathcal{G}^k_C \mid G^k_C \geq\left(1-\delta\right) C\right\rbrace$. Then for all $K \geq 0$ the iterates produced by \algname{Byz-VR-MARINA} (Algorithm \ref{alg:byz_vr_marina}) satisfy
$$
\mathbb{E}\left[f\left(x^K\right)-f\left(x^*\right)\right] \leq\left(1-\rho\right)^K \Phi^0+\frac{2\widehat{D}\zeta^2}{p\rho},
$$
where $\rho = \min\left[\gamma\mu\left(1-\frac{8\widehat{B}}{p}\right), \frac{p}{4}\right]$ and $\Phi^0=$ $f\left(x^0\right)-f^*+\frac{2\gamma}{p}\left\|g^0-\nabla f\left(x^0\right)\right\|^2$.
\end{theorem}
\begin{proof}
    The proof is analogous to the proof of Theorem~\ref{them:2}.
\end{proof}

\subsection{On the Technical Non-Triviality of the Analysis}\label{appendix:technical_challenges}

As we explain in the main part of the paper, the main reason why we propose to use clipping is to handle the situations when Byzantine workers form a majority during some communication rounds since the existing approaches are vulnerable to such scenarios. However, the introduction of the clipping does not come for free: if the clipping level is too small, clipping can create a noticeable bias to the updates. Because of this issue, existing works such as \citep{zhang2020adaptive, gorbunov2020stochastic} use non-trivial policies for the choice of the clipping level, and the analysis in these works differs significantly from the existing analysis for the methods without clipping. The analysis of \algname{Byz-VR-MARINA} is based on the unbiasedness of vectors $\mathcal{Q}(\hat \Delta_i(x^{k+1}, x^k))$, i.e., on the following identity: $\mathbb{E}[\mathcal{Q}(\hat \Delta_i(x^{k+1}, x^k)) \mid x^{k+1}, x^k] = \Delta_i(x^{k+1}, x^k) = \nabla f_i(x^{k+1}) - \nabla f_i(x^k)$. Since $\mathbb{E}[\clip_{\lambda_{k+1}}(\mathcal{Q}(\hat \Delta_i(x^{k+1}, x^k))) \mid x^{k+1}, x^k] \neq \nabla f_i(x^{k+1}) - \nabla f_i(x^k)$ in general, to analyze \algname{Byz-VR-MARINA-PP} we also use a special choice of the clipping level: $\lambda_{k+1} = \alpha_{k+1} \|x^{k+1} - x^k\|$. To illustrate the main reasons for that, let us consider the case of uncompressed communication ($\mathcal{Q}(x) \equiv x$). In this setup, for large enough $\alpha_{k+1}$ we have $\clip_{\lambda_{k+1}}\hat \Delta_i(x^{k+1}, x^k) = \hat \Delta_i(x^{k+1}, x^k)$ for all $i\in \mathcal{G}$ (due to Assumption~\ref{assm:smoothness_simplified}), which allows us using a similar proof to the one for Byz-VR-MARINA when good workers form a majority in a round. Moreover, when Byzantine workers form a majority, our choice of the clipping level allows us to bound the second moment of the shift from the Byzantine workers as $\sim \| x^{k+1} - x^k \|^2$ (see Lemmas~\ref{lemma:premainA2} and \ref{lemma:bad_aggr}), i.e., the second moment of the shift is of the same scale as the variance of $\lbrace g_i \rbrace_{i\in \mathcal{G}}$, which goes to zero. Next, to properly analyze these two situations, we overcame another technical challenge related to the estimation of the conditional expectations and probabilities of corresponding events (see Lemmas \ref{lemma:premainA2} and \ref{lemma:full_aggr} and formulas for $p_G$ and $\mathcal{P}_{\mathcal{G}_C^k}$ at the beginning of Section~\ref{section:convergence_results}). In particular, the derivation of formula \eqref{eq:bvdjbjdfbvjdf} is quite non-standard for stochastic optimization literature: there are two sources of stochasticity – one comes from the sampling of clients, and the other one comes from the sampling of stochastic gradients and compression. This leads to the estimation of variance of the average of the random number of random vectors, which is novel on its own. In addition, when the compression operator is used, the analysis becomes even more involved since one cannot directly apply the main property of unbiased compression (Definition~\ref{def:Q}), and we use Lemma~\ref{lemma:clipping} in the proof to address this issue. It is also worth mentioning that in contrast to \algname{Byz-VR-MARINA}, our method does not require full participation even with a small probability $p$. Instead, it is sufficient for \algname{Byz-VR-MARINA-PP} to sample a large enough cohort of $\widehat{C}$ clients with probability $p$ to ensure that Byzantine workers form a minority in such rounds.

\clearpage

\section{Experimental Details and Extra Experiments}
\label{app:experiments}

\subsection{Experimental Details}

For each experiment, we tune the step size using the following set of candidates $\{0.1, 0.01, 0.001\}$. The step size is fixed. We do not use learning rate warmup or decay. We use batches of size $32$ for all methods. For partial participation, in each round, we sample $20 \%$ of clients uniformly at random.  For $\lambda_k = \lambda \|x^k - x^{k-1}\|$ used for clipping, we select $\lambda$ from $\{0.1, 1., 10.\}$. Each experiment is run with three varying random seeds, and we report the mean optimality gap with one standard error. The optimal value is obtained by running gradient descent (\algname{GD}) on the complete dataset for 1000 epochs. Our implementation of attacks and robust aggregation schemes is based on the public implementation from \citep{gorbunov2023variance}. %

\subsection{Extra Experiments}

Below we provide the missing neural network experiments from the main paper. We consider the MNIST dataset~\citep{mnist} with heterogeneous splits (as in \citep{karimireddy2021learning}) with 20 clients, 5 of which are malicious. For the attacks, we consider A Little is Enough (ALIE)~\citep{baruch2019little}, Bit Flipping (BF), Label Flipping (LF), and aforementioned Shift-Back (SHB). For the aggregations, we consider coordinate median (CM)~\citep{chen2017distributed} and robust federated averaging (RFA)~\citep{pillutla2022robust} with bucketing. 

\begin{figure}[H]
\centering
\includegraphics[width=0.245\textwidth]{figures/MNIST_non_iid_comp=none_agg=cm_attack_BF_clip_sensitivity.pdf}
\includegraphics[width=0.245\textwidth]{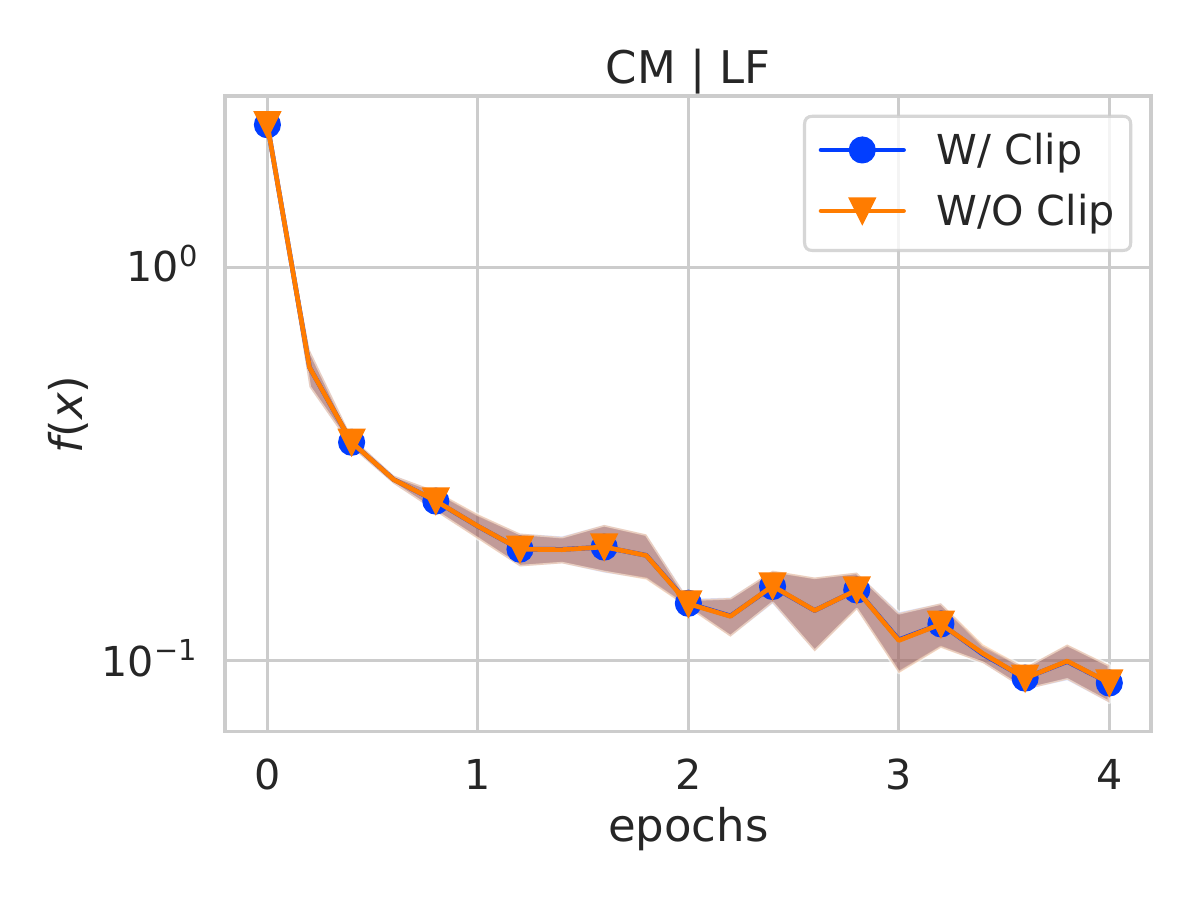}
\includegraphics[width=0.245\textwidth]{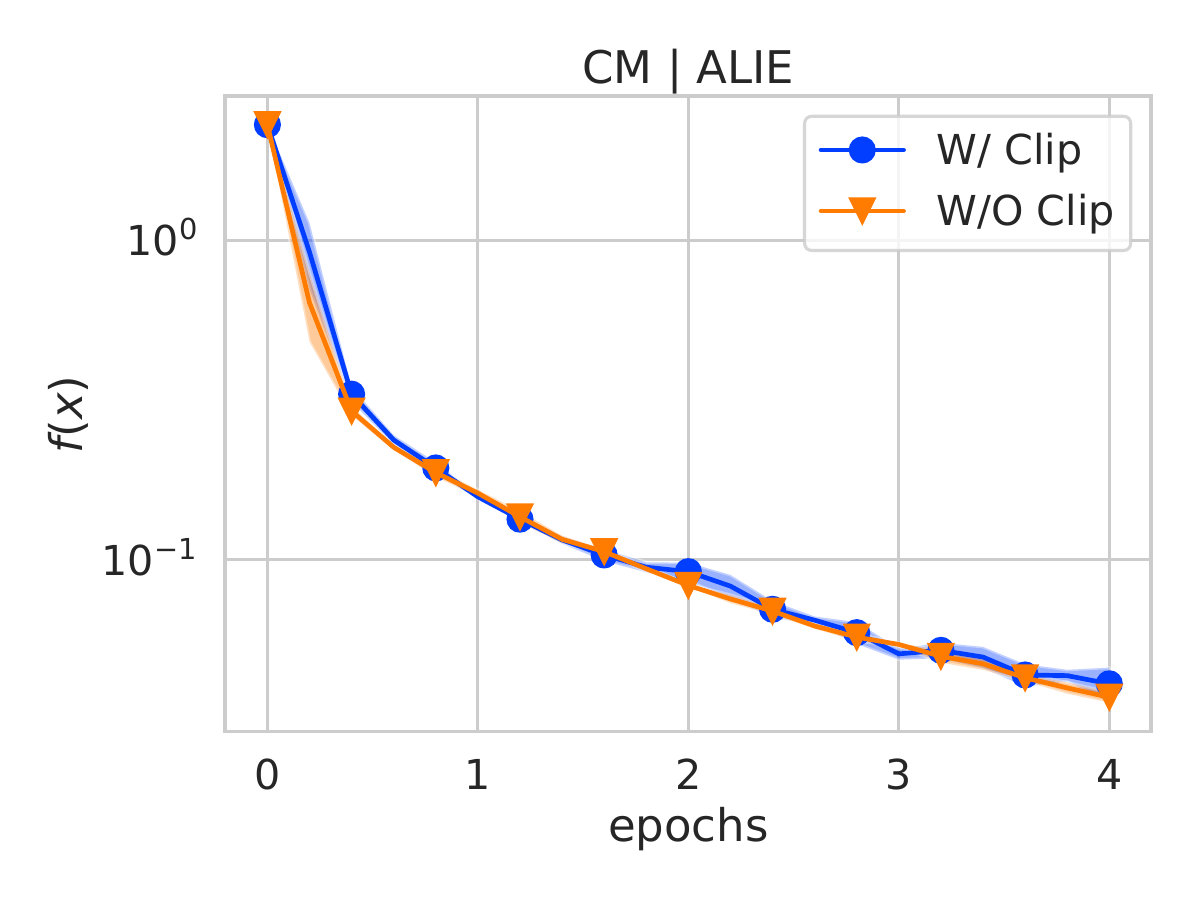}
\includegraphics[width=0.245\textwidth]{figures/MNIST_non_iid_comp=none_agg=cm_attack_SHB_clip_sensitivity.pdf}
\hfill
\includegraphics[width=0.245\textwidth]{figures/MNIST_non_iid_comp=none_agg=rfa_attack_BF_clip_sensitivity.pdf}
\includegraphics[width=0.245\textwidth]{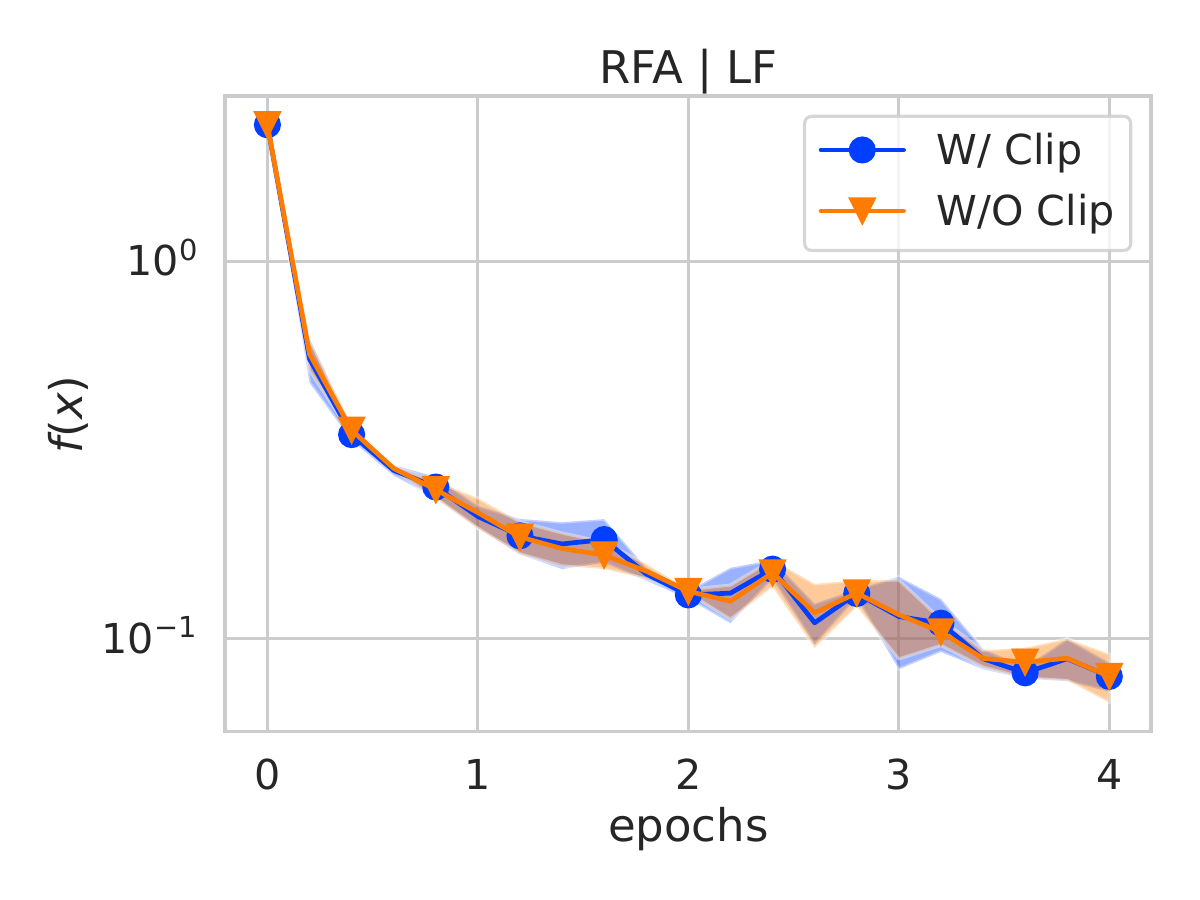}
\includegraphics[width=0.245\textwidth]{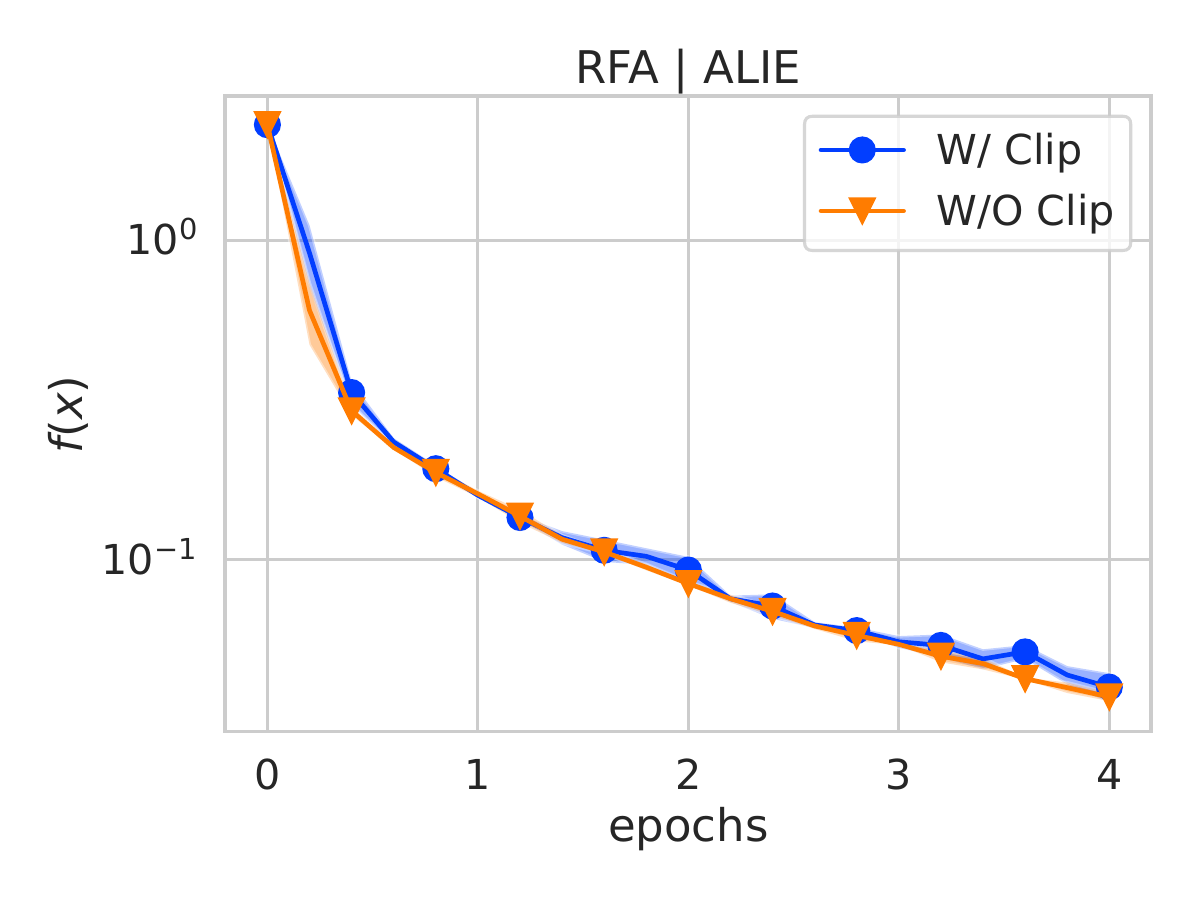}
\includegraphics[width=0.245\textwidth]{figures/MNIST_non_iid_comp=none_agg=rfa_attack_SHB_clip_sensitivity.pdf}
\caption{
Training loss of 2 aggregation rules (CM, RFA) under 4 attacks (BF, LF, ALIE, SHB) on the MNIST dataset under heterogeneous data split with 20 clients,  5 of which are malicious. } 
\label{fig:nn_app}
\vspace{-1.5em}
\end{figure}

\begin{figure}[H]
\centering
\includegraphics[width=0.245\textwidth]{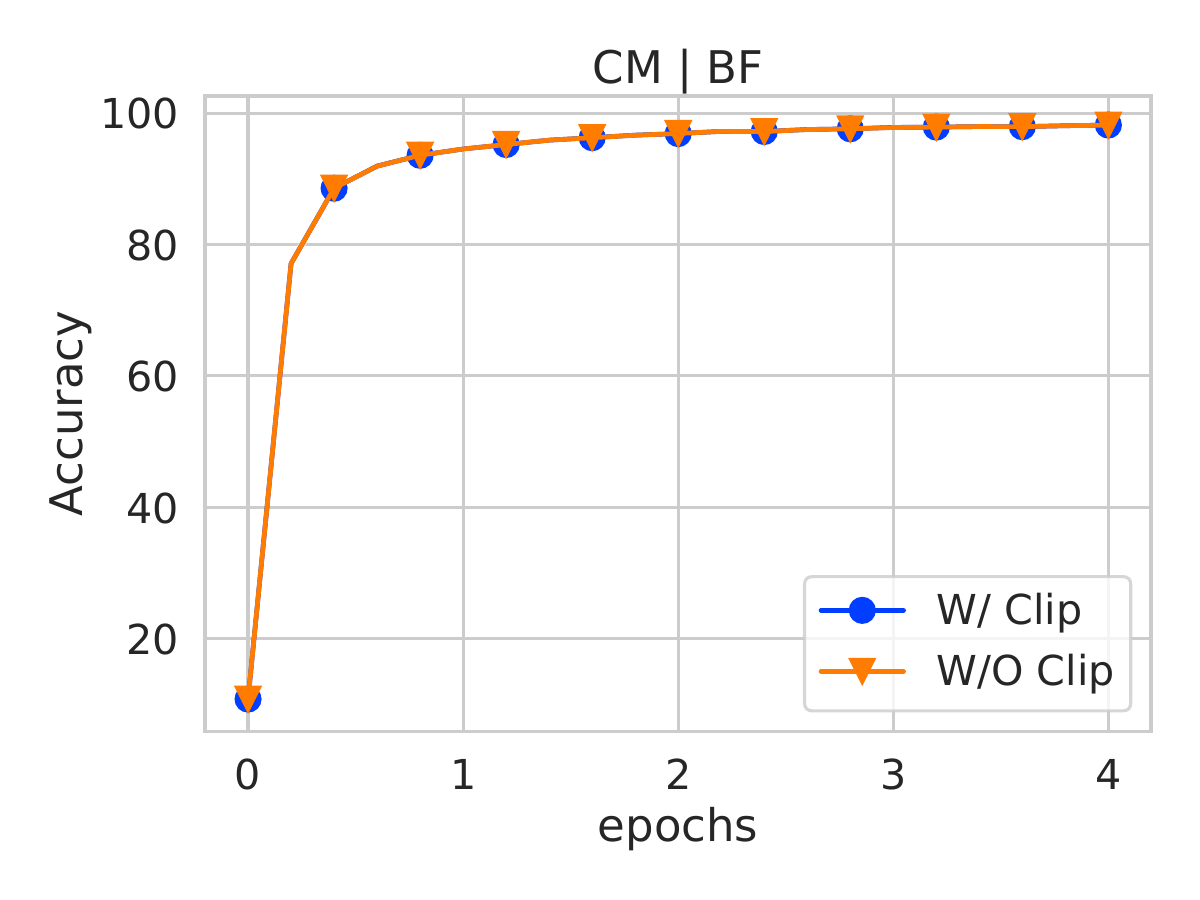}
\includegraphics[width=0.245\textwidth]{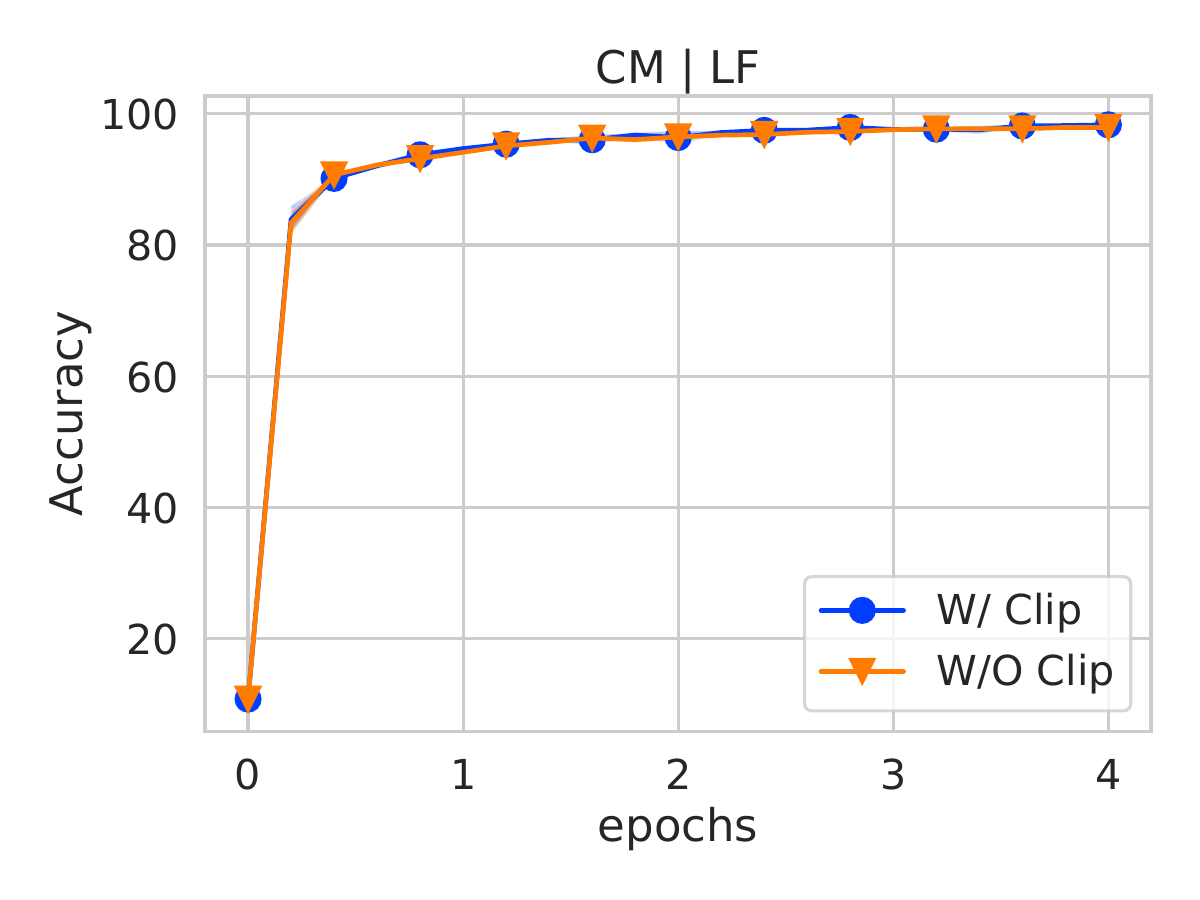}
\includegraphics[width=0.245\textwidth]{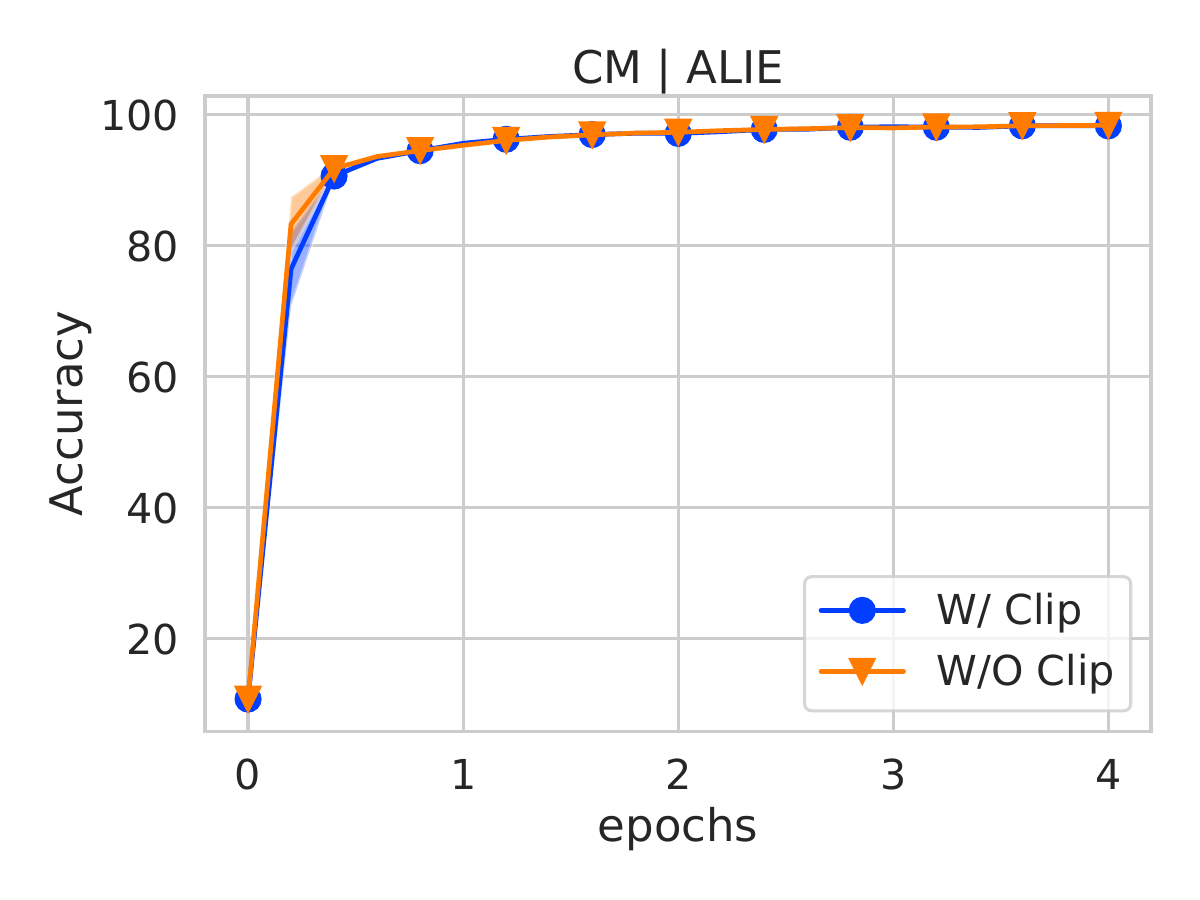}
\includegraphics[width=0.245\textwidth]{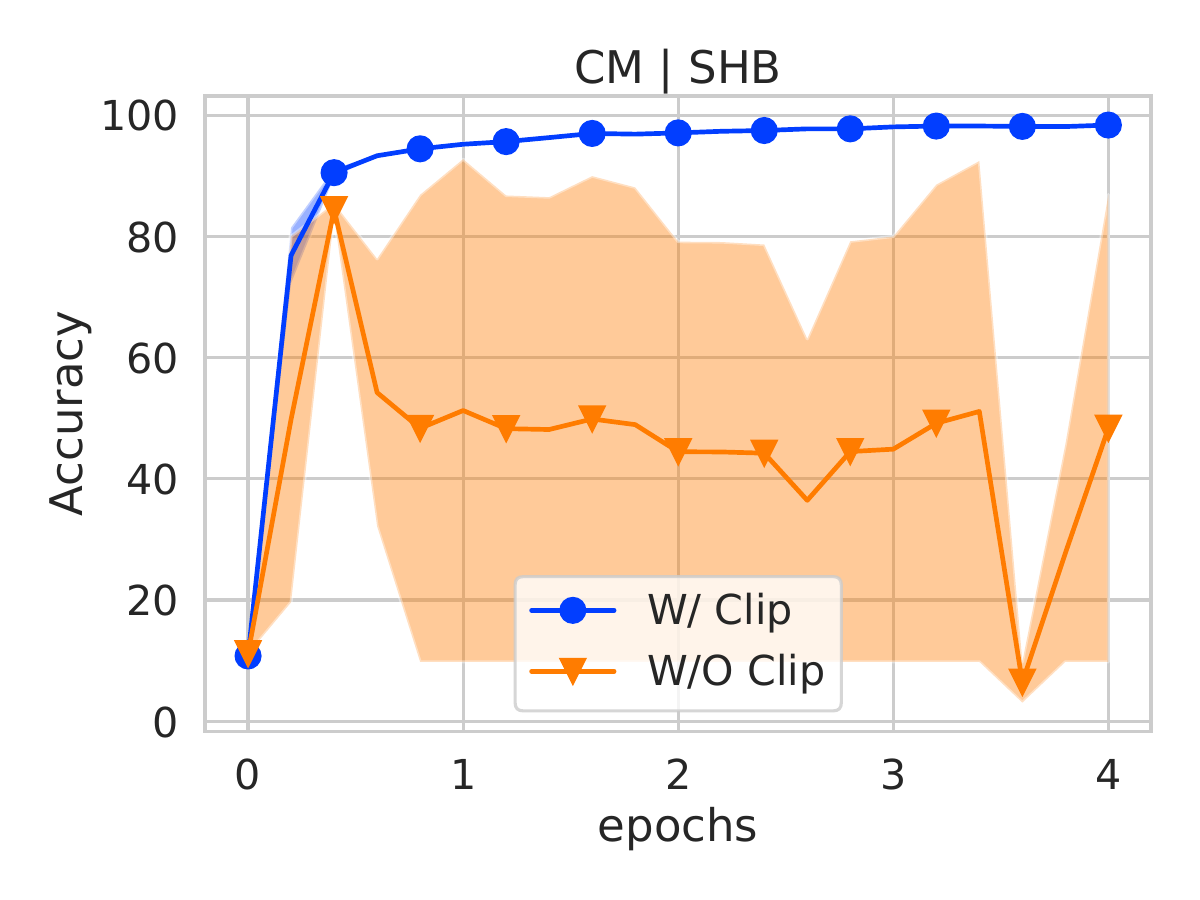}
\hfill
\includegraphics[width=0.245\textwidth]{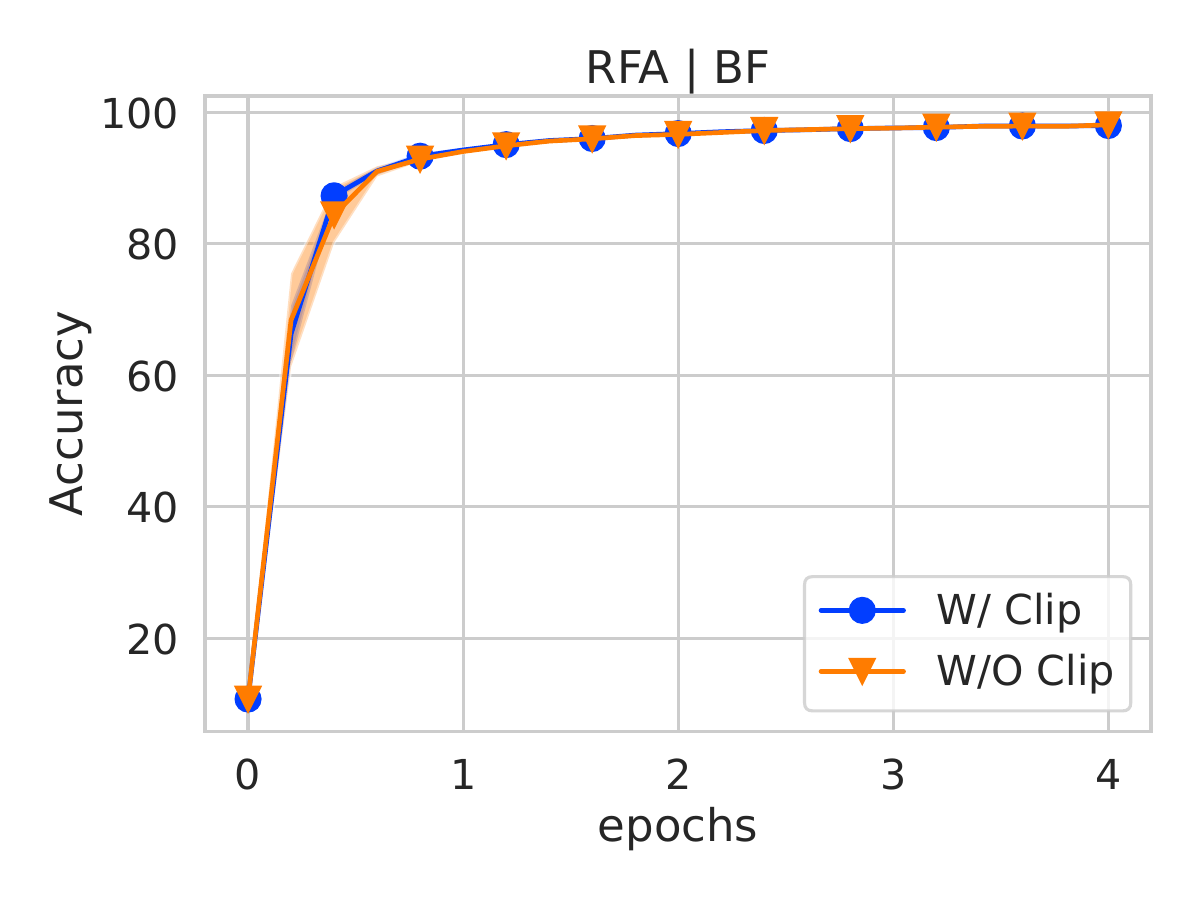}
\includegraphics[width=0.245\textwidth]{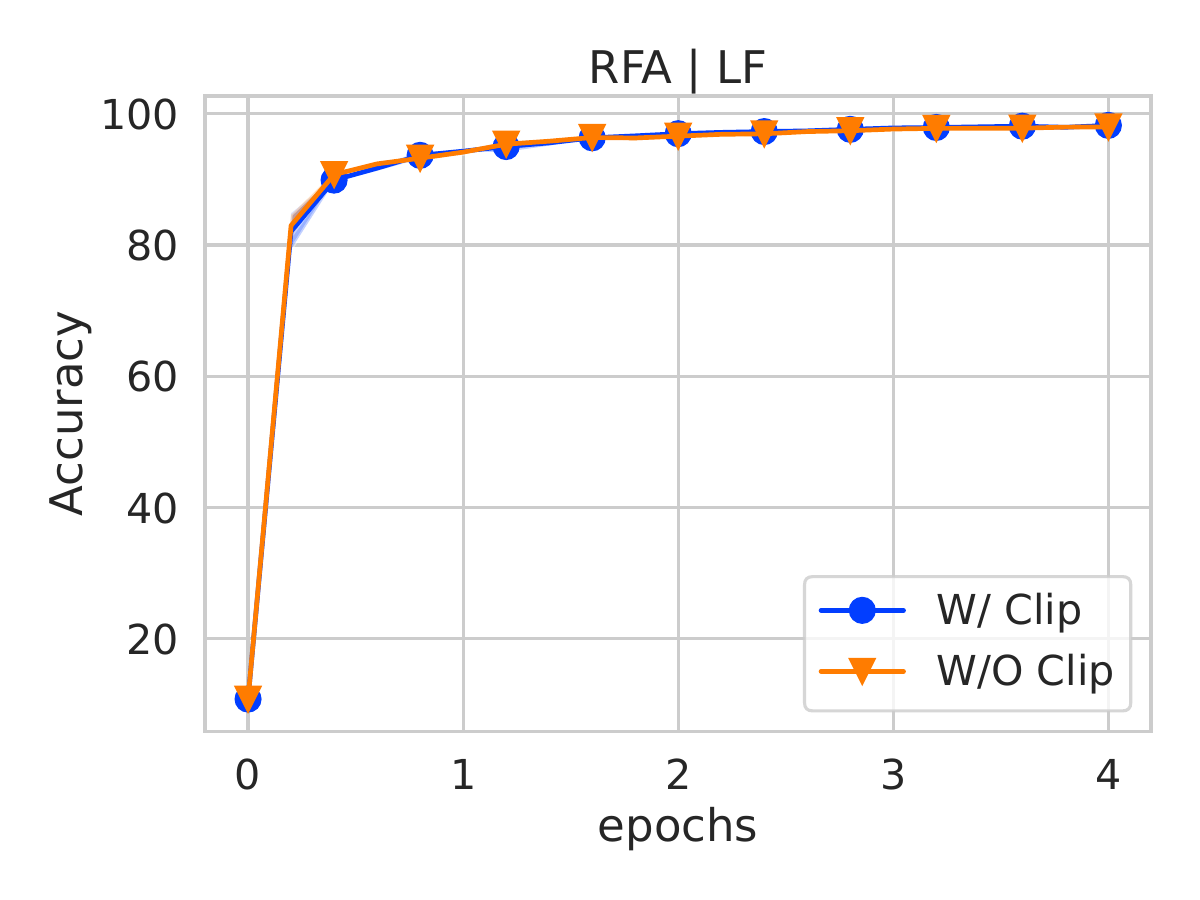}
\includegraphics[width=0.245\textwidth]{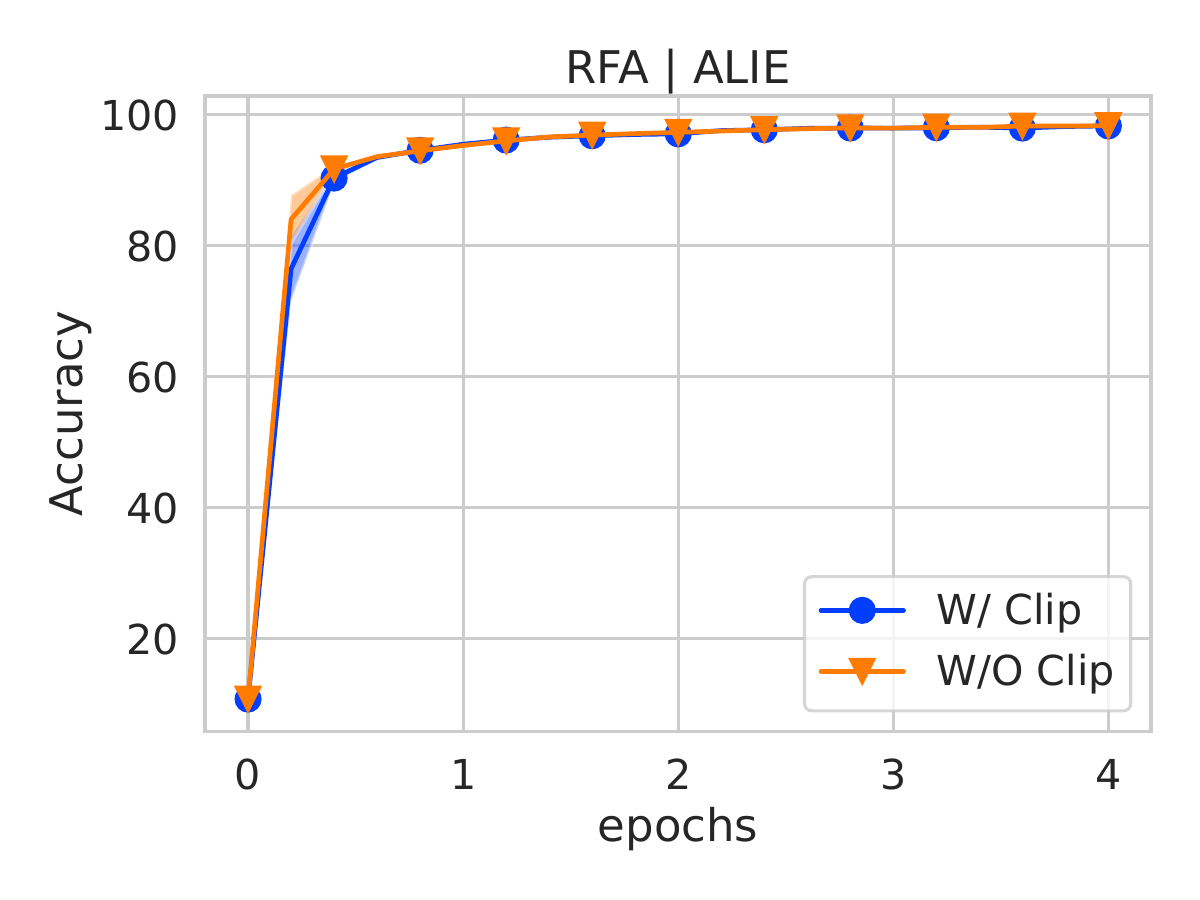}
\includegraphics[width=0.245\textwidth]{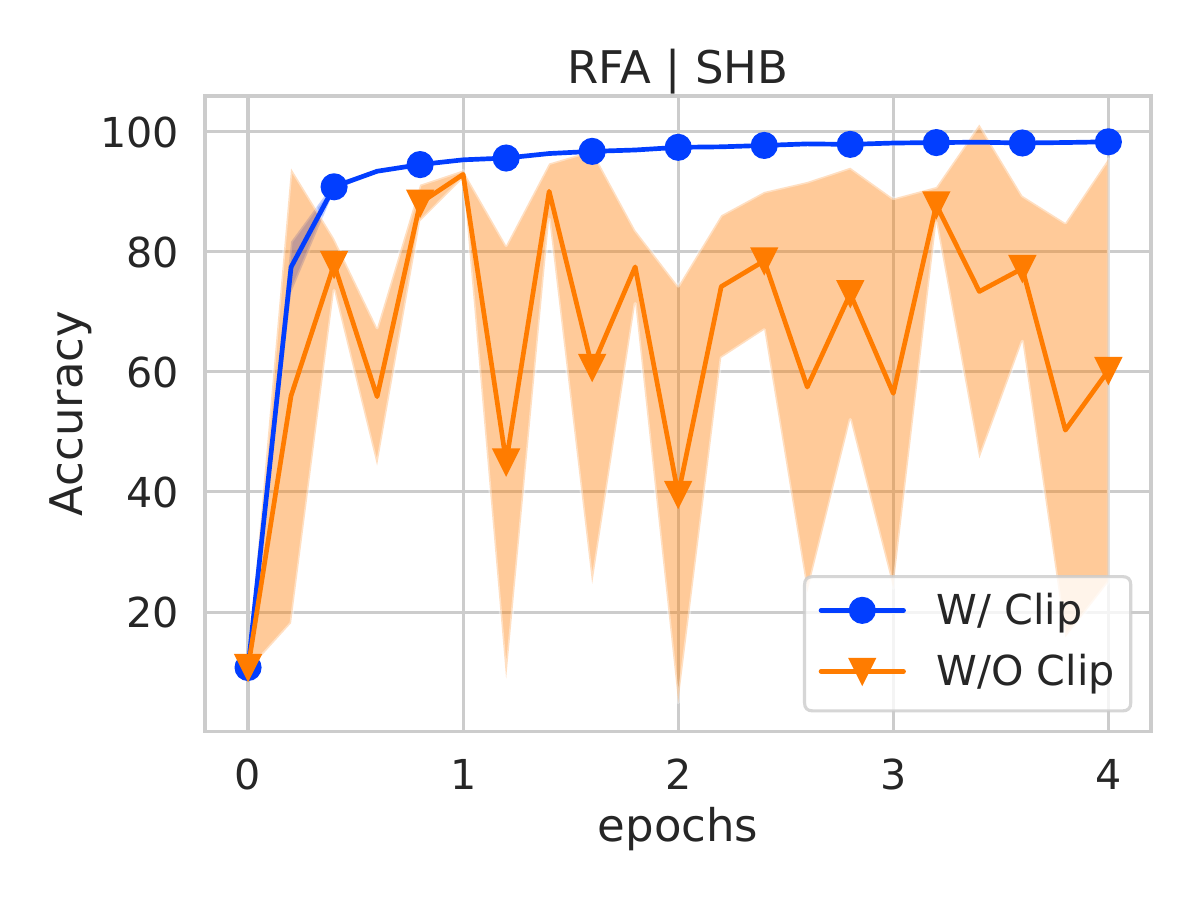}
\caption{
Testing accuracy of 2 aggregation rules (CM, RFA) under 4 attacks (BF, LF, ALIE, SHB) on the MNIST dataset under heterogeneous data split with 20 clients,  5 of which are malicious. } 
\label{fig:nn_app_ta}
\vspace{-1.5em}
\end{figure}

\end{document}

%% file: macros.tex
\usepackage{amssymb, amsmath, amsthm, latexsym}
\usepackage{url}
\usepackage{algorithm}
\usepackage{algorithmic}
\usepackage{tabularx}
\usepackage{paralist}
\usepackage{mathtools}

\usepackage{bbm} 
\usepackage{wrapfig}
\usepackage{makecell}
\usepackage{multirow}
\usepackage{booktabs}

\usepackage{nicefrac}       

\usepackage[flushleft]{threeparttable} 

\usepackage{caption}
\usepackage{multirow}
\usepackage{colortbl}
\definecolor{bgcolor}{rgb}{0.8,1,1}
\definecolor{bgcolor2}{rgb}{0.8,1,0.8}
\definecolor{niceblue}{rgb}{0.0,0.19,0.56}

\usepackage{hyperref}
\hypersetup{colorlinks,linkcolor={blue},citecolor={niceblue},urlcolor={blue}}

\usepackage{pifont}
\definecolor{PineGreen}{RGB}{0,110,51}
\definecolor{BrickRed}{RGB}{143,20,2}

\usepackage{tikz-cd} 

\usepackage{subfigure}

\newcommand{\R}{\mathbb{R}}
\newcommand{\eqdef}{:=}

\def\<#1,#2>{\left\langle #1,#2\right\rangle}

\newcolumntype{Y}{>{\centering\arraybackslash}X}

\usepackage{xspace}

\definecolor{mycolor}{RGB}{0,150,70}
\newcommand{\algname}[1]{{\color{mycolor}\sf  #1}\xspace}

\newcommand{\argmin}{\mathop{\arg\!\min}}

\usepackage[colorinlistoftodos,bordercolor=orange,backgroundcolor=orange!20,linecolor=orange,textsize=scriptsize]{todonotes}

\theoremstyle{plain}
\newtheorem{theorem}{Theorem}[section]

\newtheorem{lemma}[theorem]{Lemma}

\theoremstyle{definition}
\newtheorem{definition}[theorem]{Definition}
\newtheorem{assumption}[theorem]{Assumption}
\theoremstyle{remark}

\newcommand{\cA}{{\cal A}}
\newcommand{\cB}{{\cal B}}

\newcommand{\cG}{{\cal G}}

\newcommand{\cL}{{\cal L}}

\newcommand{\cN}{{\cal N}}
\newcommand{\cO}{{\cal O}}
\newcommand{\cP}{{\cal P}}
\newcommand{\cQ}{{\cal Q}}

\newcommand{\EE}{\mathbb{E}}

\newcommand{\PP}{\mathbb{P}}

\def\clip{\texttt{clip}}

\def\clip{\texttt{clip}}

\usepackage{hyperref}
\graphicspath{{plots/}}

\usepackage{makecell}

\usepackage{accents}
\newlength{\dhatheight}

\usepackage{pgfplotstable} 
\usetikzlibrary{automata, positioning, arrows, shapes, fit, calc, intersections}
\usepgfplotslibrary{statistics}
\include{figure}

\def\deltar{\delta_{\text{real}}}